\documentclass[11pt]{article}
\usepackage[margin=1in]{geometry}

\title{Performance Gaps Between Open-Loop, Certainty-Equivalent, and Optimal Control Under Action-Dependent Observations}
\author{
  \begin{tabular}{ccc}
Leo Maynard-Zhang & Sunmook Choi & Yahya Sattar \\
U Washington & Cornell & Cornell  \\
{\tt leomayn@cs.washington.edu} & {\tt sc3377@cornell.edu} & {\tt ysattar@cornell.edu} \\[2ex]
Yassir Jedra & Sarah Dean & Maryam Fazel  \\
Imperial College London & Cornell & U Washington \& Amazon  \\
{\tt y.jedra@ic.ac.uk} & {\tt sdean@cornell.edu} & {\tt mfazel@uw.edu} 
\end{tabular}
}
\date{}

\usepackage[round]{natbib}

\usepackage{comment}
\usepackage[table]{xcolor}
\newcounter{overviewcase}
\renewcommand{\theoverviewcase}{(\roman{overviewcase})}
\usepackage{dblfloatfix}

\ifdefined\DeclareUnicodeCharacter
  \DeclareUnicodeCharacter{03C1}{\ensuremath{\rho}}
\fi

\usepackage{longtable}
\usepackage{booktabs}

\usepackage{xcolor}

\definecolor{green}{rgb}{0,0.5,0}

\usepackage{amsmath, amsfonts, amssymb, amsthm}
\usepackage{mathtools, bm}
\usepackage{booktabs, multirow}
\usepackage{graphicx, subcaption}
\usepackage[normalem]{ulem}
\usepackage{enumitem}
\allowdisplaybreaks
\usepackage{mdframed}

\makeatletter
\@ifundefined{theorem}{\newtheorem{theorem}{Theorem}}{}
\@ifundefined{assumption}{\newtheorem{assumption}{Assumption}}{}
\@ifundefined{proposition}{\newtheorem{proposition}{Proposition}}{}
\@ifundefined{corollary}{}{}
\@ifundefined{lemma}{\newtheorem{lemma}{Lemma}}{}
\@ifundefined{claim}{}{}
\@ifundefined{definition}{}{}
\@ifundefined{remark}{\newtheorem{remark}{Remark}}{}
\@ifundefined{example}{}{}
\@ifundefined{result}{}{}
\@ifundefined{problem}{}{}
\@ifundefined{question}{}{}
\@ifundefined{notation}{}{}
\@ifundefined{condition}{}{}
\@ifundefined{conjecture}{}{}
\@ifundefined{remarks}{}{}
\@ifundefined{assumptionalt}{\newtheorem{assumptionalt}{Assumption}[assumption]}{}
\makeatother

\numberwithin{equation}{section}

\DeclareMathOperator*{\argmin}{arg\,min}

\renewcommand{\cite}{\citep}

\usepackage{algorithm}
\usepackage{algpseudocode}

\newcommand{\R}{\mathbb{R}}
\newcommand{\Z}{\mathbb{Z}}

\newcommand{\E}{\mathbb{E}}
\newcommand{\Var}{\mathrm{Var}}

\newcommand{\norm}[1]{\left\|#1\right\|}

\newcommand{\tr}{{\rm tr}}

\newcommand{\Fcal}{\mathcal{F}}
\newcommand{\Gcal}{\mathcal{G}}
\newcommand{\Hcal}{\mathcal{H}}
\newcommand{\Ical}{\mathcal{I}}

\newcommand{\Mcal}{\mathcal{M}}
\newcommand{\Ncal}{\mathcal{N}}

\newcommand{\Tcal}{\mathcal{T}}
\newcommand{\Ucal}{\mathcal{U}}
\newcommand{\Vcal}{\mathcal{V}}

\newcommand{\Xcal}{\mathcal{X}}
\newcommand{\Ycal}{\mathcal{Y}}

\newcommand{\Qtil}{\tilde{Q}}

\newcommand{\Vtil}{\tilde{V}}

\definecolor{darkred}{RGB}{150,0,0}
\definecolor{darkgreen}{RGB}{0,150,0}
\definecolor{darkblue}{RGB}{0,0,200}
\usepackage{hyperref}
\hypersetup{colorlinks=true, linkcolor=darkblue, citecolor=darkblue,urlcolor=darkblue}

\newcommand{\xhat}{\hat{x}}
\newcommand{\bsig}{\bar{\Sigma}}
\newcommand{\paren}[1]{\left( #1 \right)}
\newcommand{\squ}[1]{\left[ #1 \right]}
\newcommand{\curl}[1]{\left\{ #1 \right\}}

\usepackage{tikz}
\usetikzlibrary{arrows.meta}
\addtocontents{toc}{\protect\setcounter{tocdepth}{-1}}

\begin{document}
\maketitle

\begin{abstract}
In dynamical systems, feedback can improve performance by allowing decisions to depend on past observations, rather than being designed a priori. Under full observation, the current state is sufficient for decision-making; under partial observation, the controller must instead estimate the state from measurements.
When actions also affect observations, control can have a dual role: actions influence both the state and the information available for future decisions. This interaction can invalidate certainty-equivalent (CE) control, a widely used strategy that applies state feedback policies as if the estimated state were the true state. We ask when this suboptimality leads to a quantitative performance gap that grows with the horizon. We study this question in a linear-quadratic system with bilinear observations by comparing the optimal policy, certainty-equivalent control, and the optimal open-loop (OL) policy. We identify an instance-dependent information-strength parameter $q$ and establish distinct performance regimes as $q$ varies. When $q$ is sufficiently small, we show that no policy can improve over OL control by more than $O(1)$ cumulative cost. For large values of $q$, we find that CE achieves an $\Omega(T)$ advantage over OL control. Our main finding is that for intermediate values of $q$, we can construct a policy that occasionally probes for
information and achieves an $\Omega(T)$ advantage over OL control while CE only achieves an $O(1)$ advantage, establishing an $\Omega(T)$ gap between CE and optimal control. Our analysis relies on a regret decomposition that captures the tradeoff between the immediate cost of gathering information and its future benefit.

\end{abstract}
\section{Introduction}
 In sequential control problems, feedback allows a controller to adapt its actions based on observations collected over time, rather than committing to an action sequence in advance. In partially observed control, the controller must choose actions based on incomplete information about the state. Moreover, an action may influence not only the state, but also the quality of future observations. This creates a dual effect \cite{bar1974dual}: a controller's action affects both its immediate performance and the information available for future decisions. Consequently, an action that is not immediately optimal for control may be worthwhile to improve future decision-making. In the classical linear-quadratic-Gaussian (LQG) problem this tradeoff does not arise due to action-independent observations, and the separation principle allows estimation and control to be treated separately \cite{witsenhausen1971separation,aastrom2012introduction}. When observations are action-dependent, this separation can fail \cite{sattar2025sub}, raising the question of when the benefit of actively acquiring information is large enough to justify its cost.

The informational role of actions is also central in sequential decision-making. In regret-minimization bandits, performance is measured by cumulative sub-optimality relative to a benchmark, and actions balance immediate reward against acquiring information useful for future decisions \cite{bubeck2012regret}. In pure exploration problems such as best-arm identification, the sampling strategy is designed directly around collecting enough information to distinguish among alternatives \cite{garivier2016optimal}. Interestingly, in some of these problems, strong guarantees can be achieved using a sampling allocation fixed in advance, without adapting to observations \cite{maynard-zhang2026on}. This raises a broader question in sequential decision-making: when is a policy fixed in advance sufficient, and when does adaptation to observations provide a genuine advantage?

In dynamic or stateful environments, this distinction also determines which policy class provides an appropriate benchmark for regret analysis. Benchmarks may range from fixed action policies, to an open-loop action sequence chosen in advance, to a fully history-dependent closed-loop policy. These choices can lead to substantially different notions of performance. For example, in restless bandits, weak regret compares against a fixed-arm benchmark \cite{tekin2012online}, whereas strong regret compares against an optimal closed-loop policy \cite{tekin2011adaptive,ortner2012regret}. Optimal closed-loop policies, however, can be hard to compute. In some structured settings, this difficulty can be avoided by weakening the benchmark: when state evolution is deterministic, the state trajectory is determined by the action sequence, so an open-loop benchmark suffices. This is the case in several stateful bandit models with deterministic state evolution \cite{levine2017rotting, kleinberg2018recharging,pike2019recovering,basu2019blocking, khosravi2023preferences}. Thus, whether adaptation provides a meaningful advantage depends on the structure of the problem.

Related questions about benchmark choice arise in recent works on bandits with latent dynamics. \citet{choi2026explore} study a non-stationary bandit problem with latent linear dynamics and bilinear rewards, measuring regret against an optimal open-loop action sequence. \citet{hwang2026blockoptimismnonstationarybandits} develop an adaptive optimism-based algorithm for this model while retaining the same open-loop benchmark. Other dynamic learning and online-control formulations make different comparator choices. One common notion in non-stationary bandits is dynamic regret, which compares against a time-varying oracle \cite{ liu2023definition}; related formulations have also been studied in linear control \cite{goel2020regret}. In online and non-stochastic control, regret is often instead measured relative to a restricted class of feedback-controllers \cite{agarwal2019online,simchowitz2020improper}. Together, these formulations emphasize that there is no universal choice of benchmark for sequential problems with dynamics.

In partially observed dynamical systems, this benchmark question leads to several natural policy classes with different uses of feedback. Open-loop control commits to an action sequence without adapting to observations, while certainty-equivalent (CE) control responds to observations through a state estimate, applying the fully observed controller as if that estimate were the true state. The optimal policy may instead depend on the full information state, and explicitly account for how actions affect future information. While CE is known to be optimal in the classical LQG setting~\cite{kumar2015stochastic}, recent work shows that CE can be suboptimal in linear-quadratic control with bilinear observations \cite{sattar2025sub}. However, this result does not characterize how the resulting performance gap scales with the horizon. We ask when the performance gaps between open-loop, CE, and optimal control grow with the horizon.

We study this question in the scalar linear-quadratic setting with bilinear observations, where observation quality depends on the action. This structure creates a failure mode for CE: when the state estimate is small, CE takes a small action, resulting in weak observations and little new information. We identify an instance-dependent information strength parameter $q$ that captures how strongly information acquired under CE contributes to the growth rate of the state estimate, and hence to the informativeness of future observations. We show that the relative performance of open-loop, CE, and optimal control varies qualitatively across different ranges of $q$. Our contributions are as follows.
\begin{enumerate}
    \item We establish four instance-dependent thresholds $ q_{1}< q_{2}< q_{3}< q_{4}$ that separate different performance regimes. When $q < q_{1}$ no policy can achieve more than an $O(1)$ advantage over open-loop control. When $q < q_{3}$, CE achieves at most an $O(1)$ advantage over open-loop control. Conversely, when $q > q_{4}$, CE achieves an $\Omega(T)$ advantage over open-loop control.
    \item Our main finding is that when $q>q_{2}$ optimal control has an $\Omega(T)$ advantage over open-loop control. Consequently, when $q_2<q <q_3$, the performance gap between optimal control and CE control is $\Omega(T)$. To establish this result, we construct a policy that occasionally deviates from CE to probe for information and show that when $q > q_2$ it achieves an $\Omega(T)$ advantage over open-loop control.
\end{enumerate}

 Table~\ref{tab:regimes} summarizes the established bounds.

\begin{table}[ht]
    \centering
    \small
    \setlength{\tabcolsep}{3.5pt}
    \begin{tabular}{c|c|c|c}
        \hline
        Regime & Threshold & $R_T(\pi^\star, \pi^{OL})$ & $R_T(\pi^{CE}, \pi^{OL})$ \\
        \hline
        Low info. & $q < q_{1}$ & $O(1)$ &  $O(1)$\\
        Separation & $q_{2}<q<  q_{3}$ & $\Omega(T)$ & $O(1)$ \\
        High info. & $q_{4}< q$ & $\Omega(T)$ & $\Omega(T)$ \\
        \hline
    \end{tabular}
    \caption{$R_T(\pi,\pi')$ refers to the regret of policy $\pi'$ with respect to $\pi$ as the benchmark. $\pi^{OL}$, $\pi^{CE}$, and $\pi^\star$ refer to the optimal OL, CE, and optimal policies, respectively.}
    \label{tab:regimes}
\end{table}

\section{Preliminaries} \label{sec:prelim}
\paragraph{Notation.}

We denote the positive integers, integers, and nonnegative integers by
\(\mathbb{N}\), \(\mathbb{Z}\), and \(\mathbb{Z}_{\ge 0}\), respectively. For \(T \in \mathbb{Z}_{\ge 0}\), let \( [T] := \{0,1,\ldots,T\}. \) For a set \(\mathcal{A}\), let \(\Delta(\mathcal{A})\) denote the set of
probability distributions over \(\mathcal{A}\), and let
\(\mathbf{1}\{\mathcal{E}\}\) denote the indicator of an event
\(\mathcal{E}\). For a matrix \(A\), \(\operatorname{tr}(A)\), \(\rho(A)\), and
\(\|A\|\) denote its trace, spectral radius, and spectral norm,
respectively, while \(\|\cdot\|_2\) denotes the Euclidean norm. For symmetric matrices \(A\) and \(B\), we write \(A \succeq B\)
(resp.\ \(A \succ B\)) if \(A-B\) is positive semidefinite
(resp.\ positive definite).
Finally, \(X \overset{d}= Y\) denotes equality in distribution.

\paragraph{General problem setting.}
We first introduce a general partially observable control setting  to define the three policies being compared. 
A partially observable control problem is defined by $(\mathcal{X}, \mathcal{U}, \mathcal{Y}, f,g,c,\mu,T),$ where $\mathcal{X}, \mathcal{U}, \mathcal{Y}$ are the state, action, and observation spaces respectively, $f:\mathcal{X}\times  \mathcal{U} \rightarrow \Delta(\mathcal{X})$ is the state transition kernel, $g:\mathcal{X}\times  \mathcal{U} \rightarrow \Delta(\mathcal{Y})$ is the observation kernel, $c: \mathcal{X}\times  \mathcal{U} \rightarrow \mathbb{R}$ is the stage cost, $\mu \in \Delta(\mathcal{X})$ is the initial state distribution, and $T$ is the horizon. Interaction is as follows. We start at initial state $x_0 \sim \mu(\cdot)$, which is unobserved. Then at each time step $t= 0, \dots, T$, a controller takes action $u_t$ based on previous observations, incurs cost $c(x_t, u_t)$, and obtains observation $y_t \sim g(\cdot \mid x_t, u_t)$. The state of the system transitions to $x_{t+1} \sim f(\cdot \mid x_t, u_t)$.

\paragraph{Policies.}
We define the history of observations available when choosing $u_t$ as \begin{equation*}
    H_t := \paren{u_0,y_0, \dots, u_{t-1},y_{t-1}}
\end{equation*}
with $H_0=\emptyset$.
In particular, $y_t$ is observed only after $u_t$ has been chosen.
Let $\Hcal_t = \paren{\Ucal\times \Ycal}^{t}$ denote the space of histories. Let $\Pi$ denote the class of policies $\pi = \curl{\pi_t}_{t=0}^T$ where each $\pi_t {\colon} \Hcal_t \to \Ucal$ is measurable and $u_t =\pi_t(H_t)$. We call a policy \textbf{open-loop} if its actions do not depend on the history~\cite{kumar2015stochastic}. Formally, we define the class of open-loop policies \begin{equation*}
    \Pi^{OL} {:=} \big\{\pi \in \Pi: \pi_t(H) = \pi_t(H'),\; \forall H,H' {\in }\mathcal{H}_t, t \in [T] \big\}.
\end{equation*}
Equivalently, open-loop policies commit to a deterministic action sequence before interaction begins. For a policy $\pi$, let $\E_\pi$ denote expectation with respect to the trajectory distribution induced by the initial distribution $\mu$, transition kernel $f$, observation kernel $g$, and actions selected according to $\pi$.
\paragraph{Objectives.}
We define the expected cumulative cost of a policy $\pi \in \Pi$ as
\begin{equation*}
    J_T(\pi) = \E_\pi\left[\sum_{t=0}^T c(x_t,u_t) \right].
\end{equation*}
We adopt the convention that a policy designer aims to minimize $J_T$. For $\pi,\pi'$ we define the \textit{regret} of $\pi'$ with respect to $\pi$ as 
\begin{equation*}
    R_T(\pi,\pi') = J_T(\pi') - J_T(\pi).
\end{equation*}
Alternatively, $R_T$ can be considered the advantage of $\pi$ over $\pi'$.

\paragraph{State-action cost-to-go.}
In order to define CE we consider the corresponding fully observed problem where the state is known. For each $t$, define the state-action cost-to-go function $\widetilde{Q}_t\colon \Xcal\times\Ucal\to\R$ recursively by:
\begin{equation*}
     \widetilde{Q}_{T}(x,u)  =c(x,u)
\end{equation*}
and for each $t \in [T-1]$,
\begin{equation*}
    \widetilde{Q}_t(x,u) = c(x,u) {+} \E_{x'\sim f(\cdot \mid x,u)}\squ{\min_{u'\in\Ucal}\widetilde{Q}_{t+1}(x',u')}.
\end{equation*}
Then, \begin{equation*}
    \argmin_{u \in \mathcal{U}} \widetilde{Q}_t(x,u)
\end{equation*}
denotes the set of optimal actions at time step $t$ and state $x$ in the fully observed problem. The CE policy will apply this fully observed decision rule to a state estimate.

\paragraph{State estimation.}
The final ingredient needed to define CE is the state estimate. Because the system is only partially observed, the controller generally does not know $x_t$  when choosing $u_t$ and must instead estimate it from the history. We use the conditional mean \begin{equation*}
    \xhat_t :=\E[x_t \mid H_t]
\end{equation*}  
as the state estimate. More generally, CE control may be defined with an alternative state estimator when the conditional mean is unavailable or does not belong to the state space.

\paragraph{Benchmark policies.}

Now we define the benchmark policies that we consider in this paper.
\begin{enumerate}
    \item The optimal policy $\pi^\star$ is defined as \begin{equation*}
    \pi^\star \in \argmin_{\pi\in\Pi} J_{T}(\pi)
\end{equation*}
    \item The certainty-equivalent (CE) policy $\pi^{CE}$ applies the optimal fully observed policy to the estimated state $\xhat_t$ and is defined as \begin{equation*}
    \pi^{CE}_t(H_t) 
    \in\arg\min_{u} \widetilde{Q}_t(\xhat_t,u), \forall t
\end{equation*}
    \item  The optimal open-loop (OL) policy $\pi^{OL}$ is \begin{equation*}
        \pi^{OL} \in \argmin_{\pi\in\Pi^{OL}} J_{T}(\pi)
    \end{equation*}
\end{enumerate}
\section{Linear-Gaussian Control With Bilinear Observation}
In this section, we introduce the setting we analyze. We first summarize the main performance regimes and explain the mechanism underlying them. We then introduce the linear-quadratic Gaussian setting with bilinear observations, the associated Kalman filter, and the corresponding CE and OL policies.

\subsection{Overview of the Main Results}
Before introducing our main result, we begin with the following theorem about the classical LQG setting with action-independent observations.
\begin{theorem}\label{thm:lqg}
    In the LQG setting, we have
    \begin{enumerate}[label=(\alph*),ref=(\alph*)] 
        \item \label{thm:lqg-a}
        $\pi^\star = \pi^{CE}$
        \item \label{thm:lqg-b}
        $R_T(\pi^{\star},\pi^{OL}) = \Omega(T)$
        for all (non-degenerate) instances
    \end{enumerate}
\end{theorem}
Theorem~\ref{thm:lqg}\ref{thm:lqg-a} follows from the classical separation principle \cite{witsenhausen1971separation}, while Theorem~\ref{thm:lqg}\ref{thm:lqg-b}  follows from action-independent observations, and its proof is deferred to Appendix~\ref{app:dp-policies}. Theorem~\ref{thm:lqg} shows that in LQG, feedback provides a linear advantage over open-loop control, and since $\pi^\star = \pi^{CE} $, CE fully realizes this advantage. However, under action-dependent feedback, different relationships between open-loop, CE, and optimal control can arise. The following theorem summarizes the three possibilities that we realize in our bilinear feedback setting.
\begin{theorem} \label{thm:overview}
    For each case below, there exists a partially observed control problem satisfying the bounds:
    \begin{center}
        \setcounter{overviewcase}{0}
        \begin{tabular}{c|cc}
            & $R_T(\pi^\star,\pi^{OL})$
            & $R_T(\pi^{CE},\pi^{OL})$ \\ \hline
        
            \refstepcounter{overviewcase}\label{case:i}
            \theoverviewcase
            & $O(1)$ & $O(1)$ \\
        
            \refstepcounter{overviewcase}\label{case:ii}
            \theoverviewcase
            & $\Omega(T)$ & $O(1)$ \\
        
            \refstepcounter{overviewcase}\label{case:iii}
            \theoverviewcase
            & $\Omega(T)$ & $\Omega(T)$
        \end{tabular}
        \end{center}
\end{theorem}
We show that all three cases in Theorem~\ref{thm:overview} arise within the bilinear observation model introduced in Section~\ref{sec:model-KF}. This sharply contrasts with LQG. Theorem~\ref{thm:overview}\ref{case:i} implies there exist instances in which optimal control fails to achieve more than a constant advantage over open-loop. More strikingly, Theorem~\ref{thm:overview}\ref{case:ii} implies feedback can provide a linear advantage over open-loop control, but CE fails to realize it. Theorem~\ref{thm:overview}\ref{case:iii} recovers the LQG behavior in which CE obtains a linear advantage over open-loop control.

In Section~\ref{sec:info-strength-q} we quantify the strength of information by an instance-dependent parameter $q$, and Section~\ref{sec:results} gives sufficient conditions on $q$ for realizing the three cases above.

\subsection{Model and Kalman Filtering} \label{sec:model-KF}

\paragraph{Our Model.}
We study a linear-quadratic Gaussian setting with bilinear feedback. This is instantiated as follows. Let $\Xcal = \R^n$, $\Ucal = \R^p$, and $\Ycal = \R$.
An optimal control problem is specified by $(A,B,C, \Sigma_w, \sigma_z^2,Q,R,S, \mu_x)$ through
\begin{align*}
    f(\cdot\mid x, u) &= \mathcal{N}(Ax + Bu, \Sigma_w), \\
    g(\cdot\mid x, u) &= \mathcal{N}(u^\top  Cx, \sigma_z^2), \\
    c(x,u) &= x^\top  Q x + u^\top  R u + 2u^\top  S x, \\
    \mu(\cdot) &= \Ncal\paren{\mu_x, \bsig}.
\end{align*}
We assume $\rho(A)<1$, $R{\succ}0$, $Q {-} S^\top  R^{-1}S {\succ} 0$, $\Sigma_w{\succ}0$, and $\sigma_z^2>0$. Here, $\bsig$ is the stationary open-loop covariance, uniquely determined by $\bsig = A \bsig A^\top  + \Sigma_w$.

\paragraph{Kalman Filtering.}
To compute the state estimate used by $\pi^{CE}$, we recall the Kalman filtering results of \citet[Section~2.2 and Lemma~1]{sattar2025sub}.
Under any policy $\pi$, the posterior state distribution is
\begin{equation*}
    x_t \mid H_t \sim \mathcal{N}\paren{\xhat_{t}, \Sigma_{t}}.
\end{equation*}
Starting from $\xhat_0 = \mu_x$ and $\Sigma_0 = \bsig$, the conditional mean and covariance can be computed through the following recursion:
\begin{equation} \label{eq:kalmanrecursion}
    \begin{split}
        \xhat_{t+1} &= A\xhat_t + Bu_t + \frac{A\Sigma_t C^\top u_t (y_t-u_t^\top C\xhat_t)}{\sigma_z^2+u_t^\top C\Sigma_t C^\top u_t}, \\
        \Sigma_{t+1} &= A\Sigma_tA^\top + \Sigma_w - F(\Sigma_t,u_t),
    \end{split}
\end{equation}
where 
\begin{equation*}
    F(\Sigma,u) = \frac{A\Sigma C^\top  uu^\top  C\Sigma A^\top }{\sigma_z^2 + u^\top  C \Sigma C^\top  u} \succeq 0
\end{equation*}
These recursions form a variant of the classical Kalman filter~\cite{kalman1960new}. 
Unlike in classical LQG, the observation map is input-dependent, so both the Kalman gain and the covariance update depend on the input.

We note that $F(\Sigma,u)$ represents the next-step covariance reduction relative to the open-loop covariance $A\Sigma A^\top + \Sigma_w$.
The following lemma gives another interpretation of $F(\Sigma,u)$.
\begin{lemma}\label{lem:kalmansecond}
    \begin{equation*}
        \E[\xhat_{t+1}\xhat_{t+1}^\top  \mid H_t] = (A \xhat_t + B u_t)(A \xhat_t + B u_t)^\top + F\paren{\Sigma_t, u_t}
    \end{equation*}
\end{lemma} 
Lemma~\ref{lem:kalmansecond} implies that $F(\Sigma,u)$ can also be interpreted as influencing the growth rate of the second moment of the state estimate.

\subsection{Benchmark Policies}

In this section, we realize the certainty-equivalent (CE) policy and the open-loop (OL) policy.
Set $P_{T+1} =0$ and define for $t=T,\dots,0$,
\begin{equation} \label{eq:finite-Riccati-recursion}
    \begin{split}
        N_t &= R+B^\top P_{t+1}B, \\
        K_t &= -N_t^{-1}(B^\top  P_{t+1}A + S), \\
        P_t &= A^\top  P_{t+1} A + Q -K_t^\top  N_{t} K_t.
    \end{split}
\end{equation}
These are the finite-horizon Riccati recursions, with $P_t \succ 0$.
For the fully observed optimal control problem, it is well known \cite{kumar2015stochastic} that
\begin{equation*}
    K_t x \in \arg\min_u \tilde{Q}_t \paren{x, u}.
\end{equation*}
From the definitions in Section~\ref{sec:prelim}, the corresponding benchmark policies are as follows.

\begin{enumerate}
    \item The certainty-equivalent (CE) policy $\pi^{CE}$ is 
    \begin{equation*}
        \pi^{CE}_t(H_t) = K_t \xhat_{t}.
    \end{equation*}
    \item The optimal open-loop (OL) policy $\pi^{OL}$ is
    \begin{equation*}
        \pi^{OL}_t(H_t) = K_t  m_t
    \end{equation*}
    where $m_0 = \mu_x$ and $m_{t+1} = (A+BK_t) m_t$.
\end{enumerate}
Both CE and OL use the gains $K_t$ from the fully observed linear-quadratic problem. 
CE applies these gains to the estimate $\xhat_t$, which incorporates observations, whereas OL applies them to the unconditional state prediction $m_t$, which evolves without incorporating observations. 
Appendix~\ref{app:dp-policies} provides the derivation of the gains and the resulting CE and OL policies. 

We note that characterizing the optimal policy $\pi^\star$ is more challenging.
Previous works~\cite{sattar2025sub, cao2026dual} discuss the need to balance control and information gathering, since actions affect both the state dynamics and future estimation uncertainty.

\begin{remark}[Stationary limits] \label{rmk:stationary} \normalfont
    As the remaining horizon $T{-}t$ grows, the Riccati recursion \eqref{eq:finite-Riccati-recursion} converges exponentially to a stationary solution \citep{hager1976convergence}.
    Specifically, if $P$ is the stationary counterpart of $P_t$, there exist $C_P<\infty$ and $\rho\in(0,1)$, independent of $T$ and $t$, such that 
    \begin{equation*}
        \|P_t - P\| \le C_P \rho^{T-t}.
    \end{equation*}
    Define $D_t = A+BK_t$, and 
    \begin{equation*}
        N=R+B^\top PB, \; K=-N^{-1}(B^\top PA + S), \; D=A+BK.
    \end{equation*}
    Since $N_t \succ 0$, the coefficients $N_t,K_t,D_t$ depend continuously on $P_{t+1}$ via \eqref{eq:finite-Riccati-recursion}.
    Together with exponential convergence of $P_t$, this implies that these coefficients are uniformly close to their stationary counterparts whenever $t \le T-t_0$, where $t_0$ depends on the desired closeness but not on $T$.
    Consequently, strict inequalities involving continuous functions of the stationary coefficients also hold for the corresponding time-varying coefficients sufficiently far from the terminal time. 
    This allows us to state our sufficient conditions using stationary quantities and bound the contributions of the last $t_0$ rounds separately in the proofs.
    Details are relegated to Appendix~\ref{app:dp-policies}.
\end{remark}

\subsection{Information Strength $q$} \label{sec:info-strength-q}

Under $\pi^{CE}$, when the state estimate $\xhat_t$ is small, the CE action $K_t\xhat_t$ is small as well, providing a weak observation. In the scalar setting when $p=n=1$, the covariance reduction is 
\begin{equation*}
    F(\Sigma_{t},K_t \xhat_{t}) = \frac{A^2K_t^2 \xhat_{t}^2C^2\Sigma_{t}^2}{\sigma_z^2 +K_t^2 \xhat_{t}^2C^2 \Sigma_{t}}.
\end{equation*}
For $(\xhat_t, \Sigma_t)$ near $(0, \bsig)$, we have 
\begin{equation*}
    F(\Sigma_{t},K_t \xhat_{t}) \approx q \xhat_t^2 \quad\text{with} \quad 
    q=\frac{A^2K^2C^2 \bsig^2}{\sigma_z^2}.
\end{equation*}
Hence by Lemma~\ref{lem:kalmansecond} it follows that \begin{equation*}
    \E[\xhat_{t+1}^2 \mid H_t] \approx \paren{D^2 + q} \xhat_t^2
\end{equation*}
for $D$ in Remark~\ref{rmk:stationary}. The parameter $q$ can be viewed as the information-feedback strength under CE. Indeed, for beliefs near $(0, \bsig)$, the signal-to-noise ratio of an observation generated by CE is proportional to \begin{equation}
    \frac{K^2C^2 \bsig}{\sigma_z^2} \xhat_t^2,
\end{equation}
while $q$ additionally accounts for how information propagates through the dynamics to the next state estimate. Equivalently, $D^2 + q$ is the local growth rate factor of the second moment of the CE state estimate. This motivates the following study of how the performance gaps between $\pi^{OL}$, $\pi^{CE}$, and $\pi^\star$ vary with $q$.
Section~\ref{sec:results} establishes sufficient conditions on $q$ for bounded and linear advantages over open-loop control, thereby realizing the three cases of Theorem~\ref{thm:overview}.

\section{Results} \label{sec:results}

All results in this section concern the scalar setting where $p=n=1$.

\subsection{Regret decomposition}
Throughout our analysis, we use $\pi^{OL}$ as a common reference. Since $\pi^{OL}$ is optimal among policies that do not use observations $R_T(\pi, \pi^{OL})$ measures the advantage $\pi$ gains from using feedback. In this section, we present in Proposition~\ref{prop:regretdecomp} a regret decomposition with respect to $\pi^{OL}$, which will be used to derive our main results. 
We first define the ingredients of the first decomposition.
For $\xhat,u \in \R$, $\Sigma > 0$, define
\begin{equation} \label{eq:regret-ingredients}
    \begin{split}
        b_t(\Sigma,u) &= \Gamma_{t+1}F(\Sigma,u) \\ 
        p_t(\xhat, u) &=  N_t \paren{u - K_t \xhat}^2
    \end{split}
\end{equation}
where $\Gamma_t$ is defined recursively:
\begin{equation} \label{eq:Gamma_t_def}
    \Gamma_{t} = N_{t} K_t^2 + A^2 \Gamma_{t+1}, \quad \Gamma_{T+1} = 0.
\end{equation}
\begin{proposition}\label{prop:regretdecomp}
    For $\xhat,u\in\R$ and $\Sigma>0$, define 
    \begin{equation} \label{eq:regret-reward}
        r_t\paren{\xhat, \Sigma, u} =  b_t(\Sigma,u)-p_t(\xhat, u).
    \end{equation} 
    For any policy $\pi$,
    \begin{equation*}
        R_{T}(\pi, \pi^{OL}) =\sum_{t=0}^T \E_\pi[r_t(\xhat_{t}, \Sigma_{t}, u_t)].
    \end{equation*}
\end{proposition}
The proof of Proposition~\ref{prop:regretdecomp} is deferred to Appendix~\ref{app:regret}. 
We can interpret the summands $r_t$ as follows.
The values $p_t(\xhat,u)$ and $b_t(\Sigma,u)$ are always nonnegative.
The term $p_t(\xhat, u)$ represents the penalty of deviating from certainty-equivalence, while $b_t(\Sigma,u)$ represents the weighted reduction in next-step estimation uncertainty due to the observation. 
The two terms create a tradeoff in deciding whether to incur cost by deviating from certainty equivalence to obtain information for future benefit, a dilemma created by action-dependent observations.

\subsection{Bounded Optimal-Control Advantage}

In this section, we show that unlike classical LQG, $\pi^\star$ need not achieve a linear improvement over $\pi^{OL}$. This is implied by the following theorem.

\begin{theorem}[Low Information Regime]\label{prop:clconstant}
     Define
     \begin{equation*}
        q_1 := \frac{(1-A^2)K^2}{Q} \paren{R - \frac{S^2}{Q}}.
    \end{equation*}
    If $q<q_1$, then for any policy $\pi$
    \begin{equation}\label{eq:rbound}
       r_t(\xhat_{t}, \Sigma_{t}, u_t) \leq  P_{t} \xhat_{t}^2-P_{t+1}\E_{\pi}[ \xhat_{t+1}^2 \mid H_t].
    \end{equation} Consequently, we have 
    \begin{equation*}
        R_T(\pi^\star, \pi^{OL}) = O(1).
    \end{equation*}
\end{theorem}

The proof of Theorem~\ref{prop:clconstant} is deferred to Appendix~\ref{app:club}. When $q<q_1$, the proof shows that the weighted reduction in estimation uncertainty is dominated by the state-and-action cost.
Consequently, no policy, including the optimal policy $\pi^\star$, can reduce the total cost relative to open-loop control by more than a constant that does not grow with the horizon.
This regime arises, for example, when measurement noise is sufficiently large.
In the next section, we identify an even wider range of instances in which $\pi^{CE}$ achieves only a bounded advantage over open-loop control.

\subsection{Bounded CE Advantage} \label{sec:ceupper}

In this section, we seek an upper bound on the advantage of $\pi^{CE}$ over $\pi^{OL}$.
Under the policy $\pi^{CE}$, we note that the value $p_t(\xhat,u)$, defined in \eqref{eq:regret-ingredients}, is always zero.
This leads to a key identity for bounding this advantage:
\begin{equation}\label{eq:spresidual}
    r_t(\xhat, \Sigma, K_t \xhat) = \Gamma_{t+1} \frac{A^2K_t^2 \xhat^2C^2\Sigma^2}{\sigma_z^2 +K_t^2 \xhat^2C^2 \Sigma}
\end{equation} 
where $r_t$ is defined in \eqref{eq:regret-reward} that aggregates to regret by Proposition~\ref{prop:regretdecomp}.
The quantity~\eqref{eq:spresidual} is nondecreasing in both $\xhat^2$ and $\Sigma$, and vanishes when $\xhat =0$. 
One key attribute of the bilinear setting is that there exists a process $\curl{M_t}_{t\geq0}$ under $\pi^{CE}$ such that
\begin{equation} \label{eq:CEmultiplicative}
    \xhat_{t +1} = M_t \xhat_{t}.
\end{equation}
See Lemma~\ref{lem:multiplicative} in Appendix~\ref{app:kalmantools} for details of $M_t$. 
If $\mu_x = 0$, then $\xhat_t = 0$ for every $t$. 
Hence, $r_t(\xhat_t,\Sigma_t,K_t\xhat_t)=0$, and Proposition~\ref{prop:regretdecomp} gives $R_T(\pi^{CE},\pi^{OL})=0$.

In terms of feedback, this is easy to interpret. A zero estimate leads to the zero action,
so the observation is pure noise and carries no information about the state.
Thus $\pi^{CE}$ receives no useful feedback and coincides with $\pi^{OL}$.
This contrasts with LQG with action-independent observations, where the zero action does not
make the observations uninformative, and where $\mu_x = 0$ can be assumed without loss of generality.

A more interesting failure of $\pi^{CE}$ occurs even when $\mu_x{\neq} 0$.
When each $\E[M_t^2 \mid H_t]\leq\gamma$ is uniformly below 1, Eq.~
\eqref{eq:CEmultiplicative} implies $\E[\xhat_{t}^2]$ decays geometrically. 
Then, $r_t(\xhat_t,\Sigma_t,K_t\xhat_t)$ decays geometrically, and its sum over time remains bounded.
In other words, $\pi^{CE}$ gradually loses the feedback that gave it an advantage over $\pi^{OL}$. We formalize this as follows.

\begin{proposition}\label{prop:constantscalar}
    Define $\gamma:=(D^2+q+1)/2$ and 
    \begin{equation*} 
        q_3 := 1-D^2.
    \end{equation*} 
    If $q<q_3$, then $\gamma < 1$ and there exists a positive integer $t_0$ such that $\E[M_t^2 \mid H_t]\leq\gamma$ for all $t \leq T -t_0$. Consequently \begin{equation*}
        R_T(\pi^{CE},\pi^{OL}) = O(1).
    \end{equation*}
\end{proposition}
The proof of Proposition~\ref{prop:constantscalar} is deferred to Appendix~\ref{app:ceub}.

\subsection{Linear CE Advantage}
Our next result shows that the previous negative result is not universal for $\pi^{CE}$. We show that as long as $q$ is large enough, a linear advantage of $\pi^{CE}$ over $\pi^{OL}$ can be established. Similar to the analysis of $\pi^{CE}$ in the previous section, the result primarily depends on the behavior of $\xhat_{t}$. We show that for large enough $q$, either $\xhat_{t}^2$ or $\bsig - \Sigma_{t}$ spends sufficient time away from zero, leading to linear regret. Define the set 
\begin{equation}\label{eq:deltaset}
    \mathcal{M}_{\delta} := \curl{(\xhat, \Sigma): \xhat^2 \leq \delta, \bar{\Sigma} - \Sigma \leq \delta}.
\end{equation} 
We then have the following proposition.

\begin{proposition}\label{prop:sptightened}
     Fix $\delta,\varepsilon > 0$ with $\delta \leq \bsig$. Assume $\mu_{x} \neq 0$, $D \neq 0$, and suppose \begin{equation}\label{eq:celbcond}
        D^2 +\frac{A^2(\bsig - \delta)^2K^2C^2}{\sigma_z^2 + K^2C^2(\bsig - \delta)\delta} > 2e^{\varepsilon+\gamma_E}
    \end{equation} 
    where $\gamma_{E}$ is Euler's constant. Then
    \begin{equation}\label{eq:problinear}
        \sum_{t=0}^{T}\Pr\paren{\paren{\xhat_{t}, \Sigma_{t}} \not\in  \mathcal{M}_{\delta}} = \Omega(T).
    \end{equation}
    Furthermore, it follows that 
    \begin{equation*}
        R_T(\pi^{CE},\pi^{OL}) =  \Omega( T)
    \end{equation*}
\end{proposition}
The proof of Proposition~\ref{prop:sptightened} is deferred to Appendix~\ref{app:celb}. Proposition~\ref{prop:sptightened} identifies a sufficient condition under which the observations generated by the actions of $\pi^{CE}$ itself remain informative enough for $\pi^{CE}$ to achieve a linear advantage over $\pi^{OL}$.
The set $\Mcal_\delta$ represents beliefs where both the state estimate is small, leading to small CE actions, and the estimation uncertainty is close to the open-loop covariance, indicating little information retained from past observations. 
Under condition~\eqref{eq:celbcond}, $\pi^{CE}$ spends a positive fraction of the horizon, in expectation, outside this set.
Outside this set, either the estimated state is large enough to support informative CE actions, or past observations have already reduced uncertainty about the current state. 
These two sources of benefit yield a cumulative cost advantage over $\pi^{OL}$ that grows linearly with the horizon.

By choosing $\delta$ and $\varepsilon$ sufficiently close to zero, we obtain the following sufficient condition in terms of the information strength $q$.

\begin{theorem}[High Information Regime]\label{thm:celb}
    Assume $\mu_{x} \neq 0$ and $D \neq 0$.
    Define
    \begin{equation*}
        q_4 := 2e^{\gamma_E} - D^2.
    \end{equation*}
    If $q>q_4$, then 
    \begin{equation*}
        R_T(\pi^{CE},\pi^{OL}) = \Omega(T).
    \end{equation*}
\end{theorem}
The proof of Theorem~\ref{thm:celb} is deferred to Appendix~\ref{app:celb}.
The results so far identify both limitations and benefits of $\pi^{CE}$ over $\pi^{OL}$.
When $q<q_1$, Theorem~\ref{prop:clconstant} shows that no policy can improve on $\pi^{OL}$ by more than a constant.
For $\pi^{CE}$, Proposition~\ref{prop:constantscalar} establishes a bounded advantage when $q<q_3$, whereas Theorem~\ref{thm:celb} guarantees a linear advantage when $q>q_4$, under its stated assumptions.
Thus, sufficiently strong observations allow CE to accumulate a benefit from feedback even without active information gathering.
These sufficient conditions leave CE’s behavior between $q_3$ and $q_4$ uncharacterized. 
The next section explores whether actively probing to gather information can improve on the performance of $\pi^{CE}$.

\subsection{Probing Policy} \label{sec:probingpolicy}
To motivate this next section, we start with the following proposition.
\begin{proposition}\label{prop:constantthreshhold}
    We have that $q_1<q_3$.
\end{proposition}
Proposition~\ref{prop:constantthreshhold} leaves open whether $\pi^\star$ can achieve a linear advantage over $\pi^{OL}$ for $q_1<q<q_3$, where the advantage of $\pi^{CE}$ remains bounded. Directly characterizing $\pi^\star$ is challenging because actions affect both the state dynamics and future estimation uncertainty. Previous works~\cite{sattar2025sub,cao2026dual} show that the resulting control optimization can be nonconvex and intractable.
For our purpose, however, it suffices to construct a policy with a linear advantage over $\pi^{OL}$, since the optimal policy performs at least as well.
The policy achieves this through ``probing''; exploratory actions aimed at increasing information. At each time step $t$, the policy operates as follows:
\begingroup
\makeatletter
\renewcommand{\theenumii}{(\alph{enumii})}
\renewcommand{\labelenumii}{\theenumii}
\renewcommand{\p@enumii}{}
\makeatother
\begin{enumerate}
    \item Take a probing action if \emph{all} of the following hold: \vspace{-4pt}
    \begin{enumerate}
        \item \label{item:probe-a} no probing action has been taken recently; 
        \item \label{item:probe-b}$\xhat_t^2$ is small; and
        \item \label{item:probe-c} $\bsig - \Sigma_t$ is small.
    \end{enumerate} 
    \item Otherwise, take the CE action $K_t \xhat_{t}$.
\end{enumerate}
\endgroup

We formalize this policy using stopping times.
Conditions~\ref{item:probe-b} and~\ref{item:probe-c} together say that the belief lies in the set $\Mcal_\delta$ defined in \eqref{eq:deltaset}, and condition~\ref{item:probe-a} is enforced by requiring consecutive probes to be at least $L$ steps apart.\begin{figure*}[!b]
    \centering
    \includegraphics[width=\textwidth]{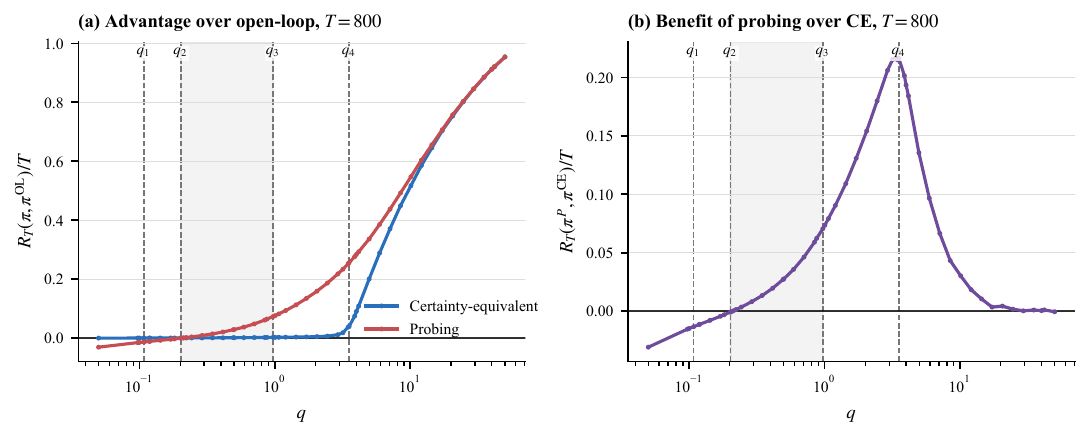}
    \caption{(a) Depicts the advantage of CE and probing over open-loop; (b) depicts the advantage of probing over CE. The advantages are normalized by $T$, hence values bounded away from zero are consistent with linear advantage, while values near zero are consistent with bounded advantage. The observed behavior is consistent with the theoretical regimes: probing begins to outperform open-loop control around $q_2$, while CE's advantage begins around $q_4$. For small $q$, probing incurs costs that outweigh its informational benefits. For large $q$, CE and probing perform similarly. Appendix~\ref{app:exp} provides the experimental setup and additional plots illustrating how these advantages scale with the horizon.}
    \label{fig:q_sweep}
\end{figure*}
Fix $\delta,a,L,t_P>0$, and let $\tau_0 := -L$, and for $k\ge 1$ define
\begin{equation} \label{eq:probestoppingtime}
    \begin{split}
        \tau_k &:= \inf_{t \in \Ical_k} \left\{(\xhat_t,\Sigma_t) \in \Mcal_\delta \right\} \\
        \Ical_k &:= \{t \in \Z\colon \tau_{k-1}+L \le t \le T-t_P-L\}
    \end{split}
\end{equation}
with $\inf\emptyset=\infty$.
Let $\Tcal:=\{\tau_k \colon k \ge 1, \tau_k<\infty\}$. The probing policy is defined as
\begin{equation} \label{eq:probing-policy-def}
    \pi^P_t(H_t) = K_t\xhat_t + a\cdot\mathbf{1}\{t\in\Tcal\}.
\end{equation}
Here, $a$ is the size of each probe, and $L$ is the minimum spacing between probes.
The parameter $\delta$ sets how close the belief must be to $(0, \bar\Sigma)$ for a probe to occur.
The cutoff parameter $t_P$ sets when to stop probing near the end of the horizon.
The following proposition characterizes when $\pi^P$ achieves a linear advantage over open-loop control.

\begin{proposition}\label{prop:cllinear3}
    Define
    \begin{equation*}
        q_2 := \frac{\paren{1-A^2} \paren{1-D^2} }{2\paren{1-AD}}.
     \end{equation*}
     If $q>q_2$,
     then there exist choices $\delta,a,L,t_P > 0$ independent of $T$ such that $\pi^{P}$ satisfies \begin{equation*}
         R_T(\pi^{P}, \pi^{OL}) = \Omega(T).
     \end{equation*}
\end{proposition}
The proof of Proposition~\ref{prop:cllinear3} is deferred to Appendix~\ref{app:cllb}. 
The result addresses the failure mode of $\pi^{CE}$ identified earlier: when the belief enters $\Mcal_\delta$, the CE action is small while estimation uncertainty remains high. 
A probe incurs a cost for deviating from this action, but produces an informative observation and can move the state estimate away from zero.
Subsequent CE actions can then gather further informative observations.
The condition $q>q_2$ ensures that the probe size $a$ and block length $L$ can be chosen so that these expected benefits outweigh the probing cost.
The policy repeats this intervention when small actions and high uncertainty recur.
Outside probing blocks, CE benefits either from informative actions or from information retained from earlier observations. 
These gains accumulate into a cost advantage over open-loop control that grows linearly with the horizon.

Building on  Propositions~\ref{prop:constantscalar} and \ref{prop:cllinear3}, the following theorem establishes an $\Omega(T)$ gap between optimal and CE control when $q_2<q<q_3$.
\begin{theorem}[Separation Regime]\label{thm:separation}
    We have that $q_{2} < q_{3}$ and that for $q \in (q_2,q_3)$, \vspace{-4pt}
    \begin{itemize}
        \item [(i)] $ R_T(\pi^\star, \pi^{OL}) = \Omega(T)$
        \item [(ii)] $R_T(\pi^{CE}, \pi^{OL}) = O(1).$
    \end{itemize}
\end{theorem}
The proof of Theorem~\ref{thm:separation} is deferred to Appendix~\ref{app:cllb}. Figure~\ref{fig:q_sweep} compares the performance of the probing policy and CE as $q$ varies.
Finally, for completeness, we have the following result.
\begin{proposition}\label{prop:cllbub}
    We have that $q_1<q_2$.
\end{proposition}

\section{Conclusion}
We studied the performance gaps between open-loop, certainty-equivalent, and optimal control in linear-quadratic systems with bilinear observations. In the scalar setting, we identified information-strength regimes in which feedback provides only a bounded advantage, CE achieves a linear advantage, or active probing achieves a linear advantage that CE fails to realize. Natural directions for future work include extending the theory beyond the scalar setting, refining the probing policy, and studying adaptive settings in which the system dynamics or observation model must be learned online.

\paragraph{Acknowledgements.} MF was supported in part by awards 
NSF TRIPODS II DMS-2023166, 
CCF 2212261, CCF 2312775, and the Moorthy Professorship at the UW. She also works at Amazon as an Amazon Scholar. 
SD was supported in part by NSF CCF-2312774, NSF OAC-2311521, NSF IIS-2442137, and an AI2050 Early Career Fellowship program at Schmidt Sciences.

\bibliography{References}

\newpage
\appendix

\addtocontents{toc}{\protect\setcounter{tocdepth}{2}}
\renewcommand{\contentsname}{Contents of the Appendix}
\tableofcontents

\section{Dynamic programming and Policies} \label{app:dp-policies}

In this section, we connect the policy definitions in Section~\ref{sec:prelim} to their explicit forms in the linear-quadratic model.
We first derive the Riccati recursion from the \emph{fully observed Bellman equation} and obtain the certainty-equivalent controller by substituting the conditional mean for the state.
We then show that the optimal open-loop policy is obtained from the same Riccati recursion applied to the deterministic mean trajectory. 
Finally, we prove the filtering results used to describe the evolution of the belief state.
Throughout this section, we work in the general setting
of Section~\ref{sec:prelim}.

\subsection{Riccati Recursion and Certainty-Equivalent Control}

Let $\Qtil_t$ be the fully observed state-action cost-to-go function
defined in Section~\ref{sec:prelim}, and define
\[
    \Vtil_t(x)
    = \min_{u\in\mathbb R^p}\Qtil_t(x,u),
    \qquad 
    \Vtil_{T+1}(x):=0.
\]
The Bellman equation becomes
\[
    \Qtil_t(x,u)
    =
    c(x,u)
    +
    \E_{w\sim\Ncal(0,\Sigma_w)}
    \big[
        \Vtil_{t+1}(Ax+Bu+w)
    \big].
\]

We show by backward induction that
\begin{equation}
    \Vtil_t(x)
    =
    x^\top P_t x
    +
    \sum_{s=t+1}^{T}
    \tr(P_{s}\Sigma_w),
    \label{eq:fully-observed-value}
\end{equation}
where $P_t\succeq0$ and $P_{T+1}=0$.
At $t=T+1$, the sum is empty and the identity holds.
Suppose it holds at $t+1$.
Since $w$ has zero mean and covariance $\Sigma_w$,
\begin{align*}
    \Qtil_t(x,u)
    &=
    x^\top Qx+u^\top Ru+2u^\top Sx
    +(Ax+Bu)^\top P_{t+1}(Ax+Bu)
    + \tr(P_{t+1}\Sigma_w) + \sum_{s=t+2}^T \tr(P_s\Sigma_w)
    \\
    &=
    x^\top(Q+A^\top P_{t+1}A)x
    +2u^\top(B^\top P_{t+1}A+S)x
    +u^\top(R+B^\top P_{t+1}B)u
    +\sum_{s=t+1}^{T}\tr(P_{s}\Sigma_w).
\end{align*}
Define
\begin{align}
    N_t &:=R+B^\top P_{t+1}B, 
    \label{eq:dp-N} \\
    K_t &:=-N_t^{-1}(B^\top P_{t+1}A+S),
    \label{eq:dp-K} \\
    P_t &:=A^\top P_{t+1}A+Q-K_t^\top N_tK_t.
    \label{eq:dp-P}
\end{align}
Since $P_{t+1}\succeq0$ and $R\succ0$, we have
$N_t\succeq R\succ0$.
Completing the square gives
\begin{equation}
    \Qtil_t(x,u) =
    x^\top P_t x +(u-K_tx)^\top N_t(u-K_tx) +\sum_{s=t+1}^{T}\tr(P_{s}\Sigma_w).
    \label{eq:fully-observed-square}
\end{equation}
Thus, the unique minimizing action is $u=K_tx$, and minimizing
over $u$ gives \eqref{eq:fully-observed-value}.
Moreover,
\[
    x^\top P_t x =
    \min_u \curl{c(x,u) + (Ax+Bu)^\top P_{t+1} (Ax+Bu)} > 0, \quad \text{for every } x\in\R^n.
\]
This is because the conditions $Q \succeq 0$, $R \succ 0$, and $Q-S^\top R^{-1}S \succ 0$ imply $c(x,u) > 0$. 
Hence $P_t\succ 0$, completing the induction.
In fact, the proof implies that $P_t \succ 0$ for $t=0,1,\dots,T$ and $P_{T+1} =0$.

Equations \eqref{eq:dp-N}--\eqref{eq:dp-P}, together with
$P_{T+1}=0$, are the finite-horizon Riccati recursion.
By the definition of certainty-equivalent control in Section~\ref{sec:prelim},
we evaluate the fully observed minimizing action at
$\hat x_t=\mathbb E[x_t\mid H_t]$. Therefore,
\[
    \pi_t^{CE}(H_t)=K_t\hat x_t.
\]

\paragraph{Stationary quantities.}
Lemma~\ref{lem:stationary} establishes convergence of $P_t$ to $P$ as the
remaining horizon grows.
Since $N_t\succeq R\succ0$, continuity of matrix inversion
also gives convergence of $N_t$ and $K_t$.
Their limits satisfy
\begin{align*}
    N&=R+B^\top PB,\\
    K&=-N^{-1}(B^\top PA+S),\\
    P&=A^\top PA+Q-K^\top NK.
\end{align*}
The last equation is the discrete algebraic Riccati equation (DARE).
For the closed-loop dynamics, we write $D_t=A+BK_t$ and $D=A+BK$.

\paragraph{Closed-loop stability}
In addition to the convergence result in Lemma~\ref{lem:stationary}, the gain $K$ stabilizes the closed-loop dynamics, $\rho(D)<1$, under some assumptions.
This follows from~\citep[Theorem~5]{hager1976convergence} where we apply after the change of variables $u_t \leftarrow u_t+R^{-1}Sx_t$ that removes the state-action cross term.
We state the corresponding lemma for completeness where we defer the proof to \cite{hager1976convergence}.

\begin{lemma} \label{lem:closed-loop-stability}
    If $(A,B)$ is stabilizable\footnote{This means that there exists a matrix $L$ such that $\rho(A+BL)<1$. If $\rho(A)<1$, this condition holds automatically by taking $L=0$.}, $R\succ0$, and $Q-S^\top R^{-1}S \succ 0$, then $\rho(D)<1$.
\end{lemma}

\subsection{Optimal Open-loop Control}

Consider any deterministic action sequence $(u_0,\ldots,u_T)$
chosen before interaction, and let $m_t:=\mathbb E[x_t]$.
The state dynamics imply
\begin{equation*}
    m_{t+1}=Am_t+Bu_t \quad \text{with } \; m_0=\mu_x.
\end{equation*}
Because the actions are deterministic, the unconditional state covariance satisfies
\begin{equation*}
    \mathrm{Cov}(x_{t+1}) = A\mathrm{Cov}(x_t)A^\top + \Sigma_w.
\end{equation*}
Since $\operatorname{Cov}(x_0)=\bar\Sigma$ and $\bar\Sigma=A\bar\Sigma A^\top+\Sigma_w$, it follows that
\begin{equation*}
    \mathrm{Cov}(x_t) = \bsig \quad \text{for every } t.
\end{equation*}
This is the unconditional state covariance, distinct from the conditional covariance $\Sigma_t=\mathrm{Cov}(x_t\mid H_t)$.
Consequently, the expected stage cost is now
\begin{equation*}
    \E[c(x_t,u_t)] = c(m_t,u_t) + \tr(Q\bsig),
\end{equation*}
so that the expected cumulative cost of the open-loop policy is
\begin{equation}
    J_T(\pi)
    = \sum_{t=0}^{T}c(m_t,u_t) +(T+1)\tr(Q\bsig).
    \label{eq:ol-mean-cost}
\end{equation}
The covariance term is independent of the action sequence.
Thus, optimizing the open-loop policy reduces to a deterministic
linear-quadratic control problem for the mean trajectory.

To solve this problem, we use a similar technique to decompose the term $c(m_t,u_t)$:
\begin{equation} \label{eq:cdecomp}
    c(m_t,u_t) = (u_t-K_tm_t)^\top N_t (u_t-K_tm_t) - m_{t+1}^\top P_{t+1}m_{t+1} +m_t^\top P_t m_t.
\end{equation}
Summing over $t$, the quadratic terms telescope, and since $P_{T+1}=0$ and $m_0=\mu_x$, we obtain
\begin{equation*}
    \sum_{t=0}^T c(m_t,u_t) = \mu_x^\top P_0 \mu_x + \sum_{t=0}^T (u_t-K_tm_t)^\top N_t (u_t-K_tm_t)
\end{equation*}
Substituting into \eqref{eq:ol-mean-cost} yields
\begin{equation*}
    J_T(\pi) =
    \mu_x^\top P_0\mu_x +(T+1)\tr(Q\bsig)
    + \sum_{t=0}^{T} (u_t-K_tm_t)^\top N_t(u_t-K_tm_t).
\end{equation*}
Each term in the final sum is nonnegative because $N_t\succ0$.
The unique minimizing action sequence therefore satisfies
\[
    u_t^{OL}=K_tm_t,
    \qquad
    m_{t+1}=(A+BK_t)m_t=D_tm_t,
    \qquad
    m_0=\mu_x.
\]
This sequence is deterministic and can be computed before
interaction, so it defines an admissible open-loop policy.
In particular,
\[
    \pi_t^{OL}(H_t)=K_tm_t,
    \qquad
    J_T(\pi^{OL})
    =
    \mu_x^\top P_0\mu_x
    +(T+1)\operatorname{tr}(Q\bar\Sigma).
\]

The certainty-equivalent and open-loop policies therefore use the same gain $K_t$, but apply it to different quantities.
The policy $\pi^{CE}$ uses the conditional mean $\xhat_t$, which incorporates
past observations, whereas the policy $\pi^{OL}$ uses the deterministic mean $m_t$.

\subsection{Riccati Equation Tools}

In this section, we provide technical lemmas related to Riccati recursion that are useful to prove the main results.

\begin{lemma}\label{lem:Gammiden} 
    Define $G_t = Q + A^2G_{t+1}$ with $G_{T+1} = 0$. Then, we have $\Gamma_{t} = G_t -P_t$.
\end{lemma}
\begin{proof}
    Recall $\Gamma_t$ defined in \eqref{eq:Gamma_t_def}. 
    Since $G_{T+1}=P_{T+1}=0$, we trivially have that $\Gamma_{T+1} = G_{T+1} - P_{T+1} = 0$. 
    Assuming $\Gamma_{t+1} = G_{t+1} -P_{t+1}$, we get
    \begin{align*}
        \Gamma_{t} \;=\; N_tK_t^2 + A^2\Gamma_{t+1} 
        \;=\; A^2 P_{t+1} - P_t+  Q +A^2\Gamma_{t+1} 
        \;=\; A^2 G_{t+1} + Q - P_t 
        \;=\; G_t - P_t.
    \end{align*}
    The proof follows by induction.
\end{proof}
\begin{lemma}\label{lem:Duniform} 
    Recall that $D_t = A+BK_t$. For each $t$, 
    \begin{equation}
        D_t = \frac{AR-BS}{N_t} 
        \quad \text{and} \quad
        |D_t| \leq |D_T| = \frac{|AR - BS|}{R}.
    \end{equation}
\end{lemma}
\begin{proof}
    First, since $K_T = -\frac{S}{R}$ as $P_{T+1}=0$, we get 
    \begin{equation*}
        D_T = A+BK_T = \frac{AR-BS}{R}
        \quad \text{and} \quad
        D_t = A+BK_t = \frac{AR-BS}{R+B^2P_{t+1}}.
    \end{equation*}
    Then the result is clear.
\end{proof}

\begin{lemma}\label{lem:stationary}
    There exists $C_{P} < \infty$ and $\rho \in (0,1)$ such that for all $T,t$ \begin{equation*}
        \| P_{t} - P\| \leq C_{P}\rho^{T-t}.
    \end{equation*}
\end{lemma}
\begin{proof}
    Denote $P_t^{(T)}$ the solution to the DARE with time horizon $T$. From the proof of Corollary 2 of \cite{hager1976convergence} we have that there exists $C_{P}$ and $0<\rho< 1$ such that \begin{equation}
        \norm{P_0^{(T)} - P} \leq C_{P} \rho^{T}
    \end{equation}
    Observing that $P_0^{(T-t)} = P_{t}^{(T)}$, it follows that \begin{equation}
        \norm{P_t^{(T)} - P} =  \norm{P_0^{(T-t)} - P} \leq C_{P} \rho^{T-t}.
    \end{equation}
    This completes the proof.
\end{proof}
We introduce two corollaries of Lemma~\ref{lem:stationary}, the second of which will play a key role in our regret analyses. 
Define
\begin{equation*}
    \Gamma = \frac{K^2N}{1-A^2}
\end{equation*}
which is a stationary counterpart of $\Gamma_t$ defined in \eqref{eq:Gamma_t_def}.
We have the following two lemmas.

\begin{lemma}\label{lem:gammaconv}
    There exists $C_\Gamma < \infty$ and $\eta \in (0,1)$ such that, for all $T,t$,
    \begin{equation*}
        |\Gamma_{t} - \Gamma| \leq C_\Gamma \, \eta^{T-t}.
    \end{equation*}     
\end{lemma}

\begin{proof}
    By Lemma~\ref{lem:Gammiden}, we have $\Gamma_{t} = G_t -P_t$. 
    The recursion of $G_t$ in Lemma~\ref{lem:Gammiden} implies that
    \begin{equation*}
        G_t = \frac{Q(1-A^{2(T-t+1)})}{1-A^2}.
    \end{equation*}
    Hence, we have
    \begin{equation*}
        \Gamma_{t+1} = G_{t+1} - P_{t+1} = \frac{Q\paren{1 - A^{2(T-t)}}}{1-A^2} - P_{t+1}.
    \end{equation*}
    Furthermore 
    \begin{equation*}
        \Gamma = \frac{NK^2}{1-A^2} = \frac{Q+PA^2-P}{1-A^2} = \frac{Q}{1-A^2} - P.
    \end{equation*}
    So
    \begin{align*}
        |\Gamma_{t+1} - \Gamma| &=\left|\frac{Q\paren{1 - A^{2(T-t)}}}{1-A^2} - P_{t+1}-\frac{Q}{1-A^2} + P \right| \\
        &\leq \frac{Q}{1-A^2}\left|A^{2(T-t)}\right| + |P_{t+1} - P| \\
        &\leq \frac{Q}{1-A^2}\left|A^{2}\right|^{(T-t)} + C_P \rho^{(T-t - 1)} \tag{Lemma~\ref{lem:stationary}}
    \end{align*}
    for some $C_P<\infty$ and $\rho \in (0,1)$.
    Choosing $C_{\Gamma} = 2 \max\curl{\frac{Q A^2}{1-A^2}, C_P}$  
    and $\eta = \max\curl{\rho, A^2}$ finishes the proof.
\end{proof}

\begin{lemma}\label{lem:riccaticont}
    Let  $h: \R^2 \rightarrow \R$ be continuous at $(P,\Gamma)$. 
    If $a<h(P, \Gamma) < b$, then there exists $t_0 \in \mathbb{N}$ such that 
    \begin{equation*}
        \frac{h(P, \Gamma) + a}{2}< h(P_t, \Gamma_t)< \frac{h(P, \Gamma)+b}{2}
    \end{equation*}
    for all $t,T$ satisfying $t \leq T-t_0$.
\end{lemma}

\begin{proof}
    By the continuity of $h$, there exists $\delta$ such that  
    \begin{equation}
        \left|h(P,\Gamma) - h(P', \Gamma')\right| < \frac{\min\big\{h(P,\Gamma) -a,\;  b -h(P,\Gamma)\big\}}{2}
    \end{equation}
    whenever $|P - P'|+|\Gamma - \Gamma'| < \delta$.
    That is if $|P - P'|+|\Gamma - \Gamma'| < \delta$ then \begin{align*}
        h(P', \Gamma') 
        \;\leq\;
        h(P, \Gamma) + \left|h(P,\Gamma) - h(P', \Gamma')\right| 
        \;<\; h(P, \Gamma) + \frac{b-h(P,\Gamma)}{2} 
        \;=\; \frac{h(P, \Gamma)+b}{2},
    \end{align*}
    and 
    \begin{align*}
        h(P', \Gamma') 
        \;\geq\; h(P, \Gamma) - \left|h(P,\Gamma) - h(P', \Gamma')\right|
        \;>\; h(P, \Gamma) - \frac{h(P,\Gamma)-a}{2} 
        \;=\; \frac{h(P, \Gamma)+a}{2}.
    \end{align*}
    By Lemmas~\ref{lem:stationary} and \ref{lem:gammaconv}, choose $t_0$ such that $C_P \rho^{t_0} + C_\Gamma \eta^{t_0}<\delta$
    where $C_P,C_\Gamma>0$ and $\rho, \eta \in(0,1)$ satisfy
    \begin{equation*}
        |P_t-P|\le C_P \rho^{T-t} \le C_P \rho^{t_0} 
        \quad \text{and} \quad
        |\Gamma_t -\Gamma|\le C_\Gamma \eta^{T-t} \le C_\Gamma \eta^{t_0}.
    \end{equation*}
    Applying the continuity estimate above with $(P',\Gamma')=(P_t,\Gamma_t)$ gives
    \begin{equation*}
        \frac{h(P,\Gamma)+a}{2} < h(P_t,\Gamma_t) < \frac{h(P,\Gamma)+b}{2}
    \end{equation*}
    as claimed.
\end{proof}

\begin{remark}
    Note that since $N_t,K_t, D_t$ are continuous mappings of $P_{t+1}$, Lemma~\ref{lem:riccaticont} holds for continuous functions of $N_t,K_t,D_t$ as well.
\end{remark}

\subsection{Proof of Theorem~\ref{thm:lqg}}

In the proof of Theorem~\ref{thm:lqg}, we assume that $\rho(A)<1$, as we assumed for our bilinear setting.

\begin{proof}
By the completion of squares used by Eq.~\eqref{eq:cdecomp}
\begin{equation}
    c(x,u) = (u-K_tx)^\top N_t (u-K_tx) - (Ax+Bu)^\top P_{t+1}(Ax+Bu) +x^\top P_t x.
\end{equation}
Hence 
\begin{equation}
    c(x_t,u_t) = (u_t-K_tx_t)^\top N_t (u_t-K_tx_t) - (Ax_t+Bu_t)^\top P_{t+1}(Ax_t+Bu_t) +x_t^\top P_t x_t.
\end{equation}
Since $x_{t+1} \sim \mathcal{N}\paren{Ax_{t} + Bu_t, \Sigma_w }$ we have 
\begin{equation}
    c(x_t,u_t) = (u_t-K_tx_t)^\top N_t (u_t-K_tx_t) - \E[x_{t+1}^\top P_{t+1} x_{t+1}] +x_t^\top P_t x_t + \tr{P_{t+1}\Sigma_w}.
\end{equation}
Summing and applying the tower rule, we have \begin{equation}
    J_T(\pi) = C_T + \sum_{t=0}^T \mathbb{E}_{\pi} \squ{\paren{u_t - K_tx_t} ^\top N_t\paren{u_t - K_tx_t}}
\end{equation}
where
\begin{equation}
    C_T = \E\squ{x_0^\top P_0x_0} + \sum_{t=0}^T\tr\paren{P_{t+1} \Sigma_w}.
\end{equation}

Observe that \begin{align*}
     \sum_{t=0}^T \mathbb{E}_{\pi^{CE}} \squ{\paren{u_t - K_tx_t} ^\top N_t\paren{u_t - K_tx_t}} 
    &= \sum_{t=0}^T \mathbb{E}_{\pi^{CE}} \squ{ \paren{\xhat_t - x_t}^\top K_t^\top N_tK_t\paren{\xhat_t - x_t}} \\
    &= \sum_{t=0}^T \tr \paren{K_t ^\top N_tK_t \Sigma_t^{CE}}.
\end{align*}
Similarly,
\begin{align*}
     \sum_{t=0}^T \mathbb{E}_{\pi^{OL}} \squ{\paren{u_t - K_tx_t} ^\top N_t\paren{u_t - K_tx_t}} 
    &= \sum_{t=0}^T \mathbb{E}_{\pi^{OL}} \squ{ \paren{m_t - x_t}^\top K_t^\top N_tK_t\paren{m_t - x_t}} \\
    &= \sum_{t=0}^T \tr \paren{K_t ^\top N_tK_t \bsig}
\end{align*}
Hence
    \begin{equation}
        R_T(\pi^{CE}, \pi^{OL}) = \sum_{t=0}^{T} \tr \paren{K_t ^\top N_tK_t \paren{\bsig - \Sigma_t^{CE}}}
    \end{equation}
Observe that since observations are action-independent, $\Sigma_t^{CE}$ is deterministic, and it converges to $\Sigma_\infty$ under stability assumption, $\rho(A)<1$. 
Moreover, continuity argument in Lemma~\ref{lem:riccaticont} implies that, as long as $\chi := \tr \paren{K ^\top NK \paren{\bsig - \Sigma_\infty}} >0$, we can choose $s_0$ and $s_1$, independent of $T$, such that 
\begin{equation*}
    \tr\big(K_t^\top N_tK_t(\bsig-\Sigma_t^{CE})\big) \ge \frac{\chi}{2}
\end{equation*}
whenever $s_0 \le t \le T-s_1$.
Therefore, we get $R_T(\pi^{CE}, \pi^{OL}) = \Omega(T)$.
\end{proof}

\begin{remark} \normalfont
    In the proof of Theorem~\ref{thm:lqg}, we assumed that $A$ is stable for the sake of clear comparison. However, under stabilizability and detectability assumptions, a similar result can be shown when $A$ is unstable. The stability assumption ensures that the open-loop covariance remains bounded. Under stabilizability and detectability, unstable systems can exhibit a faster-growing advantage over open-loop control. Here we restrict attention to stable systems.
\end{remark}

\section{Bilinear Kalman Filtering Tools} \label{app:kalmantools}

In this section, we provide technical lemmas related to Kalman filtering that are useful to prove the main results.
Throughout the remaining sections, define $\Fcal_t=\sigma(H_t)$, the $\sigma$-algebra generated by the history $H_t$.

\paragraph{Proof of Lemma~\ref{lem:kalmansecond}}
\begin{proof}
Define the innovation and its conditional variance by 
\begin{equation*}
    e_t := y_t-u_t^\top C\hat x_t = u_t^\top C(x_t-\xhat_t) + z_t, \qquad s_t := \sigma_z^2+u_t^\top C\Sigma_t C^\top u_t
\end{equation*}
respectively.
Since $u_t$ is $\Fcal_t$-measurable, the posterior state distribution and the observation
model imply
\begin{equation*}
    \E[e_t \mid \Fcal_t] = 0 
    \quad \text{and} \quad
    \E[e_t^2 \mid \Fcal_t] = s_t.
\end{equation*}
By \eqref{eq:kalmanrecursion}, we have
\begin{equation*}
    \xhat_{t+1}
    =A\xhat_t+Bu_t+\frac{A\Sigma_t C^\top u_t}{s_t}\,e_t.
\end{equation*}
All coefficients in this expression are $\Hcal_t$-measurable.
Therefore,
\begin{align*}
    \mathbb E[\hat x_{t+1}\hat x_{t+1}^\top\mid \Fcal_t]
    &=(A\hat x_t+Bu_t)(A\hat x_t+Bu_t)^\top +
    \frac{A\Sigma_t C^\top u_tu_t^\top C\Sigma_t A^\top}
         {s_t^2}\,
    \mathbb E[e_t^2\mid \Fcal_t] \\
    &=(A\xhat_t+Bu_t)(A\xhat_t+Bu_t)^\top
      +F(\Sigma_t,u_t),
\end{align*}
where the last equality follows from the definition of $F$.
\end{proof}

\begin{lemma}\label{lem:covariance}
For $t \geq 0$
    \begin{equation*}
       \Sigma_w \preceq \Sigma_{t} \preceq \bsig.
    \end{equation*}
\end{lemma}
\begin{proof}
    For the base case, observe that \begin{align*}
        \Sigma_w \preceq A \bsig A^\top + \Sigma_w 
        = \bsig 
        = \Sigma_0.
    \end{align*}
    Let $t \geq 0$. Assume $\Sigma_t \preceq \bsig$. For the upper bound, positive semidefiniteness of $F(\Sigma,u)$, the induction hypothesis, and the definition of $\bsig$ imply 
    \begin{align*}
        \Sigma_{t+1} 
        \;=\; 
        A \Sigma_t A^\top + \Sigma_w - F(\Sigma_t , u_t) 
        \;\preceq\; 
        A \Sigma_t A^\top + \Sigma_w 
        \;\preceq\; 
        A \bsig A^\top + \Sigma_w 
        \;=\; \bsig.
    \end{align*}
    For the lower bound we aim to show 
    \begin{equation*}
        \Sigma_{t+1} =A \Sigma_t A^\top + \Sigma_w - F(\Sigma_t , u_t) = A \paren{\Sigma_t -\frac{\Sigma_t C^\top  u_tu_t^\top  C\Sigma_t }{\sigma_z^2 + u_t^\top  C \Sigma_t C^\top  u_t}}A^\top + \Sigma_w \succeq \Sigma_w.
    \end{equation*}
    It suffices to show 
    \begin{equation*}
        \frac{\Sigma_t C^\top  u_tu_t^\top  C\Sigma_t }{\sigma_z^2 + u_t^\top  C \Sigma_t C^\top  u_t} \preceq \Sigma_t,
        \quad \text{i.e.,} \quad 
        \frac{\Sigma_t^{1/2} C^\top  u_tu_t^\top  C\Sigma_t^{1/2} }{\sigma_z^2 + u_t^\top  C \Sigma_t C^\top  u_t} \preceq I.
    \end{equation*}
    Observe that 
    \begin{align*}
        \norm{\frac{\Sigma_t^{1/2} C^\top  u_tu_t^\top  C\Sigma_t^{1/2} }{\sigma_z^2 + u_t^\top  C \Sigma_t C^\top  u_t}} 
        \; = \; \frac{u_t^\top  C \Sigma_t C^\top  u_t}{\sigma_z^2 + u_t^\top  C \Sigma_t C^\top  u_t} 
        \;\leq\;  1.
    \end{align*}
    The result follows from the fact that, for $M \succeq 0$ we have $\|M\| \leq 1 \iff M \preceq I$.
\end{proof}

\begin{lemma}\label{lem:gaussiankalman}
    For each $t$, define
    \[
    Z_t = \frac{y_t-u_tC \hat x_t}{\sqrt{\sigma_z^2 + C^2\Sigma_tu_t^2}}
    \]
    where $u_t$ is $\Fcal_t$-measurable.
    Then $Z_0,Z_1,\dots$ are independent standard Gaussian random variables, and each $Z_t$ is independent of $\Fcal_t$. 
    Moreover, almost surely,
    \begin{equation} \label{eqn:innovation-form}
        \hat x_{t+1} = A\hat x_t + Bu_t + \frac{A\Sigma_tCu_t}{\sqrt{\sigma_z^2 + C^2\Sigma_tu_t^2}}\,Z_t .
    \end{equation}
    In particular, 
    \[
    \hat x_{t+1} \mid \Fcal_t \sim \Ncal\paren{A\hat x_t + Bu_t, F(\Sigma_t,u_t)}.
    \]
\end{lemma}
\begin{proof}
    Since $y_t = u_tCx_t + z_t$, the innovation can be written as
    \[
        y_t - u_tC\hat x_t = u_tC\,(x_t - \hat x_t) + z_t .
    \]
    Condition on $\mathcal F_t$. The action $u_t$ is $\mathcal F_t$-measurable, 
    $x_t - \hat x_t \mid \mathcal F_t \sim \mathcal N(0, \Sigma_t)$.
    The measurement noise $z_t \sim \mathcal N(0, \sigma_z^2)$ is independent of $(x_t, \mathcal F_t)$.
    Hence
    \[
        y_t - u_tC\hat x_t \mid \mathcal F_t \sim \mathcal N\big(0,\ \sigma_z^2 + C^2\Sigma_t u_t^2\big),
    \]
    and therefore $Z_t \mid \mathcal F_t \sim \mathcal N(0,1)$.
    This conditional law does not depend on $\mathcal F_t$, so $Z_t$ is independent of $\mathcal F_t$
    and $Z_t \sim \mathcal N(0,1)$.
    Since $y_t$, $u_t$, $\hat x_t$ and $\Sigma_t$ are all $\mathcal F_{t+1}$-measurable, so is $Z_t$.
    Finally, substituting $y_t - u_tC\hat x_t = \sqrt{\sigma_z^2 + C^2\Sigma_t u_t^2}\,Z_t$
    into the recursion \eqref{eq:kalmanrecursion} gives
    \[
        \hat x_{t+1}
        = A\hat x_t + Bu_t + \frac{A\Sigma_tCu_t}{\sigma_z^2 + C^2\Sigma_tu_t^2}\sqrt{\sigma_z^2 + C^2\Sigma_tu_t^2}\;Z_t ,
    \]
    which simplifies to \eqref{eqn:innovation-form}. 
\end{proof}

\begin{lemma} \label{lem:multiplicative}
    Under $\pi^{CE}$, define
    \begin{equation*}
        M_t = D_t + \frac{A\Sigma_t CK_t}{\sigma_z^2 + K_t^2C^2\Sigma_t \hat x_t^2} \paren{y_t - K_t C\hat x_t^2}, \qquad q_t = \frac{|ACK_t|\Sigma_t}{\sqrt{\sigma_z^2 + K_t^2C^2 \Sigma_t \hat x_t^2}}.
    \end{equation*}
    Then, 
    \begin{enumerate}
        \item[(i)] $\hat x_{t+1} = M_t \hat x_t$ for all $t$.
        \item[(ii)] Conditionally on $\Fcal_t$, $M_t \overset{d}= D_t + q_t Z$, where $Z \sim \Ncal(0,1)$ is independent of $\Fcal_t$.
    \end{enumerate}
\end{lemma}
\begin{proof} \
    \begin{enumerate}
        \item[(i)] Under $\pi^{CE}$, we have $y_t - u_tC\hat x_t = y_t - K_tC\hat x_t^2$.
        Substituting into the recursion \eqref{eq:kalmanrecursion},
        \[
            \hat x_{t+1}
            = (A + BK_t)\hat x_t
              + \frac{A\Sigma_tCK_t\hat x_t}{\sigma_z^2 + K_t^2C^2\Sigma_t\hat x_t^2}\big(y_t - K_tC\hat x_t^2\big)
            = M_t\hat x_t .
        \]
        \item[(ii)] By the definition of $Z_t$ in Lemma~\ref{lem:gaussiankalman} and $u_t = K_t \hat x_t$, 
        \[
        M_t = D_t + \frac{A\Sigma_tCK_t}{\sqrt{\sigma_z^2+K_t^2C^2\Sigma_t \hat x_t^2}} Z_t
        \]
        where $Z_t \sim \Ncal(0,1)$ is independent of $\Fcal_t$.
        The coefficient of $Z_t$ is $\Fcal_t$-measurable and its square equals $q_t^2$. 
        Therefore,
        \[
        M_t \mid \Fcal_t \sim \Ncal(D_t,q_t^2),
        \]
        which is the conditional distribution of $D_t+q_tZ$ for $Z\sim\Ncal(0,1)$ independent of $\Fcal_t$.
    \end{enumerate}
\end{proof}

\begin{lemma}\label{lem:kalmanbound}
    Under $\pi^{CE}$, there exists an integer $s\ge 1$, independent of $T$, such that  \begin{equation}
        \E_{\pi^{CE}}[\xhat_{t}^2] \leq C_{CE}, 
        \quad \text{with } \; C_{CE}:= M^{2s} \paren{\mu_{x}^2 + A^2 \bsig  \paren{ \frac{1}{1- (D')^{2}} + s}}
    \end{equation}
    for every horizon $T$ and $0\le t\le T$, where 
    \begin{equation}
        D' = \frac{1+D^2}{2} 
        \quad \text{and} \quad
        M = \max\curl{1, \frac{|AR - BS|}{R}}.
    \end{equation}
\end{lemma}
\begin{proof}
    Write $b_t = \E_{\pi^{CE}}[\xhat_t^2]$, so that $b_0 = \mu_x^2$.
    By Lemma~\ref{lem:closed-loop-stability}, we have $|D|<1$, and hence,
    \begin{equation*}
        |D|<D' = \frac{1+D^2}{2}<1.
    \end{equation*}
    From Lemma~\ref{lem:riccaticont}, choose an integer $s \ge 1$, independent of $T$, such that $|D_t| \le D'$ whenever $t \le T-s$.
    Also, Lemma~\ref{lem:Duniform} gives $|D_t|\le M$ for every $t$.

    By Lemma~\ref{lem:covariance}, we have
    \begin{equation*}
        0 \;\le\; F(\Sigma_t,K_t\xhat_t) 
        \;=\; \frac{A^2K_t^2C^2\Sigma_t^2\xhat_t^2}{\sigma_z^2 + K_t^2C^2\Sigma_t \xhat_t^2} 
        \;\le\; A^2\Sigma_t 
        \;\le\; A^2 \bsig.
    \end{equation*}
    Therefore, Lemma~\ref{lem:kalmansecond} and the tower property give
    \begin{equation*}
        b_{t+1} = D_t^2b_t + \E_{\pi^{CE}}[F(\Sigma_t,K_t\xhat_t)] \le D_t^2 b_t + A^2\bsig.
    \end{equation*}
    Since $|D_t|\le D'$ for $0\le t \le T-s$, iteration gives
    \begin{equation*}
        b_t \le (D')^{2t}\mu_x^2 + A^2 \bsig \sum_{i=0}^{t-1}(D')^{2i} \le \mu_x^2 + \frac{A^2\bsig}{1-(D')^2} =: C_0
    \end{equation*}
    for $0\le t\le T-s$.

    There are at most $s$ remaining updates, so using $|D_t|\le M$ and $M\ge1$, for $0\le j \le s$ we obtain
    \begin{align*}
        b_{T-s+j} &\le M^{2j} b_{T-s} + A^2\bsig \sum_{i=0}^{j-1}M^{2i} \\
        &\le M^{2s} \paren{C_0 + sA^2\bsig} \\
        &= M^{2s} \paren{\mu_x^2 + A^2 \bsig \paren{\frac{1}{1-(D')^2}+s}} = C_{CE}.
    \end{align*}
    Since $C_0 \le C_{CE}$, the bound holds for every $0 \le t \le T$.
\end{proof}

\section{Regret decomposition}\label{app:regret}

This section proves the regret decompositions in Proposition~\ref{prop:regretdecomp} and~\ref{prop:regretdecomp2}.
Define the conditional expected stage cost by
\begin{equation}
    \bar{c}(\xhat, \Sigma, u) := Q \xhat^2+ Q\Sigma+Ru^2+2uS\xhat.
\end{equation}
Since $\xhat_t = \E[x_t \mid \Fcal_t]$ and $\Sigma_t = \mathrm{Var}(x_t \mid \Fcal_t)$, we have
\begin{equation*}
    \E[c(x_t,u_t) \mid \Fcal_t] = \bar c(\xhat_t,\Sigma_t,u_t).
\end{equation*}
The proof of Proposition~\ref{prop:regretdecomp} uses an auxiliary quantity $V_t = P_t\xhat_t^2 + G_t\Sigma_t$ to separate the cumulative cost into a policy-independent term and the accumulated quantities $r_t$.
Lemma~\ref{lem:rewarddecomp} provides a key identity for the separation, while Lemma~\ref{lem:olcost} shows that the accumulated quantities $r_t$ have zero expectation under the open-loop policy.
Proposition~\ref{prop:regretdecomp2} simply follows from Lemma~\ref{lem:covariancediff}, which shows that covariance reductions and uncertainty reductions give the same total weighted benefit.
Figure~\ref{fig:dependencies-c} describes the proof dependencies for Section~\ref{app:regret}.

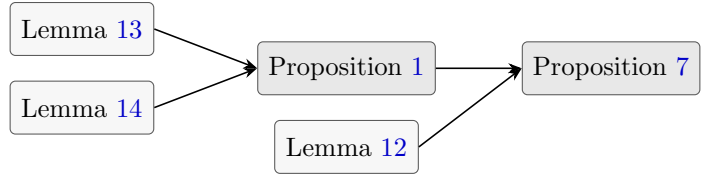
\begin{figure}[htbp]
    \centering
    \begin{tikzpicture}[
        y=0.65cm,
        dependency result/.style={
            draw=black!65, rounded corners=2pt,
            minimum width=1.9cm, minimum height=0.7cm,
            inner sep=4pt, font=\small, align=center
        },
        dependency lemma/.style={dependency result, fill=black!3},
        dependency proposition/.style={dependency result, fill=black!9},
        dependency theorem/.style={dependency result, fill=black!14, line width=0.8pt},
        dependency uses/.style={-{Stealth[length=1.8mm]}, line width=0.6pt}
    ]
        \node[dependency lemma] (n0) at (0,0.8) {Lemma~\ref{lem:rewarddecomp}};
        \node[dependency lemma] (n1) at (0,-0.8) {Lemma~\ref{lem:olcost}};
        \node[dependency proposition] (n2) at (3.5,0) {Proposition~\ref{prop:regretdecomp}};
        \node[dependency lemma] (n3) at (3.5,-1.6) {Lemma~\ref{lem:covariancediff}};
        \node[dependency proposition] (n4) at (7,0) {Proposition~\ref{prop:regretdecomp2}};

        \draw[dependency uses] (n0.east) -- (n2.west);
        \draw[dependency uses] (n1.east) -- (n2.west);
        \draw[dependency uses] (n2.east) -- (n4.west);
        \draw[dependency uses] (n3.east) -- (n4.west);
    \end{tikzpicture}
    \caption{Main proof dependencies for Section~\ref{app:regret}.
    An arrow points from a result to a result whose proof uses it.
    Common preliminary results are omitted.}
    \label{fig:dependencies-c}
\end{figure}

\subsection{Proof of Proposition~\ref{prop:regretdecomp}}
 
\begin{proof}
We consider two cases: $J_T(\pi) < \infty$ and $J_T(\pi)=\infty$.

First, suppose that $J_T(\pi)<\infty$. The tower property gives
\begin{align*}
    J_T(\pi) = \sum_{t=0}^T \E_\pi \squ{\bar c(\xhat_t,\Sigma_t,u_t)}. 
\end{align*}
Recall $V_t$ from Lemma~\ref{lem:rewarddecomp}.
Letting $V_{T+1} = 0$, we get a telescoping sum by Lemma~\ref{lem:rewarddecomp}:
\begin{align}
    J_T(\pi)
    &= \sum_{t=0}^T \E_\pi[V_t - V_{t+1}]
       + \Sigma_w\sum_{t=0}^T G_{t+1}
       - \sum_{t=0}^T \mathbb E_\pi\squ{r_t(\hat x_t, \Sigma_t, u_t)} \tag{by tower property} \\
    &= P_0\mu_x^2 + G_0\bar\Sigma + \Sigma_w\sum_{t=1}^{T+1} G_t - \sum_{t=0}^T \E_\pi[r_t(\xhat_t, \Sigma_t, u_t)]. \label{eqn:any-cumsum-decomposition}
\end{align}
This computation was possible because the inequality
\begin{equation*}
    c(x,u) \ge R\paren{u+\frac{S}{R}x}^2,
\end{equation*}
the state recursion, and $\E[x_0^2]<\infty$ imply inductively that $x_t$, $u_t$, and $\xhat_t$ have finite second moments.
Therefore, we get
\begin{align*}
    R_T(\pi,\pi^{OL}) &= J_T(\pi^{OL}) - J_T(\pi) \\
    &= \sum_{t=0}^T \E_\pi[r_t(\xhat_t, \Sigma_t, u_t)] - \sum_{t=0}^T \E_{\pi^{OL}}[r_t(\xhat_t, \Sigma_t, u_t)] \tag{by Eq.~\eqref{eqn:any-cumsum-decomposition}} \\
    &= \sum_{t=0}^T \E_\pi[r_t(\xhat_t, \Sigma_t, u_t)]. \tag{Lemma~\ref{lem:olcost}}
\end{align*}

Next, we consider the case $J_T(\pi)=\infty$.
If $\E_\pi[u_t^2]<\infty$ for every $t=0,\dots,T$, then the state recursion would imply finite second moments for every $x_t$, and hence finite expected quadratic stage costs.
This would contradict $J_T(\pi)=\infty$.
Thus, let
\begin{equation*}
    s:= \min\curl{t \in \curl{0,\dots,T} \colon \E_\pi\!\!\squ{u_t^2}=\infty}.
\end{equation*}
By the definition of $s$, every earlier action has a finite second moment.
Since
\begin{equation*}
    x_s = A^sx_0 + \sum_{j=0}^{s-1}A^{s-1-j}(Bu_j+w_j),
\end{equation*}
and $x_0$ and the process noises have finite second moments, we obtain $\E_\pi\!\!\squ{x_s^2}<\infty$.
In particular, $x_s$ is integrable, and conditional Jensen's inequality and the tower property give
\begin{align*}
    \E_\pi\!\!\squ{\xhat_s^2} &= \E_\pi\!\!\squ{\paren{\E_\pi\squ{x_s \mid \Fcal_s}}^2} \\
    &\le \E_\pi\!\!\squ{\E_\pi\!\!\squ{x_s^2 \mid \Fcal_s}} \\
    &= \E_\pi\!\!\squ{x_s^2}<\infty.
\end{align*}
Next, since $u_s^2 \le 2(u_s-K_s\xhat_s)^2 + 2K_s^2\xhat_s^2$, we get
\begin{equation*}
    \infty = \E_\pi[u_s^2] \le 2\E_\pi\!\!\squ{(u_s-K_s\xhat_s)^2} + 2K_s^2\E_\pi\!\!\squ{\xhat_s^2}.
\end{equation*}
The last term is finite, so $\E_\pi\!\!\squ{(u_s-K_s\xhat_s)^2}=\infty$.
Since $N_s \ge R >0$, it follows that 
\begin{equation*}
    \E_\pi\!\!\squ{p_s(\xhat_s,u_s)} = N_s \E_\pi\!\!\squ{(u_s-K_s\xhat_s)^2} = \infty.
\end{equation*}
On the other hand, the covariance bound gives, for every $t$, 
\begin{equation*}
    0 \le b_t(\Sigma_t,u_t) = \Gamma_{t+1}F(\Sigma_t,u_t) \le \Gamma_{t+1}A^2 \bsig < \infty.
\end{equation*}
Thus, we can see that the expectation of $r_t(\xhat_t,\Sigma_t,u_t)$ is now well-defined in the extended real system, and cannot equal $+\infty$.
In particular, at time $s$, we get 
\begin{equation*}
    \E_\pi\!\big[r_s(\xhat_s,\Sigma_s,u_s)\big]=-\infty.
\end{equation*}
Consequently, using $J_T(\pi^{OL})<\infty$,
\begin{equation*}
    \sum_{t=0}^T \E_\pi[r_t(\xhat_t,\Sigma_t,u_t)] = - \infty = J_T(\pi^{OL})-J_T(\pi) = R_T(\pi,\pi^{OL}).
\end{equation*}
This completes the proof.
\end{proof}

\subsection{Alternate Regret Decomposition}
For proving lower bounds, it will be useful to have access to a more flexible regret decomposition. 
Define
\begin{equation*}
    d_t(\Sigma) = N_t K_t^2 \big(\bsig - \Sigma\big).
\end{equation*}
The term $b_t(\Sigma_{t},u_t)$ corresponds to the immediate covariance reduced at $t$, whereas $d_t(\Sigma_t)$ corresponds to the information retained up to $t$. In Appendix~\ref{app:regret}, Lemma~\ref{lem:covariancediff} implies that these terms accumulate identically. This allows us to obtain the following generalization of Proposition~\ref{prop:regretdecomp}.

\begin{proposition}\label{prop:regretdecomp2} 
    For $\lambda \in [0,1]$, define  
    \begin{equation*}
        r_{t, \lambda}(\xhat, \Sigma, u) = \lambda b_t(\Sigma,u) + (1-\lambda) d_t(\Sigma) - p_t(\xhat, u)
    \end{equation*}
    for $\xhat,u \in \R$ and $\Sigma>0$.
    For any policy $\pi$, 
    \begin{equation*}
        R_T(\pi, \pi^{OL}) = \sum_{t=0}^{T} \E_\pi[r_{t,\lambda}(\xhat_{t}, \Sigma_{t}, u_t)].
    \end{equation*}
\end{proposition}
\begin{proof}
    Lemma~\ref{lem:covariancediff} implies the following equality:
    \begin{align*}
        \sum_{t=0}^T b_t(\Sigma_t,u_t) 
        &= \sum_{t=0}^T d_t(\Sigma_t) + \lambda \paren{\sum_{t=0}^T b_t(\Sigma_t,u_t) - \sum_{t=0}^T d_t(\Sigma_t)} \\
        &= \lambda  \sum_{t=0}^T b_t(\Sigma_t,u_t) + (1-\lambda)\sum_{t=0}^T d_t(\Sigma_t)
    \end{align*}
    Plugging this quantity into Proposition~\ref{prop:regretdecomp} completes the proof.
\end{proof}
This regret decomposition shows that regret accumulates linearly if either $b_t$ or $d_t$ is bounded away from zero, giving us more flexibility in our lower bound derivations. 

\subsection{Lemmas}

\begin{lemma}
\label{lem:covariancediff}
    For any policy, 
    \begin{equation*} \label{eq:infoequiv}
        \sum_{t=0}^T 
        b_t(\Sigma_{t},u_t) =  \sum_{t=0}^T d_t(\Sigma_t).
    \end{equation*}
\end{lemma}
\begin{proof}
    Using the fact that $\Sigma_w = (1-A^2)\bsig$, we get
    \begin{align*}
        F(\Sigma_t, u_t) = (\bsig-\Sigma_{t+1}) - A^2 (\bsig-\Sigma_t).
    \end{align*}
    Then, this equality implies that
    \begin{align*}
        \sum_{t=0}^T b_t(\Sigma_t,u_t) 
        &= \sum_{t=0}^T \paren{\Gamma_{t+1}(\bsig-\Sigma_{t+1}) - \Gamma_{t+1}A^2(\bsig-\Sigma_t)} \\
        &= \sum_{t=0}^T \paren{\Gamma_{t+1}(\bsig-\Sigma_{t+1}) - \Gamma_t(\bsig-\Sigma_t) + N_tK_t^2 (\bsig-\Sigma_t)} \tag{$\because \Gamma_t = N_tK_t^2 + A^2\Gamma_{t+1}$} \\
        &= \sum_{t=0}^T d_t(\Sigma_t)
    \end{align*}
    where the last equality follows since $\Gamma_{T+1} = d_0 = 0$ and $\Sigma_0 = \bsig$.
\end{proof}

\begin{lemma}\label{lem:rewarddecomp}  
    Recall that $G_t = Q + A^2G_{t+1}$ with $G_{T+1} = 0$. Define 
    \begin{align*}
        V_t = P_t \xhat_t^2 + G_t\Sigma_t.
    \end{align*}
    For any $t=0,1,\dots,T$, we have
    \begin{align*}
        \bar c(\xhat_t, \Sigma_t, u_t) = V_t - \E[V_{t+1} \mid \xhat_t, \Sigma_t,u_t] + G_{t+1}\Sigma_w -r_t(\xhat_t,\Sigma_t,u_t).
    \end{align*}
\end{lemma}
\begin{proof}
    Recall that $\bar c(\xhat_t, \Sigma_t, u_t) = Q\xhat_t^2 + Q\Sigma_t + Ru_t^2 + 2u_tS\xhat_t$.
    For the terms related to $\xhat_t$ in $\bar c$, using Lemma~\ref{lem:kalmansecond}, observe the following equality:
    \begin{align*}
        &Q\xhat_t^2 + Ru_t^2 + 2u_tS\xhat_t
          + P_{t+1}\E[\xhat_{t+1}^2 \mid \xhat_t,\Sigma_t,u_t] - P_t \xhat_t^2 \\
        &\quad= \underbrace{(R + B^2P_{t+1})}_{N_t} u_t^2
          + 2\underbrace{(S + ABP_{t+1})}_{-N_tK_t} u_t\xhat_t
          + \underbrace{(Q + A^2P_{t+1} - P_t)}_{N_tK_t^2}\xhat_t^2
          + P_{t+1}F(\Sigma_t,u_t) \\
        &\quad= N_t(u_t - K_t\xhat_t)^2 + P_{t+1}F(\Sigma_t,u_t).
    \end{align*}
    For the term related to $\Sigma_t$, we also observe the following equality:
    \begin{align*}
        Q\Sigma_t + G_{t+1}\Sigma_{t+1} - G_t\Sigma_t
        &= Q\Sigma_t + G_{t+1}\left(A^2\Sigma_t + \Sigma_w - F(\Sigma_t,u_t)\right) - (Q + A^2 G_{t+1})\Sigma_t \\
        &= G_{t+1}\Sigma_w - G_{t+1} F(\Sigma_t,u_t).
    \end{align*}
    Hence, combining two equalities gives
    \begin{align*}
        \bar c(\xhat_t,\Sigma_t,u_t) + \E[V_{t+1} \mid \xhat_t, \Sigma_t, u_t] -V_t 
        &= N_t(u_t-K_t\xhat_t)^2+G_{t+1}\Sigma_w+(P_{t+1}-G_{t+1})F(\Sigma_t,u_t) \\
        &= N_t(u_t-K_t\xhat_t)^2+G_{t+1}\Sigma_w-\Gamma_{t+1} F(\Sigma_t,u_t) \tag{Lemma~\ref{lem:Gammiden}} \\
        &= G_{t+1}\Sigma_w - r_t(\xhat_t,\Sigma_t,u_t)
    \end{align*}
    and rearranging completes the proof.
\end{proof}

\begin{lemma}\label{lem:olcost}  
We have \begin{equation}
    \sum_{t=0}^{T}\E_{\pi^{OL}}[r_t(\xhat_t, \Sigma_t, u_t)] = 0
\end{equation}
\end{lemma}
\begin{proof}
    From the law of total variance, we get $\Var(\xhat_t) = \Var(x_t) - \Sigma_t$ under $\pi^{OL}$.
    Also, each recursion of $x_t$ and $\Sigma_t$ gives the equalities:
    \begin{equation*}
        \Var(x_{t+1}) = A^2 \Var(x_t) + \Sigma_w \quad \text{and} \quad 
        \Sigma_{t+1} = A^2\Sigma_t + \Sigma_w - F(\Sigma_t,u_t).
    \end{equation*}
    It follows that
    \begin{equation} \label{eqn:F-Varxhat}
        \Var(\xhat_{t+1}) = A^2\Var(\xhat_{t}) + F(\Sigma_t, u_t), \quad \text{i.e.,} \quad F(\Sigma_t,u_t) = \Var(\xhat_{t+1}) - A^2 \Var(\xhat_t). 
    \end{equation}
    For the OL policy $u_t = K_t m_t = K_t\E[x_t]$, observe that
    \begin{equation} \label{eqn:OL-offset}
        \E\squ{(u_t-K_t\xhat_t)^2} = K_t^2 \E\squ{(\xhat_t - \E[\xhat_t])^2} = K_t^2 \Var(\xhat_t).
    \end{equation}
    Therefore, the quantity can be reduced as follows:
    \begin{align*}
        \sum_{t=0}^T \E_{\pi^{OL}}\squ{r_t(\xhat_t,\Sigma_t,u_t)} &= \sum_{t=0}^T \E \squ{\Gamma_{t+1}F(\Sigma_t,u_t) - N_t(u_t-K_t\xhat_t)^2} \\
        &= \sum_{t=0}^T \E\squ{\Gamma_{t+1}\Var(\xhat_{t+1}) - (N_tK_t^2 + A^2\Gamma_{t+1})\Var(\xhat_t)} \tag{by Eq.~\eqref{eqn:F-Varxhat} and \eqref{eqn:OL-offset}} \\
        &= \sum_{t=0}^T \E\squ{\Gamma_{t+1}\Var(\xhat_{t+1}) - \Gamma_t\Var(\xhat_t)} \\
        &= \Gamma_{T+1}\Var(\xhat_{T+1}) - \Gamma_0\Var(\xhat_0) 
    \end{align*}
    Since $\Gamma_{T+1} = 0$ and $\Var(\xhat_0) = 0$, we complete the proof.
\end{proof}

\section{Bounded Optimal-Control Advantage} \label{app:club}

This section establishes a uniform upper bound on the cost advantage of any policy over open-loop control when $q<q_1$.
The proof is largely algebraic and relies on the Riccati equations to show that $q < q_{1}$ implies 
\begin{equation} \label{eq:rbound2}
    r_t(\xhat_{t}, \Sigma_{t}, u_t) \leq  P_{t} \xhat_{t}^2-P_{t+1}\E_{\pi}[ \xhat_{t+1}^2 \mid \Fcal_t].
\end{equation} 
which is shown in Lemma~\ref{lem:rhobound}.
Under Eq.~\eqref{eq:rbound2}, the regret upper bound telescopes to $P_0 \mu_x^2$, which is constant due to Lemma~\ref{lem:stationary}. 
Consequently, as the policy is arbitrary, even the optimal policy can improve on open-loop policy by at most a constant independent of the horizon.
Figure~\ref{fig:dependencies-d} describes the proof dependencies for Section~\ref{app:club}.
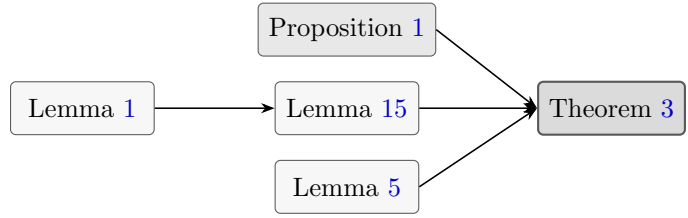
\begin{figure}[htbp]
    \centering
    \begin{tikzpicture}[
        y=0.65cm,
        dependency result/.style={
            draw=black!65, rounded corners=2pt,
            minimum width=1.9cm, minimum height=0.7cm,
            inner sep=4pt, font=\small, align=center
        },
        dependency lemma/.style={dependency result, fill=black!3},
        dependency proposition/.style={dependency result, fill=black!9},
        dependency theorem/.style={dependency result, fill=black!14, line width=0.8pt},
        dependency uses/.style={-{Stealth[length=1.8mm]}, line width=0.6pt}
    ]
        \node[dependency lemma] (n0) at (0,0) {Lemma~\ref{lem:kalmansecond}};
        \node[dependency lemma] (n1) at (3.5,0) {Lemma~\ref{lem:rhobound}};
        \node[dependency proposition] (n2) at (3.5,1.6) {Proposition~\ref{prop:regretdecomp}};
        \node[dependency lemma] (n3) at (3.5,-1.6) {Lemma~\ref{lem:stationary}};
        \node[dependency theorem] (n4) at (7,0) {Theorem~\ref{prop:clconstant}};

        \draw[dependency uses] (n0.east) -- (n1.west);
        \draw[dependency uses] (n1.east) -- (n4.west);
        \draw[dependency uses] (n2.east) -- (n4.west);
        \draw[dependency uses] (n3.east) -- (n4.west);
    \end{tikzpicture}
    \caption{Main proof dependencies for Section~\ref{app:club}.
    An arrow points from a result to a result whose proof uses it.
    Common preliminary results are omitted.}
    \label{fig:dependencies-d}
\end{figure}

\subsection{Proof of Theorem~\ref{prop:clconstant}}
\begin{proof}
    We have \begin{align*}
        R_T(\pi, \pi^{OL}) 
        &= \sum_{t=0}^{T} \E_{\pi}[r_t(\xhat_{t}, \Sigma_{t}, u_t)]   \tag{Proposition ~\ref{prop:regretdecomp}}\\
        &\leq \sum_{t=0}^{T} \E_{\pi}[ P_t \xhat_{t}^2-\E[P_{t+1} \xhat_{t+1}^2 \mid \xhat_{t}, \Sigma_{t}, u_t]] \tag{Lemma ~\ref{lem:rhobound}}  \\
        &= P_0\mu_x^2 
    \end{align*}
    where the last equality is obtained by telescoping and that $P_{T+1}=0$.
    By Lemma~\ref{lem:stationary}, for
    \begin{equation*}
        P_0\mu_x^2 \le (P+C_P)\mu_x^2
    \end{equation*} 
    for some $C_P>0$, which is independent of $T$ and $\pi$.
\end{proof}

\subsection{Lemmas}
\begin{lemma} \label{lem:rhobound} 
    If $q < q_{1}$, then 
    \begin{equation}
        r_t(\xhat_t, \Sigma_t, u_t) \leq   P_t \xhat_t^2-\E[P_{t+1} \xhat_{t+1}^2 \mid \xhat_t, \Sigma_t, u_t].
    \end{equation}
\end{lemma}
\begin{proof}
    Denote \begin{equation}
         \psi(\xhat_t, \Sigma_t, u_t) =P_t \xhat_t^2-\E[P_{t+1} \xhat_{t+1}^2 \mid \xhat_t, \Sigma_t, u_t] - r_t(\xhat_t, \Sigma_t, u_t)
    \end{equation}
    We aim to show $\psi(\xhat, \Sigma, u) \geq 0$. We have 
    \begin{align*}
         &\psi(\xhat, \Sigma, u)\\ &=P_t \xhat_t^2-\E[P_{t+1} \xhat_{t+1}^2 \mid \xhat_t, \Sigma_t, u_t] - r_t(\xhat_t, \Sigma_t, u_t) \\
         &= P_t \xhat_t^2-P_{t+1} \paren{\paren{A \xhat_t + B u_t}^2 + F(\Sigma_t,u_t)} +N_t \paren{u_t - K_t \xhat_t}^2 -\Gamma_{t+1} F(\Sigma_t,u_t) \tag{Lemma~\ref{lem:kalmansecond}}\\
         &= P_t \xhat_t^2-P_{t+1} \paren{A \xhat_t + B u_t}^2  +N_t \paren{u_t - K_t \xhat_t}^2 -G_{t+1} F(\Sigma_t,u_t) \tag{$\because \Gamma_{t+1} = G_{t+1}-P_{t+1}$}\\
         &= \paren{P_t - A^2 P_{t+1} + N_t K_t^2} \xhat_t^2 -2\paren{P_{t+1} AB + N_t K_t}\xhat_tu_t+\paren{N_t - P_{t+1} B^2} u_t^2 - G_{t+1}  F(\Sigma_t,u_t)
    \end{align*}
    Recall the following equalities: 
    \begin{equation*}
        P_t = A^2 P_{t+1} + Q - K_t^2 N_t, \quad 
        N_t = R + B^2 P_{t+1}, \quad
        -K_tN_t = S + BP_{t+1} A.
    \end{equation*}
    It follows that
    \begin{equation*}
        \psi(\xhat_t, \Sigma_t, u_t)
        = Q \xhat_t^2 +2S\xhat_tu_t+R u_t^2 - G_{t+1} F(\Sigma_t,u_t).
    \end{equation*}
    From the two inequalities
    \begin{equation*}
        F(\Sigma_t,u_t) \leq \frac{A^2\bsig^2 C^2 u_t^2}{\sigma_z^2}
        \quad \text{and} \quad
        G_{t+1} = Q\sum_{j=0}^{T-t-1} A^{2j}
        \leq Q\sum_{j=0}^{\infty } A^{2j}
        = \frac{Q}{1-A^2},
    \end{equation*}
    we get
    \begin{equation*}
        \psi(\xhat_t, \Sigma_t, u_t) \geq  \paren{R -\frac{Q}{1-A^2} \cdot \frac{A^2C^2\bsig^2 }{\sigma_z^2}}u_t^2 + 2Su_t\xhat_t  + Q \xhat_t^2.
    \end{equation*}
    Call $\alpha := R-\frac{QA^2 C^2\bar{\Sigma}^2}{\sigma_z^2\paren{1-A^2}} $, then we have 
    \begin{align*}
         \psi(\xhat, \Sigma, u) &\geq2Su_t\xhat_t +Q \xhat_t^2 + \alpha u_t^2 \\
         &= \frac{1}{Q}u_t^2S^2 +2u_tS\xhat_t +Q \xhat_t^2 + \paren{\alpha -\frac{1}{Q}S^2}u_t^2 \\
         &= \paren{\frac{1}{\sqrt{Q}}u_tS+ \sqrt{Q} \xhat_t}^2 + \paren{\alpha -\frac{1}{Q}S^2}u_t^2 
    \end{align*}
    Hence if $\alpha -\frac{1}{Q}S^2 \geq 0$, then $\psi(\xhat_t, \Sigma_t, u_t) \geq 0$. That is, we require 
    \begin{equation}
        R-\frac{QA^2 C^2\bar{\Sigma}^2}{\sigma_z^2\paren{1-A^2}} -\frac{1}{Q}S^2\geq 0.
    \end{equation}
    Rearranging yields the result.
\end{proof}

\section{Bounded CE Advantage} \label{app:ceub}
This section proves Proposition~\ref{prop:constantscalar}, which shows that $\pi^{CE}$ achieves only a bounded advantage over $\pi^{OL}$ when $q < q_3$.
Under $\pi^{CE}$, the term $p_t$ vanishes, so by Proposition~\ref{prop:regretdecomp}, the regret is equal to 
\begin{equation*}
    \sum_{t=0}^T \Gamma_{t+1} \E[F(\Sigma_t,K_t\xhat_t)]
\end{equation*}
and it suffices to show that how this sum is bounded.
The key is the multiplicative recursion $\hat x_{t+1} = M_t\hat x_t$ of Lemma~\ref{lem:multiplicative}.
When $q < q_3$, Lemma~\ref{lem:riccaticont} shows that $\E[M_t^2 \mid H_t] \le \gamma < 1$ for all $t$ away from the horizon, so $\E[\hat x_t^2]$ decays geometrically.
From the observation that $F(\Sigma_t,K_t\xhat_t)$ is at most proportional to $\xhat_t^2$, it follows that the regret is dominated by a convergent geometric series. Figure~\ref{fig:dependencies-e} describes the proof dependencies for Section~\ref{app:ceub}.

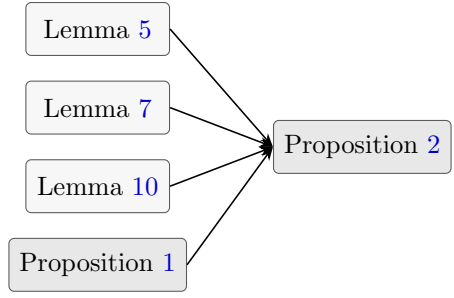
\begin{figure}[htbp]
    \centering
    \begin{tikzpicture}[
        y=0.65cm,
        dependency result/.style={
            draw=black!65, rounded corners=2pt,
            minimum width=1.9cm, minimum height=0.7cm,
            inner sep=4pt, font=\small, align=center
        },
        dependency lemma/.style={dependency result, fill=black!3},
        dependency proposition/.style={dependency result, fill=black!9},
        dependency theorem/.style={dependency result, fill=black!14, line width=0.8pt},
        dependency uses/.style={-{Stealth[length=1.8mm]}, line width=0.6pt}
    ]
        \node[dependency lemma] (n0) at (0,2.4) {Lemma~\ref{lem:stationary}};
        \node[dependency lemma] (n1) at (0,0.8) {Lemma~\ref{lem:riccaticont}};
        \node[dependency lemma] (n2) at (0,-0.8) {Lemma~\ref{lem:multiplicative}};
        \node[dependency proposition] (n3) at (0,-2.4) {Proposition~\ref{prop:regretdecomp}};
        \node[dependency proposition] (n4) at (3.5,0) {Proposition~\ref{prop:constantscalar}};

        \draw[dependency uses] (n0.east) -- (n4.west);
        \draw[dependency uses] (n1.east) -- (n4.west);
        \draw[dependency uses] (n2.east) -- (n4.west);
        \draw[dependency uses] (n3.east) -- (n4.west);
    \end{tikzpicture}
    \caption{Main proof dependencies for Section~\ref{app:ceub}.
    An arrow points from a result to a result whose proof uses it.
    Common preliminary results are omitted.}
    \label{fig:dependencies-e}
\end{figure}

\subsection{Proof of Proposition~\ref{prop:constantscalar}}
\begin{proof}
     Assume $q < q_{3}$. Let us analyze $F\paren{\Sigma_{t}, K_t\xhat_{t}}$. 
     Since $\Sigma_t \le \bsig$, we have
     \begin{align*}
        F\paren{\Sigma_{t}, K_t\xhat_{t}} &= \frac{A^2K_t^2 \xhat_t^2C^2\Sigma_t^2}{\sigma_z^2 +K_t^2 \xhat_t^2C^2 \Sigma_t} 
        \;\leq\; \frac{A^2K_t^2 \xhat_t^2C^2\Sigma_t^2}{\sigma_z^2} 
        \;\leq\; \frac{A^2K_t^2 \xhat_t^2C^2\bsig^2}{\sigma_z^2} 
        \;=\; \frac{A^2K_t^2 C^2\bsig^2}{\sigma_z^2}\xhat_t^2 
    \end{align*}
    Define the function $h$ by
    \begin{equation*}
        h(k) := \paren{A + B k}^2+\frac{A^2k^2 C^2\bsig^2}{\sigma_z^2}.
    \end{equation*} 
    Observe that by definition \begin{equation*}
        q < q_{3} \iff h(K) =D^2 + q <1.
    \end{equation*} 
    The function $h$ is clearly continuous in $k$, thus by Lemma~\ref{lem:riccaticont} there exists an integer $t_0$, depending only on the problem parameters and not on $T$, such that $h(K_t) < \frac{D^2+q+1}{2}$ for all $t \leq T- t_0$. Denote $\gamma := (D^2+q+1)/2$ so that $h(K_t)\le\gamma$ for all $t\le T-t_0$.
    By Lemma~\ref{lem:multiplicative}(ii), we have
    \begin{equation*}
        \E[M_t^2 \mid \Fcal_t] = D_t^2 + q_t^2 = (A+BK_t)^2 + \frac{A^2K_t^2C^2\Sigma_t^2}{\sigma_z^2 + K_t^2 \xhat_t^2 C^2 \Sigma_t} \le (A+BK_t)^2 + \frac{A^2K_t^2C^2\bar\Sigma^2}{\sigma_z^2} = h(K_t).
    \end{equation*}

    Then, for $t \leq T- t_0$, we have
    \begin{align*}
        \E[M_t^2 \mid H_t] \le h(K_t) \le \gamma.
    \end{align*}
    Now, Lemma~\ref{lem:multiplicative}(i) and tower property give
    \begin{align*}
        \E[\xhat_{t+1}^2] 
        = \E[\xhat_{t}^2\E[M_t^2 \mid H_t]] 
        \leq \gamma\E[\xhat_{t}^2].
    \end{align*}
    Thus, for any $t \leq T - t_0 +1$ we have \begin{equation*}
         \E[\xhat_{t}^2] \leq \mu_x^2 \gamma^{t}.
    \end{equation*}
    To connect this to regret, observe that 
    \begin{align*}
        \E[F\paren{\Sigma_{t}, K_{t}\xhat_{t}}] 
        \;\leq\; \frac{A^2K_t^2 C^2\Sigma_t^2}{\sigma_z^2} \E[\xhat_{t}^2] 
        \;\leq\; h(K_t) \E[\xhat_{t}^2] 
        \;\leq\; \mu_x^2 \gamma^{t+1},
    \end{align*}
    so that
    \begin{align*}
        \sum_{t= 0}^{T-t_0}\E[F\paren{\Sigma_{t}, K_t\xhat_{t}}] 
        \; \leq \; \mu_x^2\sum_{t= 0}^{T-t_0}\gamma^{t+1} 
        \; \leq \; \mu_x^2\sum_{t= 0}^{\infty}\gamma^{t+1}
        \; = \;\frac{\mu_x^2 \gamma }{1-\gamma}
    \end{align*}
    For $t > T- t_0+1$, we consider the natural crude bound:
    \begin{equation*}
        \E[F(\Sigma_t,K_t\xhat_t)] \le A^2 \bsig.
    \end{equation*}
    For the coefficient $\Gamma_t$, we consider its uniform bound as follows:    
    Lemma~\ref{lem:stationary} implies there exists $C_1 < \infty$ such that $\max_{t}\{N_t,K_t\} \leq C_1$, so that
    \begin{align*}
        \Gamma_t &= \sum_{s=t}^T A^{2\paren{s-t}} N_sK_s^2 
        \;\leq\; C_1^3\sum_{s=t}^T A^{2\paren{s-t}} 
        \;\leq\; \frac{C_1^3}{1-A^2}.
    \end{align*}
    Therefore, the regret now can be bounded as follows:
    \begin{align*}
        \sum_{t=0}^{T} \E_\pi[r_t(\xhat_{t}, \Sigma_{t}, u_t)] &= \sum_{t=0}^{T}\Gamma_{t+1}\E[F\paren{\Sigma_{t}, K_{t}\xhat_{t}}] \\
        &\leq \frac{C_1^3}{1-A^2} \cdot \sum_{t=0}^{T}\E[F\paren{\Sigma_{t}, K_{t}\xhat_{t}}] \\
        &= \frac{C_1^3}{1-A^2}\paren{\frac{\mu_x^2 \gamma }{1-\gamma}+t_0A^2 \bsig}
    \end{align*}
    Note that $t_0$ is chosen independent of $T$. This completes the proof.
\end{proof}

\section{Linear CE Advantage} \label{app:celb}

This section proves Theorem~\ref{thm:celb}, which shows that $\pi^{CE}$ achieves a linear advantage over $\pi^{OL}$ when $q > q_4$, through the more general Proposition~\ref{prop:sptightened}.
Here, we show that the region $\Mcal_\delta$, in which $\xhat_t^2$ and $\bsig - \Sigma_t$ are small, is unstable under $\pi^{CE}$.
Since $\xhat_{t+1} = M_t\xhat_t$ is a multiplicative recursion, we measure growth on a logarithmic scale
through the drift 
\begin{equation*}
     \Delta_t :=\E\squ{\log\xhat_{t+1 }^2- \log\xhat_{t}^2 \mid \Fcal_t} = \E\squ{\log M_t^2 \mid \Fcal_t}.
\end{equation*}
Because $M_t$ is conditionally Gaussian (Lemma~\ref{lem:multiplicative}), this drift can be bounded using moments of the logarithm of a noncentral chi-squared variable (Lemmas~\ref{lem:poisson} and~\ref{lem:logtightened}).

The argument proceeds in three steps.
First, under condition~\eqref{eq:celbcond}, the drift is at least $\varepsilon$ whenever the belief lies in
$\mathcal M_\delta$ (Lemma~\ref{lem:unstable2}), while it is bounded below by $-\ell$ everywhere
(Lemma~\ref{lem:uniform}); Lemma~\ref{lem:non-zero} ensures that $\log\hat x_t^2$ is well defined.
Second, summing the drift $\Delta_t$ over time, the left-hand side telescopes, and $\mathbb E[\log\hat x_t^2]$ is bounded above
by Lemma~\ref{lem:kalmanbound}. Hence the belief must spend a linear expected number of steps outside
$\mathcal M_\delta$, which is~\eqref{eq:problinear} (Lemma~\ref{lem:lineargrowth}).
Third, at every such step either $\hat x_t^2 > \delta$, so that $b_t$ is bounded below, or
$\bar\Sigma - \Sigma_t > \delta$, so that $d_t$ is bounded below. Proposition~\ref{prop:regretdecomp2}
with $\lambda = 1$ or $\lambda = 0$, respectively, then turns this into a linear advantage over $\pi^{OL}$.
Finally, as $\delta, \varepsilon \to 0$, the left-hand side of~\eqref{eq:celbcond} converges to $D^2 + q$,
so the condition holds for some sufficiently small $\delta, \varepsilon > 0$ whenever $q > q_4$,
which proves Theorem~\ref{thm:celb}.
Figure~\ref{fig:dependencies-f} describes the proof dependencies for Section~\ref{app:celb}.

\begin{figure}[htbp]
    \centering
    \begin{tikzpicture}[
        y=0.65cm,
        dependency result/.style={
            draw=black!65, rounded corners=2pt,
            minimum width=1.9cm, minimum height=0.7cm,
            inner sep=4pt, font=\small, align=center
        },
        dependency lemma/.style={dependency result, fill=black!3},
        dependency proposition/.style={dependency result, fill=black!9},
        dependency theorem/.style={dependency result, fill=black!14, line width=0.8pt},
        dependency uses/.style={-{Stealth[length=1.8mm]}, line width=0.6pt}
    ]
        \node[dependency lemma] (n0) at (0,0) {Lemma~\ref{lem:poisson}};
        \node[dependency lemma] (n1) at (2.7,0) {Lemma~\ref{lem:logtightened}};
        \node[dependency lemma] (n2) at (5.4,0.75) {Lemma~\ref{lem:uniform}};
        \node[dependency lemma] (n3) at (5.4,-0.75) {Lemma~\ref{lem:unstable2}};
        \node[dependency lemma] (n4) at (5.4,2.2) {Lemma~\ref{lem:kalmanbound}};
        \node[dependency lemma] (n5) at (5.4,-2.2) {Lemma~\ref{lem:non-zero}};
        \node[dependency lemma] (n6) at (8.1,0) {Lemma~\ref{lem:lineargrowth}};
        \node[dependency proposition] (n7) at (8.1,1.65) {Proposition~\ref{prop:regretdecomp}};
        \node[dependency proposition] (n8) at (8.1,-1.65) {Proposition~\ref{prop:regretdecomp2}};
        \node[dependency proposition] (n9) at (10.8,0) {Proposition~\ref{prop:sptightened}};
        \node[dependency theorem] (n10) at (13.5,0) {Theorem~\ref{thm:celb}};

        \draw[dependency uses] (n0.east) -- (n1.west);
        \draw[dependency uses] (n1.east) -- (n2.west);
        \draw[dependency uses] (n1.east) -- (n3.west);
        \draw[dependency uses] (n2.east) -- (n6.west);
        \draw[dependency uses] (n3.east) -- (n6.west);
        \draw[dependency uses] (n4.east) -- (n6.west);
        \draw[dependency uses] (n5.east) -- (n6.west);
        \draw[dependency uses] (n6.east) -- (n9.west);
        \draw[dependency uses] (n7.east) -- (n9.west);
        \draw[dependency uses] (n8.east) -- (n9.west);
        \draw[dependency uses] (n9.east) -- (n10.west);
    \end{tikzpicture}
    \caption{Main proof dependencies for Section~\ref{app:celb}.
    An arrow points from a result to a result whose proof uses it.
    Common preliminary results are omitted.}
    \label{fig:dependencies-f}
\end{figure}
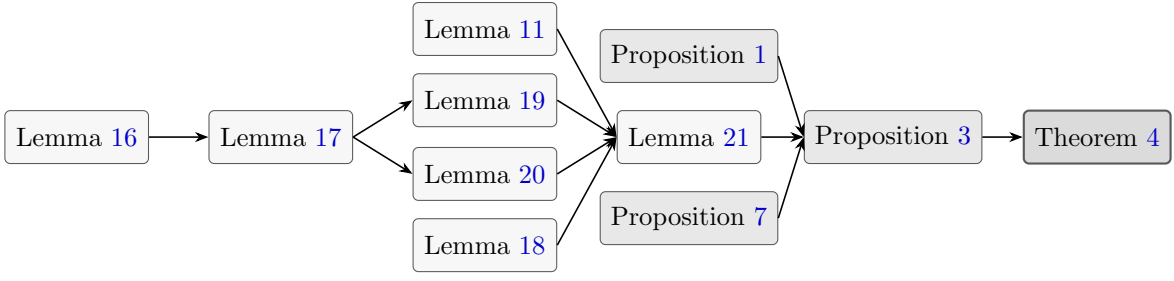

\subsection{Proof of Theorem~\ref{thm:celb}}
\begin{proof}
Since $q > q_4 = 2e^{\gamma_E} - D^2$, we have $D^2 + q > 2e^{\gamma_E}$. Hence
\[
    \kappa := \log(D^2 + q) - \gamma_E - \log 2 > 0 .
\]
For $\delta \in [0, \bar\Sigma]$ define
\[
    \phi(\delta) := \frac{A^2(\bar\Sigma - \delta)^2K^2C^2}{\sigma_z^2 + K^2C^2(\bar\Sigma - \delta)\,\delta} .
\]
On $[0, \bar\Sigma]$ we have $(\bar\Sigma - \delta)\delta \ge 0$, so the denominator is at least $\sigma_z^2 > 0$.
Hence $\phi$ is continuous on $[0, \bar\Sigma]$, and $\phi(0) = q$.
Since $D \neq 0$, we have $D^2 + \phi(\delta) \ge D^2 > 0$, so $\delta \mapsto \log\big(D^2 + \phi(\delta)\big)$
is continuous on $[0, \bar\Sigma]$ and equals $\log(D^2 + q)$ at $\delta = 0$.
Therefore there exists $\delta \in (0, \bar\Sigma]$ such that
\[
    \log\big(D^2 + \phi(\delta)\big) > \log(D^2 + q) - \frac{\kappa}{2} .
\]
Set $\varepsilon := \kappa/4 > 0$. Then we have
\[
    \log\big(D^2 + \phi(\delta)\big) - \gamma_E - \log 2 - \varepsilon
    > \Big(\log(D^2 + q) - \gamma_E - \log 2\Big) - \frac{\kappa}{2} - \frac{\kappa}{4}
    = \frac{\kappa}{4} > 0 
\]
which rearranges to \eqref{eq:celbcond}.
Note that the pair $(\delta, \varepsilon)$ is independent of $T$.
Together with the assumptions $\mu_x \neq 0$ and $D \neq 0$, all hypotheses of Proposition~\ref{prop:sptightened} hold,
and Proposition~\ref{prop:sptightened} gives $R_T(\pi^{CE}, \pi^{OL}) = \Omega(T)$.
\end{proof}
\subsection{Proof of Proposition~\ref{prop:sptightened}}
\begin{proof}
    Lemma~\ref{lem:lineargrowth} proves \eqref{eq:problinear} in Proposition~\ref{prop:sptightened}.
    Now we bound the regret.
    By Lemma~\ref{lem:riccaticont}, there exists $t_5$ such that $K_t^2 > K^2/2$, $N_t > N/2$, $\Gamma_{t+1} > \Gamma/2$ for all $t \le T-t_5$.
    Take $t_1 = \max\{t_4,t_5\}$ where $t_4$ is obtained in Lemma~\ref{lem:lineargrowth}. It follows from Lemma~\ref{lem:lineargrowth} that 
    \begin{equation}
        \max\curl{\sum_{t=0}^{T-t_1}\Pr \paren{\xhat_{t}^2 > \delta}, \sum_{t=0}^{T-t_1}\Pr \paren{\bsig-\Sigma_t > \delta}} \geq \frac{\paren{T-t_1}\varepsilon+ \log\paren{\frac{\mu_x^2}{C_{CE}}}}{2 \paren{\varepsilon+ \ell}}
    \end{equation}
    When $\sum_{t=0}^{T-t_1}\Pr \paren{\xhat_{t}^2 > \delta}$ is larger applying Proposition~\ref{prop:regretdecomp} we have 
    \begin{align*}
        R_T(\pi^{CE}, \pi^{OL}) &= \sum_{t=0}^{T} \E_{\pi^{CE}}[r_t(\xhat_{t}, \Sigma_{t}, u_t)] \\
        &\geq \sum_{t=0}^{T-t_1} \E_{\pi^{CE}}[r_t(\xhat_{t}, \Sigma_{t}, u_t)]\\
        &= \sum_{t=0}^{T-t_1} \Gamma_{t+1}  \E\squ{\frac{A^2K_t^2 \xhat_t^2C^2\Sigma_t^2}{\sigma_z^2 +K_t^2 \xhat_t^2C^2 \Sigma_t}} \tag{for $\pi^{CE}$, $p_t(\xhat_t,u_t)=0$ for all $t$} \\
        &\geq \frac{\Gamma }{2}\sum_{t=0}^{T-t_1}  \E\squ{\frac{\frac{1}{2}A^2K^2 \xhat_t^2C^2\Sigma_t^2}{\sigma_z^2 +\frac{1}{2}K^2 \xhat_t^2C^2 \Sigma_t}}  \tag{Lemma~\ref{lem:riccaticont}} \\
        &\geq \frac{\Gamma }{2} \sum_{t=0}^{T-t_1} \E\squ{\frac{\frac{1}{2}A^2K^2 \xhat_t^2C^2\Sigma_t^2}{\sigma_z^2 +\frac{1}{2}K^2 \xhat_t^2C^2 \Sigma_t } \mathbf{1}\curl{\xhat_{t}^2 > \delta}} \\
        &\geq \frac{\Gamma }{2} \sum_{t=0}^{T-t_1} \frac{\frac{1}{2}A^2K^2 \delta C^2\Sigma_w^2}{\sigma_z^2 +\frac{1}{2}K^2 \delta C^2 \Sigma_w } \E\squ{\mathbf{1}\curl{\xhat_{t}^2 > \delta}}  \\
        &= \frac{\Gamma }{2} \cdot \frac{\frac{1}{2}A^2K^2 \delta C^2\Sigma_w^2}{\sigma_z^2 +\frac{1}{2}K^2 \delta C^2 \Sigma_w } \sum_{t=0}^{T-t_1}  \Pr \paren{\xhat_{t}^2 > \delta}\\
        &\geq \frac{\Gamma }{2} \cdot \frac{\frac{1}{2}A^2K^2 \delta C^2\Sigma_w^2}{\sigma_z^2 +\frac{1}{2}K^2 \delta C^2 \Sigma_w } \cdot\frac{\paren{T-t_1}\varepsilon+ \log\paren{\frac{\mu_x^2}{C_{CE}}}}{2 \paren{\varepsilon + \ell}} 
    \end{align*} When $\sum_{t=0}^{T}\Pr \paren{\bsig-\Sigma_t > \delta}$ is larger, applying Proposition~\ref{prop:regretdecomp2} with $\lambda = 0$ we have \begin{align*}
        R_T(\pi^{CE}, \pi^{OL}) &=\sum_{t=0}^{T}K_t^2 N_t \E\squ{\bsig- \Sigma_{t}} \\
        &\geq \sum_{t=0}^{T-t_1}K_t^2 N_t \E\squ{\bsig- \Sigma_{t}}\\
        &\geq \frac{K^2 N}{4}\sum_{t=0}^{T-t_1} \E\squ{\bsig- \Sigma_{t}}  \tag{Lemma~\ref{lem:riccaticont}}\\
        &\geq \frac{K^2 N}{4} \cdot \delta \cdot \sum_{t=0}^{T-t_1} \Pr \paren{\bsig-\Sigma_t > \delta}\\
        &\geq  \frac{K^2 N}{4} \cdot \delta  \cdot  \frac{\paren{T-t_1}\varepsilon+ \log\paren{\frac{\mu_x^2}{C_{CE}}}}{2 \paren{\varepsilon + \ell}}.
    \end{align*}
    It follows that $R_T(\pi^{CE},\pi^{OL}) = \Omega(T)$.
\end{proof}

\subsection{Lemmas}

\begin{lemma}\label{lem:poisson} 
    Define the function $f$ on $\Z_{\ge 0}$ and the function $g$ on $[0,\infty)$ given by
        \begin{equation*}
            f(n) = 
            \begin{cases}
                \sum_{j=0}^{n-1}\frac{2}{2j+1} & n=1,2,\dots \\
                0 & n=0,
            \end{cases} 
            \quad \text{and} \quad 
            g(\mu) = 
            \begin{cases}
                \E_{Y \sim \mathrm{Pois}(\mu)}\squ{f(Y)} & \mu>0, \\
                0 & \mu = 0.
            \end{cases}
        \end{equation*}
        Then, $g$ is right-continuous at $\mu=0$ and 
        \begin{equation}
            g'(\mu) =  \E_{Y \sim  \mathrm{Pois}(\mu)}\squ{ f(Y+1) -f(Y)} = 
            \E_{Y \sim \mathrm{Pois}(\mu)}\squ{\frac{2}{2Y+1}} \ge \frac{2}{2\mu+1}.
        \end{equation}
        This also implies that $g(\mu) \ge \log(1+2\mu)$ for any $\mu\ge0$.
\end{lemma}
\begin{proof}
    By definition, we have 
    \[
        g(\mu) = \E_{Y \sim  \mathrm{Pois}(\mu)}\squ{ f(Y)} 
        = \sum_{y=0}^{\infty} f(y) \frac{e^{-\mu} \mu^y}{y!}.
    \]
    Each term of $f$ is at most $2$, so $0 \le f(y) \le 2y$ for all $y \in \Z_{\ge 0}$.
    Hence, for every $\mu > 0$,
    \[
        0 \le g(\mu) \le \mathbb E_{Y\sim\mathrm{Pois}(\mu)}[2Y] = 2\mu<\infty .
    \]
    Letting $\mu \to 0+$ gives $\lim_{\mu\to 0+} g(\mu) = 0 = g(0)$,
    which proves right-continuity at $\mu = 0$.
    Next, we compute $g'(\mu)$.
    Define the function $h(\mu) = e^\mu g(\mu)$. 
    For every $R>0$, we have
    \[
    h(R) = \sum_{y=0}^\infty f(y) \frac{R^y}{y!} \le 2\sum_{y=1}^\infty \frac{R^y}{(y-1)!} = 2Re^R<\infty,
    \]
    that is, $h$ is a power series whose radius of convergence is infinite.
    Therefore, we may compute the derivative of $h$ by differentiating each summands:
    \[
    h'(\mu) = \sum_{y=1}^\infty f(y) \frac{\mu^{y-1}}{(y-1)!} = \sum_{y=0}^\infty f(y+1)\frac{\mu^y}{y!}.
    \]
    Then, since $g(\mu) = e^{-\mu}h(\mu)$, we get
    \begin{align*}
        g'(\mu) 
        =\sum_{y=0}^\infty \squ{f(y+1)- f(y)} \frac{e^{-\mu}\mu^y}{y!} 
        = \E_{Y \sim \mathrm{Pois}(\mu)}\squ{\frac{2}{2Y+1}} 
        \ge \frac{2}{2\E_{Y \sim  \mathrm{Pois}(\mu)}\squ{Y}+1} = \frac{2}{2\mu+1}
    \end{align*}
    where the inequality is obtained from Jensen's inequality.
    It follows that, for $\mu>0$,
    \begin{align*}
        g(\mu) = g(\mu)-g(0) = \int_0^\mu g'(x) dx \ge \int_0^\mu \frac{2}{2x+1}dx = \log(1+2\mu).
    \end{align*}
    If $\mu=0$, then $g(0) = 0 = \log 1$. Therefore, the proof is done.
\end{proof}
\begin{lemma}\label{lem:logtightened} 
    Let $X \sim \mathcal{N}(0,1)$. Then, for any $c\in\R$, we have \begin{equation}
        \E[\log\paren{X + c}^2 ]\geq \E[\log  X^2 ] + \log(1+c^2)
    \end{equation}
\end{lemma}
\begin{proof}
    If $c=0$, then it is trivial. Assume that $c \neq 0$.
    Let $Z_i$ denote the central chi-squared random variable with $i$ degrees of freedom. Let $\psi$ be the digamma function. It is well known \cite{pav2015moments} that \begin{equation} \label{eqn:log-chi-squared-expectation}
        \E[\log Z_i] = \log 2 + \psi(i/2).
    \end{equation}
    Observe that $(X+ c)^2$ follows a non-central chi-squared distribution with non-centrality parameter $c^2$ and $1$ degree of freedom. Let $Y$ be a Poisson random variable with mean $\frac{c^2}{2}$. 
    It is also well known \cite{pav2015moments} that
    \begin{equation} \label{eqn:equal-distribution-noncentral-chi}
        (X+c)^2 \overset{d}=Z_{1+2Y}.
    \end{equation} 
    Using \eqref{eqn:log-chi-squared-expectation} and \eqref{eqn:equal-distribution-noncentral-chi}, we get
    \begin{align*}
         \E[\log(X+c)^2] &= \E[\log Z_{1+2Y}] \\
         &= \E_Y[\E[\log Z_{1+2Y} \mid Y]] \\
         &= \E_Y\squ{\E\squ{ \log 2 + \psi\paren{\frac{1}{2}+Y} \; \middle| \; Y}} \\
         &= \log 2 +\E_Y\squ{ \psi\paren{\frac{1}{2}+Y} } 
    \end{align*}
    By the recurrence of $\psi$, we have 
    \begin{equation}
        \E_Y\squ{\psi\paren{\frac{1}{2}+Y} }  =  \psi(1/2) +\E_Y\squ{ \sum_{j=0}^{Y-1}\frac{2}{2j+1} }.
    \end{equation}
    Using the functions $f$ and $g$ defined in Lemma~\ref{lem:poisson}, the lemma implies
    \begin{equation*}
        \E_Y\squ{\psi\paren{\frac{1}{2}+Y} } = \psi(1/2) + \E[f(Y)] = \psi(1/2) + g(c^2/2) \ge \psi(1/2) + \log(1+c^2).
    \end{equation*}
    Therefore, we have
    \begin{align*}
        \E \squ{\log(X+c)^2} 
        &= \log2 + \E_Y\squ{\psi\paren{\frac{1}{2}+Y}} \\
        &\ge \log2 + \psi(1/2) + \log(1+c^2) \\
        &= \E\squ{\log X^2} + \log(1+c^2)
    \end{align*}
    where the last equality is obtained from \eqref{eqn:log-chi-squared-expectation}. 
    This completes the proof.
\end{proof}

\begin{lemma}\label{lem:non-zero} 
    Assume that $\mu_x \neq 0$ and $D \neq 0$. Under the policy $\pi^{CE}$, we have
    \begin{equation*}
        \mathbb{P}(\hat x_t^2 = 0) = 0
    \end{equation*}
    for all $t \in \Z_{\ge 0}$.
\end{lemma}

\begin{proof}
    By Lemma~\ref{lem:multiplicative}(i), since $\xhat_0 = \mu_x$, it follows that
    \[
    \xhat_t = \mu_x \prod_{s=0}^{t-1} M_s.
    \]
    As $\mu_x \neq 0$, the event $\{\xhat_t = 0\}$ is contained in $\bigcup_{s=0}^{t-1}\{M_s=0\}$.
    Fix $s$. 
    By Lemma~\ref{lem:multiplicative}(ii), we have
    \[
    \mathbb P(M_s = 0 \mid \Fcal_s) = \mathbb P_{Z \sim \Ncal(0,1)}(D_s + q_sZ=0) = 
    \begin{cases}
        0 & \text{if } q_s > 0, \\
        \mathbf 1\{D_s = 0\} & \text{if } q_s = 0.
    \end{cases}
    \]
    Recall that, from Lemma~\ref{lem:Duniform}, we have
    \[
    D_s = \frac{AR-BS}{R+B^2P_{s+1}}
    \quad\text{and}\quad
    D= \frac{AR-BS}{R+B^2P}.
    \]
    Since $D\neq 0$, we have $AR\neq BS$, and therefore, $D_s \neq 0$.
    Hence, in both cases, we have $\mathbb P(M_s = 0 \mid \Fcal_s) = 0$, and so $\mathbb P(M_s=0) = 0$.
    A union bound gives
    \[
    \mathbb P(\xhat_t=0) \le \sum_{s=0}^{t-1} \mathbb P(M_s = 0) = 0
    \]
    which completes the proof.
    
\end{proof}
\begin{lemma}\label{lem:uniform}
     Assume $D \neq 0$. Under the policy $\pi^{CE}$, there exists constants $t_2,\ell > 0$ such that \begin{equation}
        \E\squ{ \log\xhat_{t+1}^2 -  \log\xhat_{t}^2 \mid H_t} \geq -\ell
    \end{equation}
    for all $t \leq T-t_2$ and $\xhat_{t}^2 > 0$.
\end{lemma}
\begin{proof}
    Since $D^2 > 0$, it follows by Lemma~\ref{lem:riccaticont} there exists $t_2 > 0$ such that $D_t^2 > \frac{D^2}{2}$ for all $t \leq T- t_2$. 
    Let $t \leq T - t_2$. When $q_t = 0$, we have $\E\squ{ \log\xhat_{t+1}^2 -  \log\xhat_{t}^2 \mid H_t} = \log D_t^2 > \log\frac{D^2}{2}$. When $q_t > 0$ we have \begin{align*}
        \E\squ{ \log\xhat_{t+1}^2 -  \log\xhat_{t}^2 \mid H_t}
        &= \E\squ{\log\paren{D_t + q_t Z}^2 \mid H_t} \\&\geq \inf_{a > 0}\E\squ{\log \paren{D_t + aZ}^2} \\
        &= \inf_{a > 0} \curl{\log a^2 +\E\squ{\log \paren{\frac{D_t}{a}+ Z}^2}} \\
        &\geq \inf_{a > 0} \curl{\log a^2 +\E\squ{\log  Z^2} + \log \paren{1+ \frac{D_t^2}{a^2}}  } \tag{Lemma~\ref{lem:logtightened}} \\
        &= \inf_{a > 0} \log(a^2 + D_t^2) +\E\squ{\log  Z^2} \\
        &= \log D_t^2 +\E\squ{\log  Z^2} \\
        &= \log D_t^2 - \log 2 - \gamma_E \\
        &> \log\frac{D^2}{2} - \log 2 - \gamma_E
    \end{align*}
    Hence we may take $\ell =  \log 2 + \gamma_E -\log\frac{D^2}{2}>0$ since $|D|<1$.
\end{proof}

\begin{lemma}\label{lem:unstable2}
    Fix $\varepsilon,\delta> 0 $, with $\delta \leq \bsig$. 
    Suppose 
    \begin{equation} 
        \log\paren{\frac{A^2\paren{\bsig - \delta}^2K^2C^2}{\sigma_z^2 + K^2C^2\paren{\bsig - \delta}\delta  } + D^2}  -\gamma_E - \log2 - \varepsilon> 0.
    \end{equation} 
    Then there exists $t_3$ such that, for any $t \leq T - t_3$, 
    \begin{equation}
          \E_{\pi^{CE}}\squ{ \log\xhat_{t+1}^2 -  \log\xhat_{t }^2 \mid H_t}> \varepsilon
    \end{equation}
    whenever $0 < \xhat_{t}^2 \leq \delta$ and $\bsig-\Sigma_t \leq \delta$.
\end{lemma}
\begin{proof}
    Define a function $g$ on $\R$ such that
    \begin{equation}
        g(k) = \log\paren{\frac{A^2\paren{\bsig - \delta}^2k^2C^2}{\sigma_z^2 + k^2C^2\paren{\bsig - \delta}\delta  } + \paren{A + B k}^2}  -\gamma_E - \log2 - \varepsilon.
    \end{equation}
    Note that $g$ is continuous in $k$ and $g(K) > 0$ by assumption. 
    It follows by Lemma~\ref{lem:riccaticont} that there exists $t_3$ such that $g(K_t) > 0$ for all $t \leq T- t_3$. 
    Fix a $t$ with $t \leq T - t_3$. 
    We have 
    \begin{align*}
        \E\squ{ \log\xhat_{t+1}^2 -  \log\xhat_{t}^2 \mid H_t}&= \E\squ{\log M_t^2 \mid H_t} \\&=\E\squ{\log\paren{D_t + q_t Z}^2 \mid H_t}  \tag{Lemma~\ref{lem:multiplicative}}\\
        &= \log\paren{q_t^2} + \E\squ{\log\paren{\frac{D_t}{q_t} +  Z}^2 \;\middle|\; H_t}.
    \end{align*}
    By Lemma~\ref{lem:logtightened} we have 
    \begin{align*}
        \E\squ{\log\paren{\frac{D_t}{q_t} +  Z}^2 \;\middle|\; H_t} 
        \;\geq\; 
        \log\paren{1+\paren{\frac{D_t}{q_t}}^2}+ \E[\log Z^2] 
        \; = \; 
        \log\paren{1 + \paren{\frac{D_t}{q_t}}^2}-\gamma_E - \log2.
    \end{align*}
    Hence, we get 
    \begin{align*}
        \E\squ{ \log\xhat_{t+1}^2 -  \log\xhat_{t}^2 \mid H_t} 
        &\geq\;
        \log\paren{q_t^2}+ \log\paren{1 + \paren{\frac{D_t}{q_t}}^2} -\gamma_E - \log2 \\
        &=\; \log\paren{q_t^2 + D_t^2}  -\gamma_E - \log2 .
    \end{align*}
    If $\xhat_{t}^2 \leq \delta$ and $\bsig-\Sigma_t \leq \delta$
    we have 
    \begin{equation*}
        q_t^2 \geq \frac{A^2\paren{\bsig - \delta}^2K_t^2C^2}{\sigma_z^2 + K_t^2C^2\paren{\bsig - \delta}\delta}
    \end{equation*}
    Therefore, we conclude
    \begin{align*}
         \E\squ{ \log\xhat_{t+1}^2 -  \log\xhat_{t}^2 \mid H_t} &\geq \log\paren{q_t^2 + D_t^2}  -\gamma_E - \log2 \\
         &\geq \log\paren{\frac{A^2\paren{\bsig - \delta}^2K_t^2C^2}{\sigma_z^2 + K_t^2C^2\paren{\bsig - \delta}\delta  } + D_t^2}  -\gamma_E - \log2 \\
         &= g(K_t) + \varepsilon \\
         &> \varepsilon 
    \end{align*}
    and the proof is done.
\end{proof}

\begin{lemma}\label{lem:lineargrowth}
     Fix $\varepsilon, \delta >0$, with $\delta \leq \bsig$.  Assume $\mu_x \neq 0$, $D \neq 0$, and suppose \begin{equation}\log\paren{\frac{A^2\paren{\bsig - \delta}^2K^2C^2}{\sigma_z^2 + K^2C^2\paren{\bsig - \delta}\delta  } + D^2}  -\gamma_E - \log2 - \varepsilon> 0.
    \end{equation} 
    Then there exists constants $t_4,\ell> 0$ such that \begin{equation}
        \sum_{t=0}^{T-t_4} \curl{\mathbb P \paren{\xhat_{t}^2 > \delta} + \mathbb P \paren{\bsig-\Sigma_t > \delta} } \geq \frac{\paren{T-t_4}\varepsilon+ \log\paren{\frac{\mu_x^2}{C_{CE}}}}{\varepsilon + \ell}
    \end{equation}
\end{lemma}
\begin{proof}
    Take $t_4 = \max \{t_2,t_3\}$ where $t_2,t_3$ are defined in Lemmas~\ref{lem:uniform} and \ref{lem:unstable2}. 
    Fix a $t$ with $t\leq T -t_4$. 
    Recall the set $\Mcal_\delta$ defined in \eqref{eq:deltaset}.
    Then,
    \begin{align*}
        \E\squ{ \log\xhat_{t+1}^2 -  \log\xhat_{t}^2 \mid H_t} &= \mathbf{1}\curl{\paren{\xhat_{t}, \Sigma_{t}} \not\in  \mathcal{M}_{\delta}}\E\squ{ \log\xhat_{t+1}^2 -  \log\xhat_{t}^2 \mid H_t} \\
        &\quad+\mathbf{1}\curl{\paren{\xhat_{t}, \Sigma_{t}} \in  \mathcal{M}_{\delta}}\E\squ{ \log\xhat_{t+1}^2 -  \log\xhat_{t}^2 \mid H_t} \\
        &\geq -\mathbf{1}\curl{\paren{\xhat_{t}, \Sigma_{t}} \not\in  \mathcal{M}_{\delta}}\ell + \varepsilon \mathbf{1}\curl{\paren{\xhat_{t}, \Sigma_{t}} \in  \mathcal{M}_{\delta}}  \tag{Lemmas~\ref{lem:uniform} and \ref{lem:unstable2}}
    \end{align*}
    Since $\mu_{x} \neq 0$, by Lemma~\ref{lem:non-zero} we have $\xhat_{t}^2 >0$ almost surely for all $t$. Hence, the tower property says
    \begin{align*}
        \E\squ{ \log\xhat_{t+1}^2 -  \log\xhat_{t}^2} &\geq -\mathbb{P}\paren{\paren{\xhat_{t}, \Sigma_{t}} \not\in  \mathcal{M}_{\delta}}\ell + \varepsilon\mathbb{P}\paren{\paren{\xhat_{t}, \Sigma_{t}} \in  \mathcal{M}_{\delta}} \\
        &=-\mathbb{P}\paren{\paren{\xhat_{t}, \Sigma_{t}} \not\in  \mathcal{M}_{\delta}}\ell + \varepsilon\paren{1- \mathbb{P}\paren{\paren{\xhat_{t}, \Sigma_{t}} \not\in  \mathcal{M}_{\delta}}}\\
        &=\varepsilon -\mathbb{P}\paren{\paren{\xhat_{t}, \Sigma_{t}} \not\in  \mathcal{M}_{\delta}}\paren{\varepsilon + \ell}.
\end{align*}
By rearranging the term and summing over $t$ from $0$ to $T-t_4$, we have
\begin{equation*}
    \sum_{t=0}^{T-t_4}\mathbb{P}\paren{\paren{\xhat_{t}, \Sigma_{t}} \not\in  \mathcal{M}_{\delta}} \geq \frac{\paren{T-t_4} \varepsilon - \E[\log\xhat_{t+1}^2] + \log \mu_x^2}{ \varepsilon + \ell}
\end{equation*}
By the union bound, we get
\begin{equation}
    \sum_{t=0}^{T-t_4} \curl{\mathbb P \paren{\xhat_{t}^2 > \delta} + \mathbb P \paren{\bsig-\Sigma_t > \delta}} 
    \geq \frac{\paren{T-t_4} \varepsilon - \E[\log\xhat_{t+1}^2] + \log \mu_x^2}{ \varepsilon + \ell}.
\end{equation}
Using that $\E[\xhat_{t+1}^2] \leq C_{CE}$ by Lemma~\ref{lem:kalmanbound} and Jensen's inequality finishes the proof.
\end{proof}

\section{Probing Policy}\label{app:cllb}

This section proves Proposition~\ref{prop:cllinear3}, which shows that the probing policy $\pi^{P}$ achieves a linear advantage over $\pi^{OL}$ when $q > q_2$, together with Theorem~\ref{thm:separation} and the threshold comparisons in Propositions~\ref{prop:constantthreshhold} and~\ref{prop:cllbub}.

After each probe, $\pi^{P}$ plays the CE action for at least the next $L - 1$ steps,
since probes are at least $L$ steps apart.
We call the probe together with these $L - 1$ steps a \emph{probe block}.
A probe incurs its cost immediately, while its benefit arrives over the rest of the block, as the subsequent CE actions exploit the information it generated.
A single step may therefore contribute negatively, so we analyze each probe block as a whole.

The main step is to show that, when $q > q_2$, every probe block pays off in expectation (Lemma~\ref{prop:probebound}).
For a small probe, both its cost and the value of the information it triggers scale with the square of the probe size, so whether the block pays off depends only on the instance through $q$.
We first establish this for a block starting at the steady state (Lemma~\ref{lem:GL1}), where the computation is explicit, and then extend it by continuity (Lemma~\ref{lem:Gcontinuity}) to the blocks that actually occur, which start near the steady state.
Every other time step lies outside $\mathcal M_\delta$ and contributes a constant advantage on average.
Summing over the horizon gives Proposition~\ref{prop:cllinear3}.
The remaining results then follow quickly from Proposition~\ref{prop:cllinear3}, Proposition~\ref{prop:constantscalar}, and the stationary Riccati identities.
Figure~\ref{fig:dependencies-g} describes the proof dependencies in Section~\ref{app:cllb}.
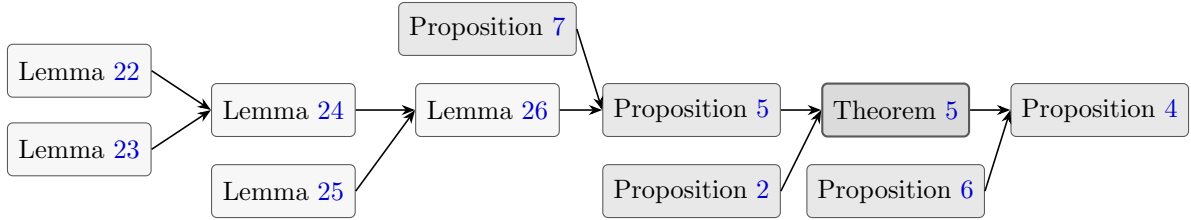
\begin{figure}[htbp]
    \centering
    \begin{tikzpicture}[
        y=0.65cm,
        dependency result/.style={
            draw=black!65, rounded corners=2pt,
            minimum width=1.9cm, minimum height=0.7cm,
            inner sep=4pt, font=\small, align=center
        },
        dependency lemma/.style={dependency result, fill=black!3},
        dependency proposition/.style={dependency result, fill=black!9},
        dependency theorem/.style={dependency result, fill=black!14, line width=0.8pt},
        dependency uses/.style={-{Stealth[length=1.8mm]}, line width=0.6pt}
    ]
        \node[dependency lemma] (n0) at (0,0.8) {Lemma~\ref{lem:fderivs}};
        \node[dependency lemma] (n1) at (0,-0.8) {Lemma~\ref{lem:greaterthan1}};
        \node[dependency lemma] (n2) at (2.7,0) {Lemma~\ref{lem:GL1}};
        \node[dependency lemma] (n3) at (2.7,-1.65) {Lemma~\ref{lem:Gcontinuity}};
        \node[dependency lemma] (n4) at (5.4,0) {Lemma~\ref{prop:probebound}};
        \node[dependency proposition] (n5) at (5.4,1.65) {Proposition~\ref{prop:regretdecomp2}};
        \node[dependency proposition] (n6) at (8.1,0) {Proposition~\ref{prop:cllinear3}};
        \node[dependency proposition] (n7) at (8.1,-1.65) {Proposition~\ref{prop:constantscalar}};
        \node[dependency theorem] (n8) at (10.8,0) {Theorem~\ref{thm:separation}};
        \node[dependency proposition] (n9) at (10.8,-1.65) {Proposition~\ref{prop:cllbub}};
        \node[dependency proposition] (n10) at (13.5,0) {Proposition~\ref{prop:constantthreshhold}};

        \draw[dependency uses] (n0.east) -- (n2.west);
        \draw[dependency uses] (n1.east) -- (n2.west);
        \draw[dependency uses] (n2.east) -- (n4.west);
        \draw[dependency uses] (n3.east) -- (n4.west);
        \draw[dependency uses] (n4.east) -- (n6.west);
        \draw[dependency uses] (n5.east) -- (n6.west);
        \draw[dependency uses] (n6.east) -- (n8.west);
        \draw[dependency uses] (n7.east) -- (n8.west);
        \draw[dependency uses] (n8.east) -- (n10.west);
        \draw[dependency uses] (n9.east) -- (n10.west);
    \end{tikzpicture}
    \caption{Main proof dependencies for Section~\ref{app:cllb}.
    An arrow points from a result to a result whose proof uses it.
    Common preliminary results are omitted.}
    \label{fig:dependencies-g}
\end{figure}

\subsection{Proof of Theorem~\ref{thm:separation}}
\begin{proof}
Recall that $|A|<1$ and $|D|<1$. In particular, $1-AD>0$, and
\begin{align*}
    q_3-q_2
    &=(1-D^2)-\frac{(1-A^2)(1-D^2)}{2(1-AD)} 
    =\frac{(1-D^2)\bigl((A-D)^2+1-D^2\bigr)}{2(1-AD)}>0.
\end{align*}
Hence $q_2<q_3$.
For $q>q_2$, by Proposition~\ref{prop:cllinear3}, there exists a probing policy $\pi^P$ such that $\R_T(\pi^P,\pi^{OL})=\Omega(T)$.
By optimality of $\pi^\star$, we have
$J_T(\pi^\star)\le J_T(\pi^P)$. Therefore, we have
\begin{align*}
    R_T(\pi^\star,\pi^{OL})
    =J_T(\pi^{OL})-J_T(\pi^\star) 
    \ge J_T(\pi^{OL})-J_T(\pi^P) 
    =R_T(\pi^P,\pi^{OL})
    =\Omega(T),
\end{align*}
which proves Theorem~\ref{thm:separation}(i). The result of Theorem~\ref{thm:separation}(ii) is clear by Proposition~\ref{prop:constantscalar}.
\end{proof}

\subsection{Proof of Proposition~\ref{prop:constantthreshhold}}

\begin{proof}
    This is directly implied by Theorem~\ref{thm:separation} and Proposition~\ref{prop:cllbub}.
\end{proof}
\subsection{Proof of Proposition~\ref{prop:cllinear3}}
\begin{proof} 
    Denote 
    \begin{equation}
        B_k = \sum_{t=0}^{L -1} r_{\tau_k+t,\lambda}(\xhat_{\tau_k +t}, \Sigma_{\tau_k + t}, u_{\tau_k + t}).
    \end{equation}
    By Lemma~\ref{prop:probebound}, there exist
    integers $L \ge 1$, $t_6 \ge 0$, constants $\delta,a,c_1>0$, and $\lambda \in (0,1)$, all independent of $T$, such that $\pi^P$ defined with parameters $\delta,a,L,t_6$ satisfies 
    \begin{equation} \label{eq:probebound-prop}
        \E_{\pi^P} \squ{B_k \mid H_{\tau_k}} \ge c_1a^2
    \end{equation}
    on $\tau_k <\infty$, for every $k \ge 1$.
    Here, the probing times use the cutoff $T-t_6-L$ and $\Mcal_\delta$.
    By Lemma~\ref{lem:riccaticont}, there exists $t_5>0$ such that
    \begin{equation}\label{eq:lemma8-app}
        K_t^2 > \frac{K^2}{2}, \quad N_t > \frac{N}{2}, \quad \Gamma_{t+1} > \frac{\Gamma}{2}
    \end{equation}
    for all $t \le T-t_5$.
    Let $t_P := \max\{t_5,t_6\} + L$.
    Recall $\Tcal = \{\tau_k \colon k\ge 1, \, \tau_k<\infty\}$, and let $M=|\Tcal|$.
    For the following set 
    \begin{equation} 
        \mathcal{Y} := \bigcup_{k:\tau_k < \infty }\{\tau_k, \dots, \tau_k + L -1\},
    \end{equation}
    we have $|\mathcal{Y}| = L M$.
    Notice that we have
    \begin{equation*}
        \sum_{t \in \Ycal} r_{t,\lambda}(\xhat_t,\Sigma_t,u_t) = \sum_{k \ge 1} \mathbf{1}\{\tau_k<\infty\}B_k
    \end{equation*}
    From \eqref{eq:probebound-prop}, the tower property gives the following bound
    \begin{align*}
        \E_{\pi^P}[\mathbf{1} \curl{\tau_k < \infty}  B_k] &= \E_{\pi^P}\big[\mathbf{1} \curl{\tau_k < \infty}  \E_{\pi^P} \squ{B_k \mid H_{\tau_k}} \big] \\
        &\geq c_1a^2 \cdot \E\squ{\mathbf{1} \curl{\tau_k < \infty}}
    \end{align*}
    Since $M = \sum_{k \geq 1}\mathbf{1}\{\tau_k <\infty\}$, we have \begin{align*}
        \E\squ{\sum_{t \in \mathcal{Y}} r_{t,\lambda}(\xhat_{t}, \Sigma_{t}, u_{t})} \geq c_1a^2 \E[M].
    \end{align*}
    Define 
    \begin{equation*}
        \Vcal = \{t \in \Z \colon 0\le t \le T-t_P\}
        \quad \text{and} \quad
        \mathcal{Z} := \mathcal{V} \setminus \mathcal{Y}.
    \end{equation*}
    For each $t \in \mathcal{Z}$ we have $(\xhat_t, \Sigma_t) \not\in \mathcal{M}_\delta$. This is because, since $t \in \mathcal{V}$, if $(\xhat_t, \Sigma_t) \in \mathcal{M}_\delta$ then either a probe will happen at $t$ or has occurred fewer than $L$ steps ago, implying $t \in \mathcal{Y}$, which is a contradiction. Thus \begin{equation*}
        t \in \mathcal{Z} \implies \xhat_t^2 > \delta \text{ or }\bsig -\Sigma_t > \delta
    \end{equation*} 
    For the case when $\xhat_t^2 > \delta$, we have 
    \begin{align*}
        r_{t,\lambda}(\xhat_{t}, \Sigma_{t}, u_{t})  
        &= \lambda \Gamma_{t+1} F(\Sigma_t,K_t\xhat_t) + \paren{1-\lambda } N_t K_t^2\squ{\bsig - \Sigma_t} \\
        &\geq \lambda \Gamma_{t+1} \frac{A^2K_t^2 \delta^2C^2\Sigma_w^2}{\sigma_z^2 +K_t^2 \delta^2C^2 \Sigma_w} \tag{$\xhat_t^2>\delta$, $\Sigma_t \ge \Sigma_w$} \\
        &\geq \lambda \cdot \frac{\Gamma }{2} \frac{\frac{1}{2}A^2K^2 \delta^2C^2\Sigma_t^2}{\sigma_z^2 +\frac{1}{2}K^2 \delta^2C^2 \Sigma_t} & \tag{by \eqref{eq:lemma8-app}}
    \end{align*}
    When $\bsig-\Sigma_t > \delta$ we have 
    \begin{align*}
        r_{t,\lambda}(\xhat_{t}, \Sigma_{t}, u_{t})&= \lambda \Gamma_{t+1} F(\Sigma,K_t\xhat_t) + \paren{1-\lambda} N_t K_t^2\squ{\bsig - \Sigma_t} \\ 
        &\geq \paren{1-\lambda } N_t K_t^2\squ{\bsig - \Sigma_t} \\
        &\geq \paren{1-\lambda } \frac{NK^2}{4}\squ{\bsig - \Sigma_t} \tag{by \eqref{eq:lemma8-app}}  \\
        &\geq \paren{1-\lambda } \cdot  \frac{NK^2}{4}\delta 
    \end{align*}
    Let
    \begin{equation*}
        c_2 := \min\curl{\lambda \cdot \frac{\Gamma }{2} \frac{\frac{1}{2}A^2K^2 \delta C^2\Sigma_w^2}{\sigma_z^2 +\frac{1}{2}K^2 \delta C^2 \Sigma_w}, \paren{1-\lambda } \cdot  \frac{NK^2}{4}\delta } > 0
    \end{equation*}
    It follows that $ r_{t,\lambda}(\xhat_{t}, \Sigma_{t}, u_{t}) \geq c_2$ for all $t \in \mathcal{Z}$. So 
    \begin{align*}
        \E\squ{\sum_{t \in \mathcal{Z}} r_{t,\lambda}(\xhat_{t}, \Sigma_{t}, u_{t})} \geq c_2 \E[|\mathcal{Z}|].
    \end{align*}
    Finally, let $E = \{0,\dots, T\}\setminus \paren{\mathcal{Y} \cup \mathcal{Z}}$.
    Note for all $t \in E$, $u_t = K_t \xhat_t$ so $r_{t,\lambda}(\xhat_{t}, \Sigma_{t}, u_{t}) \geq 0$. 
    Hence, we have 
    \begin{align*}
        \sum_{t=0}^{T} \E_{\pi^{P}}\!\squ{r_{t,\lambda}(\xhat_{t}, \Sigma_{t}, u_t)} &=\E_{\pi^{P}}\squ{\sum_{t \in E} r_{t,\lambda}(\xhat_{t}, \Sigma_{t}, u_{t}) + \sum_{t \in \mathcal{Y}} r_{t,\lambda}(\xhat_{t}, \Sigma_{t}, u_{t}) + \sum_{t \in \mathcal{Z}} r_{t,\lambda}} \\
        &\geq c_1a^2 \E[M]+c_2 \E[|\mathcal{Z}|] \\
        &= c_1a^2 \E\squ{\frac{|\mathcal{Y}|}{L}}+c_2 \E[|\mathcal{Z}|] \\
        &\geq \min\curl{\frac{c_1a^2}{L}, c_2} \cdot \E \squ{|\mathcal{Y}| +|\mathcal{Z}| } \\
        &\geq \min\curl{\frac{c_1a^2}{L}, c_2} \cdot \E \squ{|\mathcal{V}|} \\
        &= \min\curl{\frac{c_1a^2}{L}, c_2} \cdot \paren{T- t_P +1}.
    \end{align*}
    Therefore, by Proposition~\ref{prop:regretdecomp2}, we now conclude that $R_T(\pi^P,\pi^{OL}) = \Omega(T)$.
\end{proof}
\subsection{Proof of Proposition~\ref{prop:cllbub}}
\begin{proof}
    Throughout the proof, we use the fact that $|A|<1$ and $|D|<1$.
    Recall the quantities that we want to compare:
    \begin{equation*}
        q_1 = K^2 \cdot \frac{1-A^2}{Q}\cdot\paren{R-\frac{S^2}{Q}}, 
        \qquad 
        q_2 = \frac{(1-A^2)(1-D^2)}{2(1-AD)}.
    \end{equation*}
    We use the following 3 identities \begin{equation}\label{eq:p1}
        N = R + B^2 P
    \end{equation}
    \begin{equation}\label{eq:p2}
        -KN = S + BP A
    \end{equation}
    \begin{equation}\label{eq:p3}
        P = A^2P + Q - K^2N.
    \end{equation}
    The identities \eqref{eq:p2} and \eqref{eq:p3} give different expressions for $K^2N$:
    \begin{align}
        K^2N &= Q-(1-A^2)P \label{eq:p4}\\
        K^2N &= -K(S+BPA) = -KS-DPA+PA^2 \notag
    \end{align}
    where we used $BK=D-A$ in the second expression.
    Equating the two expressions for $K^2N$ and solving for $KS$, we get
    \begin{equation*}
        KS = (1-AD)P-Q.
    \end{equation*}
    Further, the identities \eqref{eq:p1} and \eqref{eq:p4} give
    \begin{equation*}
        K^2R = K^2N - K^2B^2P = Q-(1-A^2)P - (D-A)^2P = Q-(1-2AD+D^2)P
    \end{equation*}
    where we again used $BK=D-A$.
    Hence, we get
    \begin{align*}
        K^2\paren{R - \frac{S^2}{Q}} 
        &= K^2R - \frac{\paren{KS}^2}{Q} \\
        &= Q-\paren{1 -2AD+ D^2}P - \frac{\big((1 -AD)P-Q\big)^2}{Q}\\
        &=Q-\paren{1 -2AD+ D^2}P - Q +2(1 -DA)P- \frac{(1 -AD)^2P^2}{Q} \\
        &=  P\paren{1-D^2}- \frac{(1-AD)^2P^2}{Q}.
    \end{align*}
    Multiplying by $(1-A^2)/Q$ gives
    \begin{align*}
        q_1 &= K^2 \cdot  \frac{1-A^2}{Q}\cdot \paren{R - \frac{S^2}{Q}} \\
        &= \frac{1-A^2}{Q} \paren{P\paren{1-D^2}- \frac{(1 -AD)^2P^2}{Q}} \\
        &= \paren{1-A^2} \paren{\frac{\paren{1-D^2}^2}{4\paren{1-AD}^2} - \frac{\paren{1-D^2}^2}{4\paren{1-AD}^2}+\frac{P\paren{1-D^2}\paren{1-AD}}{Q\paren{1-AD}}- \frac{(1 -AD)^2P^2}{Q^2}} \\
        &= \paren{1-A^2} \paren{\frac{\paren{1-D^2}^2}{4\paren{1-AD}^2} - \paren{\frac{\paren{1-D^2}}{2\paren{1-AD}} -\frac{\paren{1-AD}P}{Q}}^2} \\
        &\leq \paren{1-A^2} \frac{\paren{1-D^2}^2}{4\paren{1-AD}^2} \\
        &= q_2 \frac{1-D^2}{2\paren{1-AD}}
    \end{align*}
    One can observe that $2\paren{1{-}AD}{-}\paren{1{-}D^2} \ge 1{-}A^2 >0$, 
    that is,
    \begin{equation*}
        \frac{1-D^2}{2\paren{1-AD}} < 1.
    \end{equation*}
    Therefore, we conclude that $q_1<q_2$.
\end{proof}

\subsection{Lemmas}

Throughout this section, we analyze a single probe block in isolation.
Fix a block length $L$, a positive integer, and a probe size $a>0$.
The Riccati coefficients used within the block are collected in a vector 
\begin{equation} \label{eq:theta-pgamma}
    \theta = \paren{p_1, \dots, p_L, \gamma_1, \dots, \gamma_L} \in [0,\infty)^{2L},
\end{equation}
where, for a block starting at time $\tau$, the coordinates $p_{t+1}$ and $\gamma_{t+1}$ play the roles of
$P_{\tau+t+1}$ and $\Gamma_{\tau+t+1}$.
For $t = 0, \dots, L-1$, the corresponding gain is
\begin{equation} \label{eq:block-gains}
    n_t = R + B^2 p_{t+1}, \qquad k_t = -n_t^{-1}\paren{Bp_{t+1}A + S},
\end{equation}
where $n_t \ge R > 0$ since $p_{t+1} \ge 0$.
The block policy $\pi^{\theta,a}$ plays the role of $\pi^{\mathrm P}$ during one block:
it probes once at the start and then follows the CE action,
\[
    \pi_0^{\theta,a}(H_0) = k_0\xhat_0 + a, \qquad \pi_t^{\theta,a}(H_t) = k_t\xhat_t \quad (1 \le t \le L-1).
\]
For $\lambda \in [0,1]$, the per-step payoff and the expected payoff of a block are
\begin{align}
    r^{\theta}_{t,\lambda} \paren{\xhat, \Sigma, u} 
    &= - n_t \paren{u - k_t \xhat}^2 +    \paren{\lambda \cdot\gamma_{t+1} \cdot F(\Sigma,u) + \paren{1-\lambda } \cdot  n_t k_t^2\squ{\bsig - \Sigma}}, \label{eq:rthetalambda}\\
    \mathcal{G}_{L,\lambda}\paren{a; \xhat,\Sigma, \theta} 
    &= \sum_{t=0}^{L-1} \E_{\pi^{\theta,a}} \squ{r^{\theta}_{t,\lambda } \paren{\xhat_t, \Sigma_t, u_t} \mid \xhat_0 = \xhat,  \Sigma_0 =\Sigma}. \label{eq:GcalLlambda}
\end{align}
The stationary coefficients $\bar\theta$ is defined as
\begin{equation} \label{eq:stationary-theta}
    \bar{\theta} = \paren{P, \dots , P,\Gamma, \dots , \Gamma} \in \R^{2L},
\end{equation}
and the expected payoff of the block starting at the steady state is denoted as 
\begin{equation} \label{eq:f_L(a)}
    f_L(a) = \Gcal_{L,1}\paren{a;0,\bsig,\bar\theta}.
\end{equation}

\begin{lemma} \label{lem:fderivs}
We have \begin{equation}
    \liminf_{a \downarrow 0} \frac{f_L(a)}{a^2} \geq N \paren{g(L) -1},
\end{equation}
where
\begin{equation}
        g(L)  =\frac{q}{1-A^2} \paren{1 + \paren{\paren{D -A}^2+ q}\sum_{j=0}^{L-2}\paren{D^2+q}^j}
    \end{equation}
\end{lemma}
\begin{proof}
    Fix independent standard Gaussian random variables $Z_0,\dots,Z_{L-1}$.
    For every $a>0$, construct the process $\curl{\xhat_t(a), \Sigma_t(a), u_{t}(a)}_{t\geq 0}$ by $\xhat_0(a) = 0$, $\Sigma_0(a) = \bsig$, $u_0(a) = a$, and for $t \ge 0$, 
    \begin{align*}
        \xhat_{t+1}(a) &= A\xhat_t(a) + Bu_t(a) + \frac{AC\Sigma_t(a) u_t(a)}{\sqrt{\sigma_z^2 +u_t(a)^2 C^2 \Sigma_t(a)}}\cdot Z_t \\
        \Sigma_{t+1}(a) &= A^2\Sigma_t(a) + \Sigma_w - F\paren{\Sigma_t(a),u_t(a)}, \\
        u_{t+1}(a) &= K\xhat_{t+1}(a),
    \end{align*}
    with the same Gaussian sequence.

    By Lemma~\ref{lem:gaussiankalman} and the covariance recursion, for each fixed $a$, the constructed process has the same joint distribution as the belief-state and action process under $\pi^{\bar\theta,a}$, initialized at $(0,\bsig)$.
    Consequently,
    \begin{equation}
        f_L(a) =-N a^2 + \Gamma \,\E\squ{\sum_{t=0}^{L-1}  F\paren{\Sigma_t(a), u_t(a)}}.
    \end{equation} 
    Since all processes are driven by the same $Z$, we may fix a realization of $Z$ and study pointwise limits as $a \downarrow 0$.
    
    For the initial values, we trivially have
    \[
    \xhat_0(a) = 0 \to 0, 
    \quad 
    \Sigma_0(a) = \bsig \to \bsig, \quad 
    u_0(a) = a \to 0 
    \quad \text{as } a \downarrow 0.
    \]
    Furthermore 
    \begin{equation}
        F\paren{\Sigma_0(a), u_0(a)} = \frac{A^2u_{0}(a)^2C^2\Sigma_{0}(a)^2}{\sigma_z^2 +u_{0}(a)^2C^2 \Sigma_{0}(a)} \to 0 \quad \text{as } a \downarrow 0.
    \end{equation}
    Inductively, it is clear that for all $t$, 
    \[
    \xhat_t(a) \to 0, 
    \quad 
    \Sigma_t(a) \to \bsig, \quad 
    u_t(a)\to 0 
    \quad \text{as } a \downarrow 0.
    \]
    Define $Y_t(a) = \xhat_t(a)/a$, so that $Y_0(a) = 0$ and
    \begin{equation*}
        Y_1(a) = B  + \frac{AC\bsig}{\sqrt{\sigma_z^2 +a^2C^2 \bsig}} Z_0.
    \end{equation*}
    For $t\geq 1$ we have $u_t(a) = K\xhat_t(a) = aKY_t(a)$. 
    Dividing the recursion for $\xhat_t(a)$ by $a$,
    \begin{equation}
        Y_{t+1}(a) = A Y_t(a) + B K Y_t(a) + \frac{AC\Sigma_t(a) K Y_t(a)}{\sqrt{\sigma_z^2 +K^2\xhat_t(a)^2 C^2 \Sigma_t(a)}}Z_t.
    \end{equation}
    We define the limiting process by
    \begin{align*}
        Y_1 &= B+\frac{AC\bsig}{\sigma_z}Z_0, \qquad 
        Y_{t+1} = \paren{D+\frac{ACK\bsig}{\sigma_z}Z_t}Y_t \quad (t\ge1).
    \end{align*}
    Inductively, it is clear to see that $Y_t(a) \to Y_t$ as $a \downarrow0$ for all $t \ge 1$.

    For each $t \ge 1$, $Y_t$ is a function of $Z_0,\dots,Z_{t-1}$, so it is independent of $Z_t$.
    Therefore, we get
    \begin{equation*}
        \E[Y_1^2] = B^2 + \frac{A^2C^2\bsig^2}{\sigma_z^2}, \qquad \E[Y_{t+1}^2] = (D^2+q) \E[Y_t^2] \quad (t \ge 1)
    \end{equation*}
    hence
    \begin{align*}
        \E[Y_t^2] = \paren{B^2 + \frac{A^2C^2\bsig^2}{\sigma_z^2}} \paren{D^2 +q}^{t-1} \quad (t\ge 1).
    \end{align*}
    Observe that for $t \geq 1$
    \begin{equation*}
        \frac{F\paren{\Sigma_t(a), u_t(a)}}{a^2} = \frac{A^2K^2C^2\Sigma_{t}(a)^2}{\sigma_z^2 +K^2\xhat_t(a)^2C^2 \Sigma_{t}(a)}Y_{t}(a)^2
    \end{equation*}
    which implies that 
    \begin{equation*}
         \lim_{a\downarrow 0}\frac{F\paren{\Sigma_t(a), u_t(a)}}{a^2} = \frac{A^2K^2C^2\bsig^2}{\sigma_z^2 }Y_{t}^2 = qY_t^2.
    \end{equation*}
    By Fatou's Lemma, for $t \ge 1$, we have
    \begin{equation*}
        \liminf_{a \downarrow 0} \E \squ{\frac{F\paren{\Sigma_t(a), u_t(a)}}{a^2}} \geq  \E \squ{\liminf _{a \downarrow 0}\frac{F\paren{\Sigma_t(a), u_t(a)}}{a^2}} =q\E[Y_t^2] = q\paren{B^2 + \frac{A^2C^2\bsig^2}{\sigma_z^2}} \paren{D^2 +q}^{t-1}.
    \end{equation*}
    while for $t=0$, we get
    \begin{equation*}
        \lim_{a\downarrow 0}\frac{F\paren{\Sigma_0(a), u_0(a)}}{a^2} = \frac{A^2C^2\bsig^2}{\sigma_z^2}.
    \end{equation*}
    Therefore, 
    \begin{align*}
        \liminf_{a \downarrow 0}\frac{f_L(a)}{a^2} 
        &\geq -N + \Gamma\sum_{t=0}^{L-1} \liminf_{a \downarrow 0}\E\squ{\frac{F\paren{\Sigma_t(a), u_t(a)}}{a^2} } \\
        &\geq -N + \Gamma \frac{A^2C^2\bsig^2}{\sigma_z^2} + \Gamma\sum_{t=1}^{L-1}q\paren{B^2 + \frac{A^2C^2\bsig^2}{\sigma_z^2}} \paren{D^2 + q}^{t-1} \\
        &= -N + \frac{N K^2}{1-A^2} \frac{A^2C^2\bsig^2}{\sigma_z^2} + \frac{N K^2}{1-A^2} q\paren{B^2 + \frac{A^2C^2\bsig^2}{\sigma_z^2}}\sum_{t=0}^{L-2} \paren{D^2 + q}^{t} \\
        &= N \paren{-1 + \frac{q}{1-A^2}+ \frac{q}{1-A^2} \paren{\paren{D-A}^2+q}\sum_{t=0}^{L-2} \paren{D^2 + q}^{t}} \\
        &= N (g(L) -1)
    \end{align*}
    which completes the proof.
\end{proof}
\begin{lemma}\label{lem:greaterthan1}
    If $q > q_{2}$, then there exists finite L such that $g(L) > 1$.
\end{lemma}
\begin{proof}
    In the case when $D^2 + q \geq 1$ we have that $g(L) \rightarrow \infty$ as $L\to \infty$ and the claim follows trivially. When $D^2 + q < 1$ we have 
    \begin{equation*}
        g(L) \rightarrow \frac{q}{1-A^2} \paren{1 + \frac{\paren{D - A}^2+ q}{1-D^2-q}} \quad \text{as } L \to \infty.
    \end{equation*}
    Observe that
    \begin{equation} \label{eq:q2-range}
        \frac{q}{1-A^2} \paren{1 + \frac{\paren{D - A}^2+ q}{1-D^2-q}} >1
        \quad \Longleftrightarrow \quad
        q > \frac{\paren{1-A^2} \paren{1- D^2}}{2\paren{1-A D}} = q_{2}.
    \end{equation}
    Therefore, if $q>q_2$, \eqref{eq:q2-range} shows that the limit of $g(L)$ exceeds $1$, so by the definition of limit, there exists a finite $L$ such that $g(L) > 1$.
\end{proof}
\begin{lemma} \label{lem:GL1}
If $q > q_{2}$ there exist finite $L,a, c_0 > 0$ such that
    $ \mathcal{G}_{L,1}\paren{a; 0,\bsig, \bar{\theta}} \geq  c_0 \cdot a^2$.
\end{lemma}
\begin{proof}
    If $q>q_2$, there exists a finite $L_0$ such that $N\big(g(L_0)-1\big)>0$ by Lemma~\ref{lem:greaterthan1}.
    Next, by Lemma~\ref{lem:fderivs}, there exists $a_0$ such that if $0<a<a_0$ then
    \begin{equation}
         \frac{f_{L_0}(a)}{a^2} \geq \frac{1}{2}N \paren{g(L_0) -1}.
    \end{equation}
    by the definition of limit inferior.
    Fix $a \in (0,a_0)$.
    It follows that 
    \begin{equation}
        \mathcal{G}_{L_0,1}\paren{a; 0,\bsig, \bar{\theta}} = f_{L_0}(a) \geq  a^2\frac{1}{2}N \!\paren{g(L_0) -1} > 0
    \end{equation}
    Choosing $c_0 = \frac{1}{2}N \!\paren{g(L_0) -1}$ completes the proof.
\end{proof}
\begin{lemma}\label{lem:Gcontinuity}
    Fix $L \in \Z_{>0}$, $a>0$ and $ \lambda \in [0,1]$. 
    Then, the mapping 
    \[
    \paren{\xhat, \Sigma, \theta} \mapsto \mathcal{G}_{L,\lambda}\paren{a; \xhat,\Sigma, \theta}
    \]
    is continuous on $\R \times [\Sigma_w, \bsig] \times [0,\infty)^{2L}$.
\end{lemma}
\begin{proof}
    For $Z = (Z_0,\dots,Z_{L-1}) \in \R^L$, define the trajectory recursively by
    \[\xhat_0 = \xhat, \quad \Sigma_0 = \Sigma, \quad u_0 = k_0\xhat_0 +a\]
    and, for $0\le t \le L-2$,
    \begin{equation} \label{eq:trajrecursion}
    \begin{split}
        \xhat_{t+1} &= A\xhat_t + Bu_t + \frac{A\Sigma_tCu_t}{\sqrt{\sigma_z^2 + C^2\Sigma_t u_t^2}} Z_t \\
        \Sigma_{t+1} &= A^2 \Sigma_t + \Sigma_w - F(\Sigma_t,u_t) \\
        u_{t+1} &= k_{t+1}\xhat_{t+1}.
    \end{split}
    \end{equation}
    Using the trajectory, define
    \begin{equation*}
        \Phi_L(\xhat, \Sigma, \theta, Z) =-n_0 a^2 + \sum_{t=0}^{L-1} \squ{ \lambda \gamma_{t+1} F\paren{\Sigma_t, u_t} + \paren{1-\lambda} n_t k_t^2 \paren{\bsig - \Sigma_t} },
    \end{equation*}
    then this equals 
    \begin{equation*}
        \sum_{t=0}^{L-1} r^\theta_{t,\lambda}(\xhat_t,\Sigma_t,u_t).
    \end{equation*}
    By Lemma~\ref{lem:gaussiankalman}, when $Z \sim \Ncal(0,I_L)$, the constructed trajectory above has the same joint distribution as the belief-state and action process under $\pi^{\theta,a}$, initialized at $(\xhat,\Sigma)$.
    Hence,
    \begin{equation*}
        \Gcal_{L,\lambda}(a;\xhat,\Sigma,\theta) = \E_{Z \sim \Ncal(0,I_L)}\squ{\Phi_L(\xhat,\Sigma,\theta,Z)}.
    \end{equation*}
    Note that, for every $t$, the covariance gives
    \[
    \Sigma_w \le \Sigma_t \le \bsig, \quad \text{and} \quad
    0 \le F(\Sigma_t,u_t) \le A^2 \bsig.
    \]
    Let $(\xhat_n,\Sigma_m,\theta_m) \to (\xhat,\Sigma,\theta)$ be an arbitrary convergent sequence in $\R \times [\Sigma_w, \bsig] \times [0,\infty)^{2L}$.
    Since $\{\theta_m\}_{m=1}^\infty$ converges, its coordinates are uniformly bounded. For each parameter vector $\theta_m$, the corresponding coefficients satisfy
    \[
    n_t = R+B^2p_{t+1} \ge R >0, 
    \quad \text{and} \quad
    k_t = - \frac{Bp_{t+1}A+S}{n_t}.
    \]
    Hence, there exists $M<\infty$ such that 
    \[
    \max\{|n_t|,|k_t|,|\gamma_{t+1}|\} \le M
    \]
    uniformly over all $m$ and all $0\le t <L$.
    Consequently, 
    \begin{equation} \label{eq:phibound}
        |\Phi_L (\xhat_m,\Sigma_m,\theta_m,Z)| \le Ma^2 + L\paren{MA^2 \bsig + M^3 \bsig}
    \end{equation}
    independently of $m$ and $Z$.

    Next, we show pointwise convergence for each fixed $Z$.
    By the definitions of $n_t$ and $k_t$, together with $n_t \ge R>0$, these coefficients are continuous functions of $\theta$.
    This also implies that the map $(\Sigma,u) \mapsto F(\Sigma,u)$ is continuous.
    Moreover, $\sigma_z^2 + C^2\Sigma_tu_t^2 \ge \sigma_z^2 > 0$.
    Hence the updates in \eqref{eq:trajrecursion} are continuous in $(\hat x_t, \Sigma_t, u_t, \theta)$ for each fixed $Z$, and so is the initial action $u_0 = k_0\hat x_0 + a$.
    By induction on $t$, each of $\hat x_t$, $\Sigma_t$ and $u_t$ is a continuous function of
    $(\hat x, \Sigma, \theta)$ for each fixed $Z$.
    Since $\Phi_L$ is a finite sum of continuous functions of these quantities and of $\theta$,
    \begin{equation}
        \Phi_L(\hat x_m, \Sigma_m, \theta_m, Z) \to \Phi_L(\hat x, \Sigma, \theta, Z)
        \quad \text{as } m \to \infty \text{ for every } Z \in \mathbb R^L .
    \end{equation}
    The right-hand side of \eqref{eq:phibound} is a constant, so it is integrable with respect to $\mathcal N(0, I_L)$.
    By the dominated convergence theorem,
    \begin{align*}
        \lim_{m\to\infty} \Gcal_{L,\lambda}(a; \xhat_m, \Sigma_m, \theta_m)
        &= \lim_{m\to\infty} \E_{Z\sim\mathcal N(0,I_L)}\big[\Phi_L(\hat x_m, \Sigma_m, \theta_m, Z)\big] \\
        &= \mathbb E_{Z\sim\mathcal N(0,I_L)}\Big[\lim_{m\to\infty}\Phi_L(\hat x_m, \Sigma_m, \theta_m, Z)\Big] \\
        &= \mathbb E_{Z\sim\mathcal N(0,I_L)}\big[\Phi_L(\hat x, \Sigma, \theta, Z)\big] \\
        &= \mathcal G_{L,\lambda}(a; \hat x, \Sigma, \theta) .
    \end{align*}
    Since the convergent sequence is chosen arbitrarily, the mapping
    $(\hat x, \Sigma, \theta) \mapsto \mathcal G_{L,\lambda}(a; \hat x, \Sigma, \theta)$ is continuous on
    $\mathbb R \times [\Sigma_w, \bar\Sigma] \times [0,\infty)^{2L}$.
\end{proof}
\begin{lemma} \label{prop:probebound}
    If $q > q_{2}$, there exist
    integers $L \ge 1$, $t_6 \ge 0$, constants $\delta,a,c_1>0$, and $\lambda \in (0,1)$, all independent of $T$, such that the probing policy $\pi^P$ defined with parameters $\delta,a,L,t_6$ satisfies
    \begin{equation} \label{eq:probebound}
        \E_{\pi^P} \squ{\sum_{t=0}^{L -1} r_{\tau_k+t,\lambda}(\xhat_{\tau_k +t}, \Sigma_{\tau_k + t}, u_{\tau_k + t}) \;\middle|\; H_{\tau_k}}  \geq  c_1 \cdot a^2
    \end{equation}
    on $\tau_k <\infty$, for every $k \ge 1$.
    Here, the probing times use the cutoff $T-t_6-L$ and $\Mcal_\delta$.
\end{lemma}
\begin{proof}
    For any $L$ and $\lambda \in (0,1)$, observe that
    \begin{equation*}
        \Gcal_{L,\lambda}(a;\xhat,\Sigma,\theta) = \lambda\Gcal_{L,1}(a;\xhat,\Sigma,\theta) + (1-\lambda) \Gcal_{L,0}(a;\xhat,\Sigma,\theta)
    \end{equation*}
    for all $(\xhat,\Sigma,\theta)$.
    Also, note that $\mathcal{G}_{L,0}\paren{a; 0,\bsig, \bar{\theta}} \geq -N a^2$.
    Choose $L$ and $c_0>0$ such that $\Gcal_{L,1}\paren{a;0,\bsig,\bar\theta} \ge c_0a^2$.
    This is guaranteed by Lemma~\ref{lem:GL1}.
    For such $L$ and $c_0>0$, we get
    \begin{equation}
        \mathcal{G}_{L,\lambda} \paren{a; 0,\bsig, \bar{\theta}} \geq \lambda c_0a^2 - \paren{1-\lambda }N a^2
    \end{equation}
    for any $\lambda \in (0,1)$.
    Choose $\lambda \in (0,1)$ such that $\frac{N }{c_0 + N}<\lambda <1 $ so that $c_1 :=\frac12\paren{\lambda c_0 - \paren{1-\lambda }N} > 0$. 
    Continuity from Lemma~\ref{lem:Gcontinuity} implies that there exists $\varepsilon$ such that 
    \begin{equation*}
        \mathcal{G}_{L,\lambda } \paren{a; \xhat,\Sigma, \theta } \geq c_1a^2  \quad \text{whenever } \; \xhat^2 + |\Sigma - \bsig| + \|\theta - \bar{\theta}\| \leq \varepsilon.
    \end{equation*}
    Let $\theta_{s} = \paren{P_{s+1}, \dots , P_{s+L}, \Gamma_{s+1}, \dots , \Gamma_{s+L}}$.
    By Lemmas~\ref{lem:stationary} and \ref{lem:gammaconv}, we can choose $t_6$, independent of $T$, such that $\|\theta_{s}-\bar{\theta} \| \leq \varepsilon/2$ for all $s \leq T- t_6 - L$. 
    Choose $\delta>0$ such that $2\delta \leq \varepsilon/2$. 
    For every $(\xhat,\Sigma) \in \Mcal_\delta$,
    \[
    \xhat^2 + |\Sigma - \bsig| + \|\theta_s-\bar\theta\| \le \delta + \delta + \frac{\varepsilon}{2} \le \varepsilon.
    \]
    Therefore, 
    \begin{equation}
        \mathcal{G}_{L,\lambda } \paren{a; \xhat,\Sigma, \theta_{s} } \geq c_1 a^2
    \end{equation}
    for all $s \leq T- t_6 - L$ and $\paren{\xhat, \Sigma} \in \mathcal{M}_{\delta}$.
    Take the cutoff parameter in the definition of $\pi^P$ to be $t_6$.
    Fix $k \ge 1$ and consider the event $\{\tau_k < \infty\}$.
    For each deterministic starting time $s$, the corresponding local coefficients satisfy
    \[
    n_t = N_{s+t}, \quad k_t = K_{s+t}, \quad \gamma_{t+1} = \Gamma_{s+t+1}, \quad 0 \le t<L.
    \]
    In particular, 
    \[
    r^{\theta_s}_{t,\lambda}(\hat x,\Sigma,u)
    =r_{s+t,\lambda}(\hat x,\Sigma,u).
    \]
    On the event $\{\tau_k=s\}$, the actions within the block therefore coincide with those of $\pi^{\theta_s,a}$, initialized at $(\xhat_s,\Sigma_s)$.
    Hence, on each event $\{\tau_k = s\}$,
    \begin{align*}
        \E_{\pi^P}\squ{\sum_{t=0}^{L-1} r_{\tau_k+t, \lambda}\paren{\xhat_{\tau_k+t}, \Sigma_{\tau_k+t}, u_{\tau_k+t}} \; \middle|\; \Fcal_{\tau_k}} = \Gcal_{L,\lambda}(a;\xhat_{\tau_k}, \Sigma_{\tau_k},\theta_{\tau_k}) \ge c_1 a^2
    \end{align*}
    since $\tau_k \leq T- t_6 - L$ and $\paren{\xhat_{\tau_k},\Sigma_{\tau_k}} \in \mathcal{M}_{\delta}$. This completes the proof.
\end{proof}
\newpage

\section{Numerical Simulations}\label{app:exp}
\subsection{Regime visualization}
In this section we numerically visualize the key regimes characterized by our results.  We fix the scalar instance \begin{equation}
    A = 0.8 \quad B=C=Q=R = \Sigma_w =\mu_x=1 \quad S=0
\end{equation} and vary $q$ by varying $\sigma_z^2$. We use the fixed, heuristically chosen probing parameters \begin{equation}
    a = 0.25 \quad \delta = 0.03\quad  L = 5 \quad t_P = 20
\end{equation} 
throughout all experiments. Figure~\ref{fig:q_sweep} sweeps over different values of $q$ at $T = 800$. The result exhibits behavior predicted by our theory. For small $q$, neither CE nor probing significantly improves over OL; for intermediate values of $q$, probing improves linearly over open loop, while the advantage of CE remains bounded; for large values of $q$ both probing and CE achieve linear advantage over open-loop. Figure~\ref{fig:regime-horizon} complements this by visualizing cumulative advantage as $T$ varies for fixed values of $q$. For $q =0.5$, which lies between $q_2$ and $q_3$, the advantage of CE over open-loop remains bounded, while the advantage of probing grows linearly. For $q =4$, which lies above $q_4$, both the advantage of CE and probing grow linearly. All experiments use 2000 independent Monte Carlo trajectories per policy. Error bands denote 95\% normal confidence intervals for the trajectory-level mean.
\begin{figure*}[t]
    \centering
    \includegraphics[width=\textwidth]{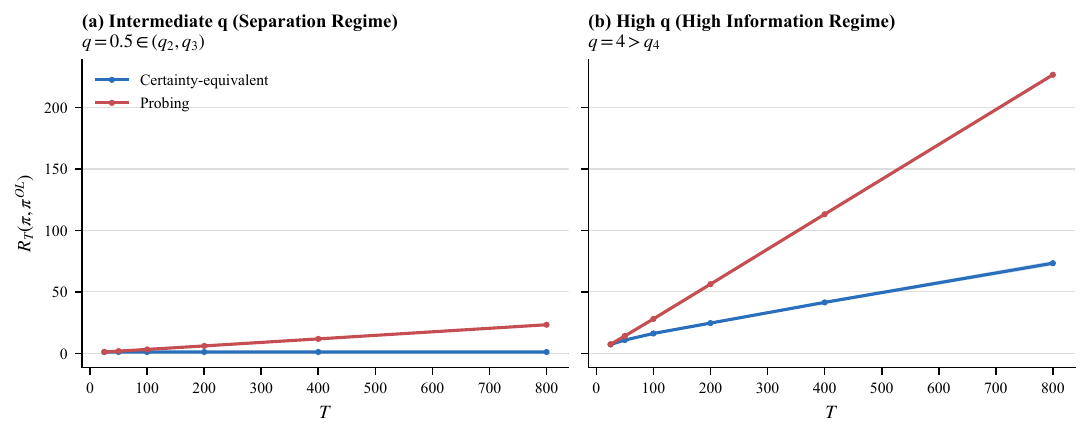}
    \caption{Numerical illustration of CE and probing behavior in two different regimes. We plot the cumulative regret $R_T(\pi,\pi^{OL})$ over the open-loop policy as a function of the horizon $T$. In the intermediate q regime $q_2<q=0.50<q_3$ (left), the CE regret remains bounded while the probing policy accumulates regret linearly. In the high q regime $q=4.00>q_4$ (right), CE itself accumulates linear regret. The same probing-policy parameters are used in both panels. }
    \label{fig:regime-horizon}
\end{figure*}
\subsection{Multivariate dynamical system}

We next investigate whether this phenomenon persists in a multivariate system. We consider a dynamical system with 
$p=n=3$, with:
\begin{equation}
    A=\begin{bmatrix}
    0.8 & 0.1 &0 \\
    0 & 0.75 &0.1 \\
    0.05 &0 & 0.7
\end{bmatrix} \quad B = C = Q =R = \Sigma_w = I_3 \quad S = 0 \quad \mu_x = \mathbf{1}_3.
\end{equation}
We use the multivariate analog of $q$:
\begin{equation}
    q = \frac{\|A \bsig C^\top K\|_2^2}{\sigma_z^2}
\end{equation}
and we vary $q$ by varying $\sigma_z^2$ similarly to the scalar experiment. We emphasize that here $q$ is only a convenient measure of information strength, and our theoretical thresholds $q_1, \dots, q_4$ do not necessarily extend to this multivariate setting. We generalize $\pi^{P}$ as follows. Define
\begin{equation}
    \mathcal{M}_{\delta} = \curl{(\xhat, \Sigma): \| \xhat\|_2^2 \leq \delta, \tr\paren{\bar{\Sigma} - \Sigma} \leq \delta}.
\end{equation} Whenever the probing conditions of Eq~\eqref{eq:probestoppingtime} are satisfied $\pi^{P}$ takes the action \begin{equation}
    K_t\xhat_t + a \cdot \mathbf{1}_3.
\end{equation} We use the same probing parameters $a = 0.25, \delta = 0.03,L = 5,t_P = 20$ as the previous section.

Figure~\ref{fig:q_sweep2} shows that the same qualitative phenomenon persists in this multivariate example. For small values of $q$ both policies do no better than open-loop. As $q$ increases, probing begins to achieve a linear advantage over open-loop before CE does. For large values of $q$, CE eventually achieves a linear advantage over open-loop, and its performance approaches the probing policy. This experiment suggests that the separation between CE and probing is not restricted only to scalar settings. All experiments use 2000 independent Monte Carlo trajectories per policy. Error bands denote 95\% normal confidence intervals for the trajectory-level mean.
\begin{figure*}[t]
    \centering
    \includegraphics[width=\textwidth]{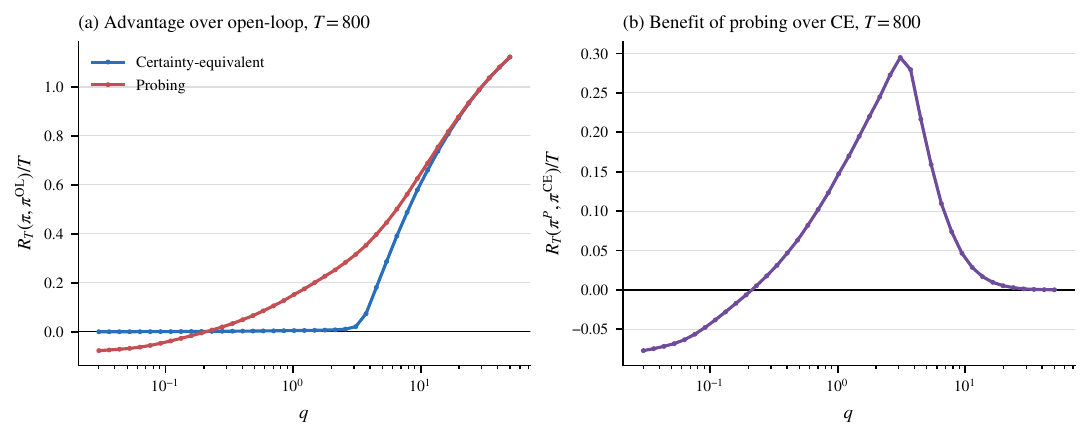}
    \caption{(a) Depicts the normalized advantage of CE and probing over open-loop; (b) depicts the normalized advantage of probing over CE. The policy performance as $q$ varies is similar to the scalar experiment in Figure~\ref{fig:q_sweep}; we see a clear separation between the two policies for intermediate values of $q$.}
    \label{fig:q_sweep2}
\end{figure*}
\newpage

\section*{Nomenclature}
\addcontentsline{toc}{section}{Nomenclature}

\begingroup
\small
\renewcommand{\arraystretch}{1.18}
\setlength{\LTpre}{4pt}
\setlength{\LTpost}{4pt}
\begin{longtable}{@{}>{\raggedright\arraybackslash}p{0.22\linewidth}>{\raggedright\arraybackslash}p{\dimexpr0.78\linewidth-2\tabcolsep\relax}@{}}
\caption{Main notation for the scalar setting ($n=p=1$). Constants local to individual proofs are omitted.}
\label{tab:notation}\\
\toprule
\textbf{Symbol} & \textbf{Meaning} \\
\midrule
\endfirsthead
\multicolumn{2}{@{}l}{\textit{Table~\ref{tab:notation} continued}}\\
\toprule
\textbf{Symbol} & \textbf{Meaning} \\
\midrule
\endhead
\midrule
\multicolumn{2}{r@{}}{\textit{Continued on next page}}\\
\endfoot
\bottomrule
\endlastfoot

\multicolumn{2}{@{}l}{\textbf{Model}}\\*
$t,\ T$ & Time index and horizon; $t=0,\ldots,T$. \\
$x_t,\ u_t,\ y_t$ & State, action, and observation. \\
$A,\ B,\ C$ & State transition, action, and bilinear observation coefficients. \\
$w_t,\ z_t$ & Process noise and measurement noise. \\
$\Sigma_w,\ \sigma_z^2$ & Process-noise and measurement-noise variances. \\
$Q,\ R,\ S$ & State, action, and cross-term weights in the quadratic stage cost. \\
$\mu_x$ & Initial state mean. \\
$\bar\Sigma$ & Stationary open-loop variance, also used as the initial variance. \\
\addlinespace[5pt]

\multicolumn{2}{@{}l}{\textbf{History and filtering}}\\*
$H_t$ & History of actions and observations available before choosing $u_t$. \\
$\mathcal H_t,\ \mathcal F_t$ & Space of histories and the information $\sigma$-algebra $\mathcal F_t=\sigma(H_t)$. \\
$\hat x_t$ & Conditional state mean, $\mathbb E[x_t\mid H_t]$. \\
$\Sigma_t$ & Conditional state variance, $\operatorname{Var}(x_t\mid H_t)$. \\
$F(\Sigma,u)$ & Reduction in the next-step estimation variance due to an observation. \\
$Z_t$ & Normalized Gaussian innovation. \\
$M_t,\ q_t$ & Multiplier in the CE estimate recursion and its conditional standard deviation. \\
\addlinespace[5pt]

\multicolumn{2}{@{}l}{\textbf{Policies and costs}}\\*
$\Pi,\ \Pi^{OL}$ & Classes of policies and open-loop policies. \\
$\pi^\star$ & Optimal policy. \\
$\pi^{CE}$ & Certainty-equivalent policy. \\
$\pi^{OL}$ & Optimal open-loop policy. \\
$\pi^P$ & Probing policy. \\
$c(x,u)$ & Stage cost. \\
$\bar c(\hat x,\Sigma,u)$ & Conditional expected stage cost given the state estimate and variance. \\
$\widetilde Q_t(x,u)$ & Optimal state-action cost-to-go in the fully observed problem. \\
$J_T(\pi)$ & Expected cumulative cost under policy $\pi$. \\
$R_T(\pi,\pi')$ & Cost advantage of $\pi$ over $\pi'$: $J_T(\pi')-J_T(\pi)$. \\
\addlinespace[5pt]

\multicolumn{2}{@{}l}{\textbf{Riccati and open-loop quantities}}\\*
$P_t$ & Quadratic state coefficient in the fully observed optimal cost-to-go. \\
$N_t$ & Quadratic action coefficient in the completed-square cost-to-go. \\
$K_t$ & Optimal fully observed feedback gain, also used by CE and OL. \\
$D_t$ & Closed-loop state coefficient, $A+BK_t$. \\
$G_t$ & Coefficient of the state variance in the open-loop cost-to-go. \\
$\Gamma_t$ & Weight assigned to estimation-variance reduction. \\
$P,\ N,\ K,\ D,\ \Gamma$ & Stationary counterparts of $P_t,N_t,K_t,D_t,\Gamma_t$. \\
$m_t$ & State mean under the optimal open-loop policy. \\
\addlinespace[5pt]

\multicolumn{2}{@{}l}{\textbf{Regret decomposition}}\\*
$p_t(\hat x,u)$ & Penalty for deviating from the CE action. \\
$b_t(\Sigma,u)$ & Weighted benefit of the current observation's covariance reduction. \\
$d_t(\Sigma)$ & Weighted benefit of uncertainty reductions retained from earlier observations. \\
$r_t(\xhat,\Sigma,u)$ & Per-step regret term, $b_t(\Sigma,u)-p_t(\xhat,u)$. \\
$r_{t,\lambda},\ \lambda$ & Weighted per-step regret term and its mixing weight, $\lambda\in[0,1]$. \\
\addlinespace[5pt]

\multicolumn{2}{@{}l}{\textbf{Information strength and sufficient thresholds}}\\*
$q$ & Information strength, $A^2K^2C^2\bar\Sigma^2/\sigma_z^2$. \\
$q_1$ & Threshold for bounded advantage of $\pi^\star$ over $\pi^{OL}$. \\
$q_2$ & Threshold for linear advantage of $\pi^P$ over $\pi^{OL}$. \\
$q_3$ & Threshold for bounded advantage of $\pi^{CE}$ over $\pi^{OL}$. \\
$q_4$ & Threshold for linear advantage of $\pi^{CE}$ over $\pi^{OL}$. \\
$\gamma$ & CE second-moment contraction factor, $(D^2+q+1)/2$. \\
$\gamma_E$ & Euler's constant. \\
\addlinespace[5pt]

\multicolumn{2}{@{}l}{\textbf{Probing policy}}\\*
$a,\ L$ & Probe size and block length; $L$ is also the minimum spacing between probes. \\
$\delta$ & Tolerance defining the probing region. \\
$\mathcal M_\delta$ & Region with a small state estimate and variance close to $\bar\Sigma$. \\
$t_P$ & Parameter controlling the cutoff for probes near the end of the horizon. \\
$\tau_k,\ \mathcal T$ & Time of the $k$th probe and set of finite probe times. \\
$\mathcal I_k$ & Search window for the $k$th probe. \\
\addlinespace[5pt]

\multicolumn{2}{@{}l}{\textbf{Probe-block analysis}}\\*
$\theta,\ \bar\theta$ & Vector of block coefficients and its stationary counterpart. \\
$k_t,\ n_t$ & Feedback gain and quadratic action coefficient associated with $\theta$. \\
$\pi^{\theta,a}$ & Block policy that probes at its first step, then uses the gains $k_t$. \\
$r^\theta_{t,\lambda}$ & Per-step payoff for the block with coefficients $\theta$. \\
$\mathcal G_{L,\lambda}$ & Expected total payoff over a probe block of length $L$. \\
$f_L(a)$ & Expected block payoff starting from $(\hat x,\Sigma)=(0,\bar\Sigma)$ with stationary coefficients and $\lambda=1$. \\
$g(L)$ & Probe-gain factor in the small-probe payoff bound. \\
\end{longtable}
\endgroup

\end{document}